\documentclass[twoside]{article}
\usepackage[preprint]{aistats2027}

\usepackage[utf8]{inputenc}
\usepackage[T1]{fontenc}
\usepackage{microtype}
\usepackage{wrapfig}
\usepackage{float}
\usepackage{xcolor}
\usepackage{graphicx}
\usepackage{subcaption}
\usepackage{tabularx}
\usepackage{booktabs}
\usepackage{longtable}
\usepackage{url}
\usepackage{tikz}
\usepackage{amsmath}
\usepackage{amssymb}
\usepackage{amsfonts}
\usepackage{mathtools}
\usepackage{enumitem}
\usepackage{amsthm}
\usepackage{caption}

 \usepackage{natbib}
\usepackage[ruled,vlined,linesnumbered]{algorithm2e}
\usepackage{nicefrac}
\usepackage[ruled,vlined,linesnumbered]{algorithm2e}
\usepackage{hyperref}
\usepackage[capitalize,noabbrev]{cleveref}

\usetikzlibrary{positioning,arrows.meta,calc,fit,shapes.geometric, backgrounds}

\definecolor{RawPurple}{HTML}{5B2DA0}

\definecolor{RawFill}{HTML}{F3EEFB}

\definecolor{RawLight}{HTML}{F7F3FC}

\definecolor{ExpGreen}{HTML}{2E7D32}

\definecolor{ExpFill}{HTML}{EEF7ED}

\definecolor{ExpLight}{HTML}{F4FAF3}

\definecolor{SccrBlue}{HTML}{174EA6}

\definecolor{SccrFill}{HTML}{F2F7FF}

\definecolor{HeadOrange}{HTML}{B06D15}

\definecolor{HeadFill}{HTML}{FFF6E8}

\definecolor{Ink}{HTML}{111827}

\definecolor{SoftGray}{HTML}{F8FAFC}

\definecolor{MidGray}{HTML}{6B7280}

\theoremstyle{plain}
\newtheorem{theorem}{Theorem}[section]
\newtheorem{proposition}[theorem]{Proposition}

\newtheorem{corollary}[theorem]{Corollary}

\theoremstyle{definition}

\theoremstyle{remark}
\newtheorem{remark}[theorem]{Remark}

\newcommand{\R}{\mathbb{R}}

\providecommand{\auditTableStyle}{}
\renewcommand{\auditTableStyle}{%
  \normalsize
  \setlength{\tabcolsep}{2pt}%
  \renewcommand{\arraystretch}{1}%
}
\providecommand{\auditsd}[2]{}
\renewcommand{\auditsd}[2]{\ensuremath{#1{\pm}#2}}
\providecommand{\auditci}[3]{}
\renewcommand{\auditci}[3]{\ensuremath{#1\,[#2,#3]}}

\hypersetup{%
  pdftitle={Shared Canonical Component Removal Beyond Expert Views},
  pdfauthor={Jose Roberto Tello Ayala, Akl C. Fahed, Roy Perlis, Finale Doshi-Velez}
}

\begin{document}
\runningauthor{Tello Ayala, Fahed, Perlis, and Doshi-Velez}

\twocolumn[

\aistatstitle{Shared Canonical Component Removal Beyond Expert Views}

\aistatsauthor{%
  Jose Roberto Tello Ayala \\
  \normalfont School of Engineering and Applied Sciences\\
  \normalfont Harvard University\\
  \normalfont\texttt{jtelloayala@g.harvard.edu}
  \And
  Akl C. Fahed\\
  \normalfont Heart and Vascular Institute\\
  \normalfont Mass General Brigham
  \AND
  Roy Perlis\\
  \normalfont Center for Quantitative Health\\
  \normalfont Mass General Brigham
  \And
  Finale Doshi-Velez\\
  \normalfont School of Engineering and Applied Sciences\\
  \normalfont Harvard University
}
\bigskip
]

\begin{abstract}
 In clinical practice, experts often begin by converting raw inputs into a set of expert features before making a decision. While black-box models trained on raw input can outperform models trained on expert views, it remains unclear what predictive signal the raw model exploits beyond what the expert view captures. We introduce an audit framework to quantify this gap using Shared Canonical Component Removal (SCCR). Unlike standard residualization or concept erasure, SCCR selectively removes only those directions strongly shared between learned raw and expert embeddings. 
 Evaluations on synthetic data, coronary artery disease (CAD) detection, antidepressant treatment dropout prediction, and skin lesion classification using SkinCon annotations show that SCCR identifies the ways in which expert views are insufficient and recovers residual predictive signal while minimizing redundancy with the expert view.
\end{abstract}

\section{Introduction}

In clinical practice, experts often begin by converting the raw input into an \textit{expert view}: a domain-informed transformation of the data, such as a vessel-morphology summary from a coronary angiogram or a curated set of psychiatric markers from electronic health records (EHR). ML models using the original raw input may use features not used by the expert. For example, when detecting coronary artery disease (CAD), clinicians focus on vessel geometry and stenosis patterns; however, a black-box classifier trained on the full angiogram may exploit additional cues, 
such as grayscale intensity and small vessels omitted from the mask. Similarly, when predicting antidepressant treatment dropout, a model trained only on psychiatric codes may miss broader clinical signals captured by a model trained on all diagnosis codes—such as sleep-disorder diagnoses, pain-related prescriptions, or repeated primary-care visits—that are predictive of antidepressant dropout but not explicitly included in the psychiatric view. Especially in high-stakes settings such as clinical decision making, this leads to a critical audit question: \textit{What predictive signal does the raw model use beyond what the expert view captures?} This question is important not only to identify potentially problematic short-cuts, but also to advance science: perhaps the model has found a valuable feature of which the experts are not aware. 

Classical concept bottleneck models route prediction through a specified layer of human-interpretable concepts \citep{pmlr-v119-koh20a}. More recent variants relax the assumption that a fixed human-specified concept set is sufficient for prediction \citep{NEURIPS2022_944ecf65,yuksekgonul2023posthoc,oikarinen2023labelfree,zarlenga2022concept,Shang_2024_CVPR,11094052}.  %
Our goal is to test whether linear information from the raw input still improves prediction after reducing its strongest overlap with the expert view. We then ask whether omitted concepts explain this improvement. The audit can therefore guide later refinement of the concept set.

Other techniques, such as latent-space concept erasure, aim to prevent a specific attribute from being recovered by a linear probe \citep{ravfogel-etal-2020-null, 10.5555/3666122.3669006}. While appropriate for fairness or privacy, removing every linear association can discard useful predictive signal \citep{holstege2025preservingtaskrelevantinformationlinear}. For example, psychiatric codes may weakly reflect general primary-care use, while the broader EHR may record those visits directly. Removing every association with psychiatric codes may therefore discard information that the psychiatric view captures only weakly. We seek to retain such information for further audit.

Finally, multi-view learning typically %
combines information from paired inputs\citep{jordanbach,pmlr-v28-andrew13,pmlr-v37-wangb15,wang2017deepvariationalcanonicalcorrelation,NIPS2016_45fbc6d3,10.1214/12-AOAS597,PALZER2022107547}. In our setting, the expert view serves as the reference for examining the raw input. For instance, in CAD detection, combining the raw angiogram and vessel morphology summary such as a segmentation mask may improve prediction. Our audit asks whether the angiogram still adds predictive value after reducing its strongest overlap with the mask.

Formally, let $x$ be the raw input and $\tilde{x}=g(x)$ be the expert view. Learned encoders produce raw and expert embeddings $z_r$ and $z_e$, respectively. The audit problem is to identify which complementary signal in \(z_r\) explain the performance gap between an expert-view model and a raw black-box model. A natural approach is to regress $z_r$ on $z_e$ and subtract the predicted component. However, when the expert view contains only a noisy proxy for a raw feature, this can leave a noisy version of that feature in the residual (see Sec.~\ref{sec:ped-example}).

To support this audit, we introduce Shared Canonical Component Removal (SCCR), a selective residualization operator based on regularized canonical correlation analysis (CCA) \citep{JordanEssaiSL,da6385d2-9c65-3860-bbcd-b821fdff69ff}. SCCR identifies linear directions shared between the raw and expert embeddings and removes only those whose canonical correlation meets or exceeds a threshold $\tau$. 
We combine the remaining raw signal with the expert view for prediction. The audit tests whether this remaining signal improves prediction and examines what explains that improvement. Here, ``beyond the expert view'' refers to signal remaining after this selective linear removal; it does not imply independence from the expert view. Our contributions are three-fold:
\begin{enumerate}[
    leftmargin=*,
    topsep=0pt,
    itemsep=0pt,
    parsep=0pt,
    partopsep=0pt
]
\item \textbf{An audit framework:} We formulate a paired-view setting in which an expert-derived view serves as a reference for testing whether domain-specified information is sufficient for prediction and for measuring complementary predictive signal beyond that view. Through synthetic and clinical experiments on CAD, antidepressant treatment dropout, and skin lesions, we show how the expert view can be used as an audit anchor.

\item \textbf{The SCCR operator:} We define a canonical removal operator that uses the expert view as a reference anchor and removes only strongly shared linear directions from the raw embedding. %

\item \textbf{Frozen calibration for deterministic inference:} 
We prove stability of both the frozen operator and the induced inference map, and show empirically that the resulting dual-branch model can retain black-box performance while reducing redundancy with the expert view.
\end{enumerate}

\section{Related Work}

\paragraph{Concept-based interpretation and bottleneck models.}
Linear probes and concept-based explanation methods diagnose whether task-relevant attributes or user-defined concepts are represented in hidden activations \citep{alain2018understandingintermediatelayersusing,pmlr-v80-kim18d}, but they are primarily diagnostic and do not impose a concept-mediated prediction pathway. Concept bottleneck models (CBMs) offer a more constructive alternative by first predicting human-interpretable concepts and then using those concepts to predict the label \citep{pmlr-v119-koh20a}. While classical CBMs rely on a user-specified concept bottleneck, later work has shown that this bottleneck can be limited by insufficient concept sets, scarce concept supervision, or accuracy--interpretability trade-offs. Leakage-aware CBMs analyze how insufficient concept sets and inexpressive concept predictors can induce leakage \citep{NEURIPS2022_944ecf65}; post-hoc and label-free CBMs reduce the need for dense concept annotations and end-to-end CBM training \citep{yuksekgonul2023posthoc,oikarinen2023labelfree}; concept embedding models replace scalar concept variables with richer concept representations \citep{zarlenga2022concept}; and residual or hybrid CBM variants explicitly model information missing from an incomplete predefined concept bank \citep{Shang_2024_CVPR,11094052}. Our framework differs in its fundamental goal: we do not seek to complete the concept set or optimize a bottleneck for accuracy. Instead, we use the expert view as an audit anchor to quantify residual predictive signal in the raw representation that is not captured by domain expertise.

\paragraph{Concept erasure and information removal.}
Concept erasure studies how to suppress designated attributes from a representation, typically for fairness, privacy, or controlled interpretation \citep{ravfogel-etal-2020-null,pmlr-v162-ravfogel22a}. Many of these methods aim for ``linear guardedness'': after transformation, the target concept should not be recoverable by linear probes. Extensions include closed-form linear erasure, kernelized adversarial erasure, and spectral covariance-based removal methods such as SAL \citep{10.5555/3666122.3669006,ravfogel-etal-2022-adversarial,shao-etal-2023-gold}. Our objective is fundamentally different: we do not seek complete erasure or unlearnability of the expert view. In an audit setting, a complementary predictive factor may be only partially represented in the expert view as a noisy proxy. For example, curated EHR markers may weakly reflect treatment failure, while the raw longitudinal record contains cleaner temporal evidence of repeated medication changes, utilization patterns, or comorbidity trajectories. Methods that enforce guardedness can therefore remove label-relevant variation when the erased concept and task-relevant factors overlap \citep{10.5555/3600270.3601578,holstege2025preservingtaskrelevantinformationlinear}. By thresholding the regularized canonical spectrum, SCCR replaces full erasure with selective filtering: it suppresses strongly shared raw--expert directions while preserving lower-correlation directions that may encode signal related to domain knowledge but not sufficiently explained by the expert view.

\paragraph{Multi-view learning and shared/private representations.}
The closest conceptual literature is multi-view representation learning, especially methods that separate shared/common structure from view-specific/private structure. Classical CCA and its probabilistic interpretation \citep{jordanbach} model common linear variation across paired views. Deep extensions such as DCCA \citep{pmlr-v28-andrew13}, DCCAE \citep{pmlr-v37-wangb15}, and VCCA \citep{wang2017deepvariationalcanonicalcorrelation} learn nonlinear common representations, while VCCA-private and Domain Separation Networks (DSN) \citep{NIPS2016_45fbc6d3} explicitly introduce both shared and private components. Related factorization methods such as JIVE \citep{10.1214/12-AOAS597}, AJIVE \citep{10.1016/j.jmva.2018.03.008}, and sJIVE \citep{PALZER2022107547} decompose data into joint and individual structure, with sJIVE incorporating supervision through the outcome. Decomposition-based CCA (D-CCA) \citep{Shu02012020} separates common and distinctive variation, requiring the distinctive components of the two views to be uncorrelated. Recent work has questioned the assumption that downstream-relevant information is fully redundant across views, emphasizing shared, view-specific, complementary, and synergistic information \citep{10.5555/3666122.3667552,dufumier2025what,gao2025deep}; related theory studies partial observability, where each view may expose only a subset of the underlying latent factors \citep{ICLR2024_956e5427}. Unlike standard multi-view learning, our two views are not peer modalities: we do not seek to learn a common representation of both views, but instead use the expert view asymmetrically as an audit anchor for testing whether it is sufficient and for identifying what predictive signal remains beyond it.

\section{Method}
\label{sec:method}
In this section, we define the prediction setting and introduce our SCCR operator and architecture. In Sec.~\ref{sec:ped-example}, we show how SCCR differs from existing residualization techniques in ways that support our audit goals.

\subsection{Setting}

\paragraph{Paired-view setup.}
We consider supervised classification tasks with paired views of the same underlying input. Let \(x \in \R^{d_{\mathrm{in}}}\) denote the raw input and let \(\tilde{x}=g(x)\in\R^{d_{\mathrm{exp}}}\) denote an expert view, such as a segmentation mask or a curated subset of features. We observe training samples
\(\{(x_i,\tilde{x}_i,y_i)\}_{i=1}^N, y_i \in \{1,\dots,C\}\). The expert view is intended to capture information that practitioners consider relevant, but it is not assumed to be complete. The central audit question is whether this expert view is sufficient for prediction, and, if not, whether the raw input contains additional predictive signal beyond it.  

\paragraph{Dual branch architecture.}
The model consists of a raw branch mapping $x$ to $z_r = f_{\mathrm{raw}}(x)$ and an expert branch mapping $\tilde{x}$ to \(z_e = f_{\mathrm{expert}}(\tilde{x})\). These networks, along with the prediction head, are trained end-to-end. The SCCR operator which removes strongly shared canonical directions from the raw representation $z_r$ to obtain the residual $z_\perp$, is not a trainable layer; instead, it is computed from empirical second moments of paired embeddings for any desired correlation threshold $\tau$. The final predictor takes the form $\hat{y} = h(z_\perp,z_e^{\mathrm{ctr}})$, where \(z_\perp\) is the residual raw embedding and \(z_e^{\mathrm{ctr}}\) is the centered expert embedding.

\begin{figure*}[!t]
\centering

\def\RawInputImage{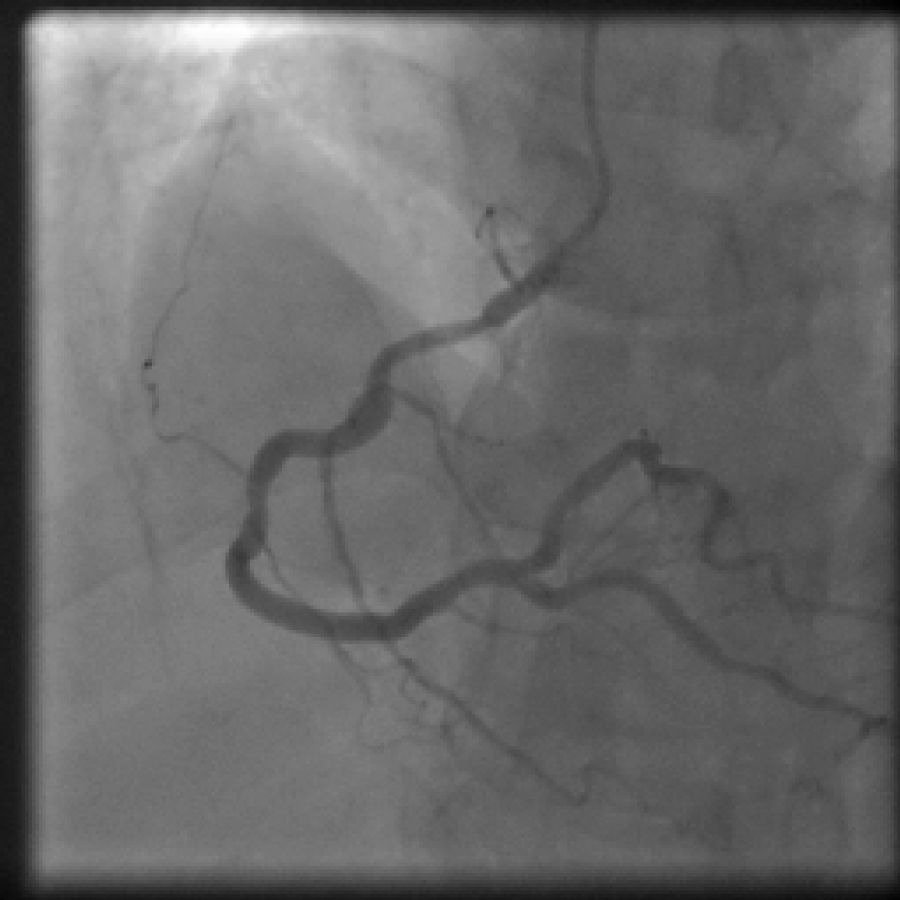}
\def\ExpertViewImage{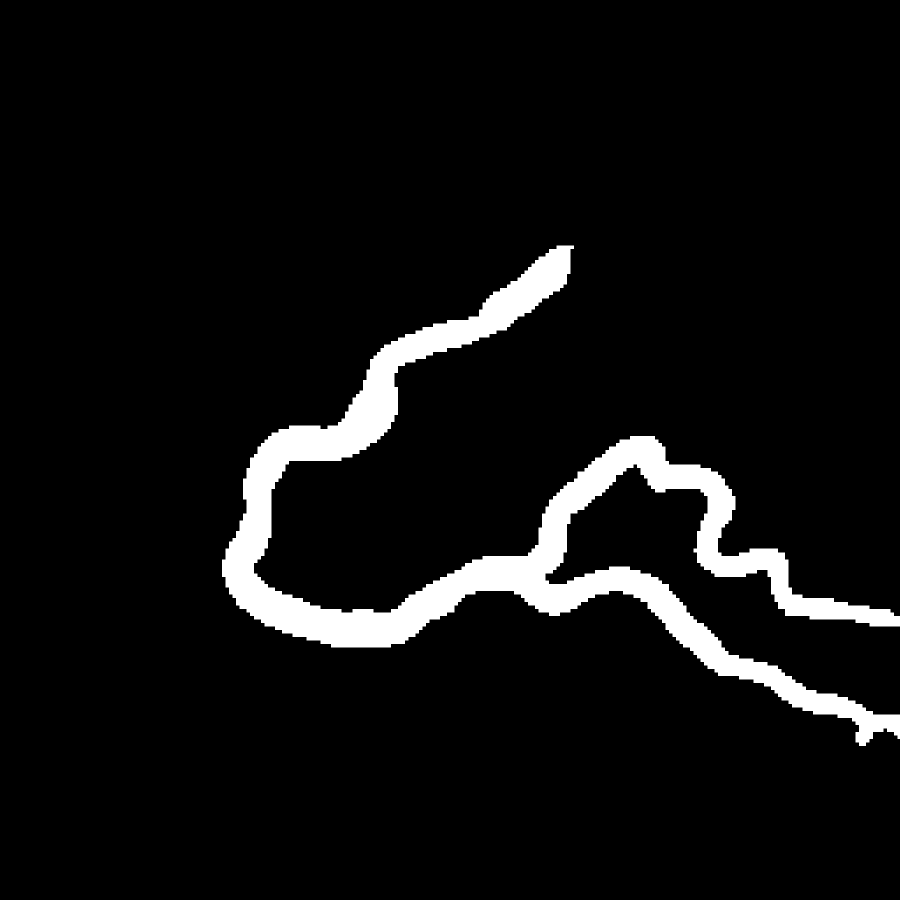}

\def\InputImageSize{1.50cm}
\def\InputImagePanel#1{%
    \node[inner sep=0pt, outer sep=0pt] at (0.94,0.47)
        {\includegraphics[width=\InputImageSize,height=\InputImageSize]{#1}};
}

\resizebox{\textwidth}{!}{    %
\begin{tikzpicture}[
    >=Latex,
    line cap=round,
    line join=round,
    node distance=0.65cm and 0.95cm,
    data/.style={draw, rounded corners=4pt, fill=SoftGray, minimum width=2.25cm, minimum height=0.72cm, align=center, inner sep=3pt},
    block/.style={draw, rounded corners=5pt, minimum width=1.65cm, minimum height=0.92cm, align=center, inner sep=4pt},
    procraw/.style={block, draw=RawPurple, fill=RawFill},
    procexp/.style={block, draw=ExpGreen, fill=ExpFill},
    embedraw/.style={draw=RawPurple!80!black, rounded corners=5pt, minimum width=0.42cm, minimum height=1.78cm, align=center, inner sep=1.5pt},
    embedexp/.style={draw=ExpGreen!80!black, rounded corners=5pt, minimum width=0.42cm, minimum height=1.78cm, align=center, inner sep=1.5pt},
    funnelraw/.style={draw=RawPurple, fill=RawLight, line width=0.75pt, rounded corners=2.2pt},
    funnelexp/.style={draw=ExpGreen, fill=ExpLight, line width=0.75pt, rounded corners=2.2pt},
    sccr/.style={draw=SccrBlue, dashed, line width=0.8pt, rounded corners=9pt, fill=SccrFill, minimum width=4.35cm, minimum height=3.75cm, align=center, inner sep=7pt},
    residual/.style={draw=RawPurple, rounded corners=5pt, fill=RawFill, minimum width=2.65cm, minimum height=1.12cm, align=center, inner sep=4pt, font=\scriptsize},
    expertbranch/.style={draw=ExpGreen, rounded corners=5pt, fill=ExpFill, minimum width=2.05cm, minimum height=0.9cm, align=center, inner sep=4pt, font=\scriptsize},
    head/.style={draw=HeadOrange, rounded corners=5pt, fill=HeadFill, minimum width=2.75cm, minimum height=1.22cm, align=center, inner sep=4pt, font=\scriptsize},
    rawarrow/.style={->, line width=1.1pt, draw=RawPurple},
    exparrow/.style={->, line width=1.1pt, draw=ExpGreen},
    statarrow/.style={->, line width=1.0pt, draw=ExpGreen, dashed},
    outarrow/.style={->, line width=1.0pt, draw=HeadOrange}
]

\coordinate (rawLane) at (0,1.70);
\coordinate (expLane) at (0,-1.25);

\node[align=center, text=RawPurple, font=\bfseries\small] (rawtitle) at (0,2.85) {RAW INPUT};
\node[align=center, font=\small] (rawlabel) at (0,2.51) {$x$};

\begin{scope}[shift={(-0.94,1.10)}]
    \InputImagePanel{\RawInputImage}
\end{scope}

\node[align=center, text=ExpGreen, font=\bfseries\small] (exptitle) at (0,-0.10) {EXPERT VIEW};
\node[align=center, font=\small] (explabel) at (0,-0.44) {$\tilde{x}=g(x)$};

\begin{scope}[shift={(-0.94,-1.85)}]
    \InputImagePanel{\ExpertViewImage}
\end{scope}

\coordinate (encrawIn)  at (1.72,1.70);
\coordinate (encrawOut) at (3.18,1.70);

\path[funnelraw]
    (1.72,2.43) -- (3.18,2.09) -- (3.18,1.31) -- (1.72,0.97) -- cycle;

\node[align=center, font=\small] (encraw) at (2.43,1.70) {encoder\\[-2pt]$f_{\mathrm{raw}}$};

\coordinate (encexpIn)  at (1.72,-1.25);
\coordinate (encexpOut) at (3.18,-1.25);

\path[funnelexp]
    (1.72,-0.52) -- (3.18,-0.86) -- (3.18,-1.64) -- (1.72,-1.98) -- cycle;

\node[align=center, font=\small] (encexp) at (2.43,-1.25) {encoder\\[-2pt]$f_{\mathrm{expert}}$};

\node[embedraw] (zrbar) at (3.95,1.70) {$\bullet$\\[-2pt]$\bullet$\\[-2pt]$\vdots$\\[-2pt]$\bullet$};
\node[above=0.08cm of zrbar, font=\small] {$z_r$};

\node[embedexp] (zebar) at (3.95,-1.25) {$\bullet$\\[-2pt]$\bullet$\\[-2pt]$\vdots$\\[-2pt]$\bullet$};
\node[above=0.08cm of zebar, font=\small] {$z_e$};

\node[procraw] (projraw) at (5.45,1.70) {project\\[-2pt]\& center\\[-1pt]$z_r-\mu_r$};
\node[procexp] (projexp) at (5.45,-1.25) {project\\[-2pt]\& center\\[-1pt]$z_e-\mu_e$};

\node[sccr] (sccrbox) at (9.05,0.35) {};

\node[font=\bfseries\small, text=SccrBlue] 
    at ($(sccrbox.north)+(0,-0.30)$) {SCCR OPERATOR};

\node[font={\itshape\fontsize{5.5pt}{6pt}\selectfont}, text=SccrBlue]
    at ($(sccrbox.north)+(0,-0.60)$)
    {shared canonical component removal};

\draw[SccrBlue!40, line width=0.35pt] 
    ($(sccrbox.center)+(0,0.42)$) -- ($(sccrbox.center)+(0,-1.18)$);

\node[
    align=center,
    font=\scriptsize,
    text=SccrBlue!95!black,
    text width=1.55cm
] at ($(sccrbox.center)+(-1.3,0.62)$) {
    \textbf{ESTIMATION}\\[-2pt]
    $T$
};

\node[
    align=center,
    font=\scriptsize,
    text width=1.55cm
] at ($(sccrbox.center)+(-1.1,-0.42)$) {
    Minibatch\\[-2pt]
    statistics\\[-1pt]
    Regularized\\[-2pt]
    CCA\\[-1pt]
    Threshold\\[-2pt]
    $s_i \ge \tau$
};

\node[
    align=center,
    font=\scriptsize,
    text=SccrBlue!95!black,
    text width=1.55cm
] at ($(sccrbox.center)+(1.03,0.62)$) {
    \textbf{INFERENCE}\\[-2pt]
    $T^\star$
};

\node[
    align=center,
    font=\scriptsize,
    text width=1.55cm
] at ($(sccrbox.center)+(1.12,-0.42)$) {
    Calibration\\[-2pt]
    set $\mathcal{C}$\\[-1pt]
    Frozen means\\[-2pt]
    $\mu^\star$\\[-1pt]
    Deterministic\\[-2pt]
    projection
};

\node[font=\small, text=black] 
    at ($(sccrbox.south)+(0,0.25)$) {$T = L_r^{-\top} K L_r^{\top}$};

\node[residual, right=0.45cm of sccrbox, yshift=1.10cm] (resid) {
    \textbf{RESIDUAL RAW BRANCH}\\[-1pt]
    $z_\perp=(z_r-\mu_r^\star)T^\star$\\[-1pt]
    {\scriptsize complementary signal}
};

\node[expertbranch] (expertbranch) at ($(resid.center)+(0,-3.10)$) {
    \textbf{EXPERT BRANCH}\\[-1pt]
    $z_e^{\mathrm{ctr}}$
};

\node[head, right=0.45cm of resid, yshift=-1.45cm] (head) {
    \textbf{DUAL-BRANCH HEAD}\\[-1pt]
    $\hat y=h(z_\perp,z_e^{\mathrm{ctr}})$\\[-1pt]
    {\scriptsize main effects + interaction}
};

\node[below=0.38cm of head, font=\bfseries\scriptsize] (predlab) {PREDICTION};
\node[below=0.12cm of predlab, font=\small] (pred) {$\hat y$};

\coordinate (sccrRawIn)   at (sccrbox.west |- projraw.east);
\coordinate (sccrStatIn)  at (sccrbox.west |- projexp.east);
\coordinate (sccrRawOut)  at (sccrbox.east |- resid.west);
\coordinate (expBypassA)  at ($(projexp.south east)+(0.58,0)$);
\coordinate (expBypassB)  at ($(expertbranch.west)+(-0.72,0)$);
\coordinate (headRawIn)   at ($(head.west)+(0,0.33)$);
\coordinate (headExpIn)   at ($(head.west)+(0,-0.33)$);
\coordinate (rawHeadTurn) at ($(headRawIn)+(-0.34,0)$);
\coordinate (expHeadTurn) at ($(headExpIn)+(-0.58,0)$);

\draw[rawarrow] (0.88,1.70) -- (encrawIn);
\draw[exparrow] (0.88,-1.25) -- (encexpIn);

\draw[rawarrow] (encrawOut) -- (zrbar.west);
\draw[exparrow] (encexpOut) -- (zebar.west);

\draw[rawarrow] (zrbar.east) -- (projraw.west);
\draw[exparrow] (zebar.east) -- (projexp.west);

\draw[rawarrow]  (projraw.east) -- (sccrRawIn);
\draw[statarrow] (projexp.east) -- (sccrStatIn);

\draw[exparrow, rounded corners=4pt] 
    (projexp.south east) -- (expBypassA) |- (expBypassB) -- (expertbranch.west);

\draw[rawarrow] (sccrRawOut) -- (resid.west);

\draw[rawarrow, rounded corners=4pt] 
    (resid.east) -- (resid.east -| rawHeadTurn) -- (rawHeadTurn) -- (headRawIn);

\draw[exparrow, rounded corners=4pt] 
    (expertbranch.east) -- (expertbranch.east -| expHeadTurn) -- (expHeadTurn) -- (headExpIn);

\draw[outarrow] (head.south) -- (predlab.north);

\fill[RawPurple] (sccrRawIn) circle (1.15pt);
\fill[ExpGreen] (sccrStatIn) circle (1.15pt);
\fill[RawPurple] (sccrRawOut) circle (1.15pt);

\node[
    draw=MidGray,
    rounded corners=5pt,
    fill=white,
    minimum width=19.00cm,
    minimum height=0.85cm
] (legendbox) at (9.15,-2.85) {};

\draw[rawarrow]
    ($(legendbox.west)+(0.35,0.18)$) -- ++(0.85,0);
\node[anchor=west, font=\scriptsize]
    at ($(legendbox.west)+(1.55,0.18)$)
    {raw branch (complementary residual)};

\draw[exparrow]
    ($(legendbox.west)+(0.35,-0.18)$) -- ++(0.85,0);
\node[anchor=west, font=\scriptsize]
    at ($(legendbox.west)+(1.55,-0.18)$)
    {expert branch (anchor + predictive)};

\draw[statarrow]
    ($(legendbox.west)+(6.50,0.18)$) -- ++(0.85,0);
\node[anchor=west, font=\scriptsize] (statLegend)
    at ($(legendbox.west)+(7.70,0.18)$)
    {statistics flow (training only)};

\node[
    anchor=west,
    draw=SccrBlue,
    dashed,
    rounded corners=3pt,
    minimum width=0.55cm,
    minimum height=0.42cm
] (sccrLegend)
    at ($(statLegend.east)+(0.45,-0.18)$) {};

\node[anchor=west, font=\scriptsize]
    at ($(sccrLegend.east)+(0.35,0)$)
    {shared canonical component removal (SCCR)};
\end{tikzpicture}
}
\caption{
\textbf{Dual-branch SCCR architecture.} The raw input \(x\) and expert view \(\tilde{x}=g(x)\) are encoded into raw and expert embeddings. The expert branch serves as both a predictive input and a reference for constructing the residual raw branch. During training, the SCCR operator $T$ is estimated from minibatch statistics to compute \(Z_\perp=Z_{r,c}T\). After training, a calibration pass produces a frozen SCCR operator \(T^\star\), yielding deterministic residuals \(z_\perp=(z_r-\mu_r^\star)T^\star\).
}
\label{fig:architecture}
\end{figure*}

\subsection{Shared canonical component removal}

Because learned embedding coordinates are arbitrary, we use CCA to identify directions maximally correlated across views. We introduce a correlation threshold $\tau \in (0,1)$ to distinguish between redundant signal and complementary evidence. Without this threshold---as in standard orthogonal projection---any component of $z_r$ that has any linear overlap with $z_e$ is removed. This is problematic when a raw feature is a ``stronger'' version of a ``weak'' proxy in the expert view; such a feature contains complementary signal despite having some correlation with the expert view. SCCR addresses this by selectively thresholding the canonical spectrum: we remove only those directions whose correlations exceed $\tau$, treating directions with correlations above $\tau$ as redundant and directions with correlations below $\tau$ as candidate complementary evidence.
Let \(Z_r\in\mathbb{R}^{n\times d_r}\) and \(Z_e\in\mathbb{R}^{n\times d_e}\) be a collection of $n$ paired raw and expert embeddings. We define the row-wise centered versions,
\[
Z_{r,c}=Z_r-\mathbf{1}\mu_r, \quad Z_{e,c}=Z_e-\mathbf{1}\mu_e,
\]
where $\mu_j$ is the empirical mean of branch $j$. Let \(B=\max(n-1,1)\). The centered empirical second moments are:
\[
S_{rr}=\frac{1}{B}Z_{r,c}^{\top}Z_{r,c}, \enspace S_{ee}=\frac{1}{B}Z_{e,c}^{\top}Z_{e,c}, \enspace S_{re}=\frac{1}{B}Z_{r,c}^{\top}Z_{e,c}.
\]

\paragraph{OAS-selected branchwise ridge.} 
The empirical second moments above may be noisy because they are computed from finite samples. To stabilize the whitening, we
regularize the auto-covariance matrices using separate ridge
parameters selected by Oracle Approximating Shrinkage (OAS)\citep{10.1109/TSP.2010.2053029}. 
For the raw branch, let $m_r=\operatorname{tr}(S_{rr})/d_r$ be its average variance. The OAS coefficient is
\begin{equation}
\label{eq:oas_rho}
\rho_r = \min\!\left\{1,
\frac{
(1-\frac{2}{d_r})\operatorname{tr}(S_{rr}^2)
+\operatorname{tr}(S_{rr})^2
}{
(B+1-\frac{2}{d_r})
\left(\operatorname{tr}(S_{rr}^2)
-\frac{\operatorname{tr}(S_{rr})^2}{d_r}\right)
}
\right\}.
\end{equation}
The expert coefficient $\rho_e$ and average variance $m_e$
are computed analogously. We clamp both coefficients strictly
below one, using their maximum value when the denominator
is numerically zero. We then define
\begin{equation}
\label{eq:equiv_ridge}
\lambda_r=\frac{\rho_r m_r}{1-\rho_r},
\qquad
\lambda_e=\frac{\rho_e m_e}{1-\rho_e},
\end{equation}
and the regularized auto-covariances \(R_{rr}=S_{rr}+\lambda_r I\), \(R_{ee}=S_{ee}+\lambda_e I\).
These matrices account for differences in variance between
branches and improve conditioning for Cholesky whitening.
They are proportional to the OAS estimates. For example,
$S_{rr}^{\mathrm{OAS}}=(1-\rho_r)R_{rr}$. We use the ridge matrices $R_{rr}$ and $R_{ee}$ for whitening. We leave $S_{re}$ unchanged as we do not invert it during whitening
\paragraph{Whitened cross-covariance and thresholding.}
Let \(R_{rr}=L_rL_r^{\top}\) and \(R_{ee}=L_eL_e^{\top}\)
be Cholesky factorizations. The ridge-scale whitened cross-covariance is \(M=L_r^{-1}S_{re}L_e^{-\top}\).
We compute a thin singular value decomposition \mbox{\(M=U\operatorname{diag}(s)V^{\top}.\)}
The singular values \(s_i\in[0,1]\) correspond to the regularized canonical correlations.  Given a threshold \(\tau\in(0,1)\),  we define the high-correlation index set $H=\{i:s_i\ge \tau\}$ and the projector onto the complement of these directions in whitened raw space as $K=I-U_HU_H^{\top}$, where $U_H=U[:,H]$. The SCCR operator $T$ is then defined by transforming this projector back to the raw embedding coordinates:
\begin{equation}
\label{eq:sccr_operator_oas}
T=L_r^{-\top}KL_r^{\top}.
\end{equation}
The resulting SCCR residual is $Z_{\perp}=Z_{r,c}T$. %
Let $z_{\perp}$ be the residual raw embedding and $z_e^{\mathrm{ctr}}$ be the centered expert embedding. The architecture then allows each branch a nonlinear transformation:
$u_r=\phi_r(z_{\perp})$, and $u_e=\phi_e(z_e^{\mathrm{ctr}})$,
where $u_r$ and $u_e$ have the same dimensionality, and $\phi_e$ and $\phi_r$ are one-hidden-layer MLPs with SiLU activations and dropout, with optional branch normalization before the linear map. The logits $\ell$ for the final prediction consist of these outputs $u_r$ and $u_e$ as well as an interaction feature $u_{\times}=u_r\odot u_e$: 
{\small\begin{equation}
\label{eq:mei_head}
\ell = W_r^{\mathrm{out}}u_r + W_e^{\mathrm{out}}u_e + W_{\times}^{\mathrm{out}}u_{\times} + b, \enspace \hat y=\operatorname{softmax}(\ell).
\end{equation}}
This design supports auditability in two ways. First, by maintaining separate hidden states $u_r$ and $u_e$, we ensure that the raw-residual and expert pathways remain distinct. Second, the explicit interaction term \(u_{\times}\) allows us to isolate whether predictive utility depends on complementary evidence from both branches. %
\subsection{End-to-end Training}

The encoders, projectors, and prediction head are trained jointly by minimizing cross-entropy loss between $\hat{y}$ (Eq.~\ref{eq:mei_head}) and the ground truth labels. During training, we use the current minibatch as the reference set for the SCCR operator, rendering it a stochastic differentiable layer. The SCCR operation is differentiable with respect to branch embeddings, as gradients propagate through the covariance estimates, OAS-selected ridge construction, Cholesky whitening, and SVD. The only non-differentiable point is the hard threshold \(H=\{i:s_i\ge\tau\}\). For a fixed $H$ and distinct singular values, the map is smooth. In practice, ridge regularization and shrinkage clipping ensure stable optimization.

\subsection{Calibration for Frozen Inference}

To avoid dependence on test-batch composition, we freeze the SCCR operator after training via a calibration pass over the training set. Let \(\{(z_{r,i},z_{e,i})\}_{i=1}^N\) denote the post-projector embeddings for all training samples. We compute global emebdding means $\mu_r^\star$ and $\mu_e^\star$, then apply the SCCR construction to obtain the frozen operator $T^\star$. At test time:
\[ z_{\perp}=(z_r-\mu_r^\star)T^\star, \quad z_e^{\mathrm{ctr}}=z_e-\mu_e^\star.\]
This ensures that the residual map is deterministic and consistent with the training distribution. 

\paragraph{Ridge Scaling and Calibration Dynamics.}
Because OAS adaptive regularization depends on the effective sample size $n_{\mathrm{eff}}$, the ridge levels $\lambda_j$ differ between training and calibration. During training, SCCR is computed per-minibatch ($n_{\mathrm{eff}} = B_{\mathrm{mini}}-1$), resulting in larger ridge levels $\lambda_j^{\mathrm{train}}$ that ensure stable Cholesky decomposition on small batches. In contrast, the calibration pass utilizes the full training set ($n_{\mathrm{eff}} = N-1$) and global means $\mu_r^\star, \mu_e^\star$. This typically reduces the ridge penalty ($\lambda_j^\star  <\lambda_j^{\mathrm{train}}$), yielding canonical correlations that are closer to their population values. The threshold $\tau$ must  be chosen relative to calibration-scale correlations, as minibatch spectra may lead to inconsistent component removal during inference. The complete algorithm is given in Appendix~\ref{app:algorithm}.%
\section{Analysis}
\label{sec:analysis}

\subsection{Pitfalls of Existing Residualization Techniques and Benefit of SCCR}
\label{sec:ped-example}
To show the trade-off between redundancy and complementary signal discussed in Section \ref{sec:method}, we consider a toy example with four raw factors: \(S\) is fully shared with the expert view, \(A\) is strongly shared, and \(F_1,F_2\) are weakly shared factors that we seek to retain.
Let \(S,A,F_1,F_2,N_A,N_1,N_2\sim\mathcal N(0,1)\) be mutually independent. The noise variables \(N_A,N_1,N_2\) form the expert proxies
\[
Q=\rho A+\sqrt{1-\rho^2}N_A,\qquad
P_i=\gamma F_i+\sqrt{1-\gamma^2}N_i,
\]
where \(Q\) is a noisy version of \(A\), and \(P_i\) is a noisy version of \(F_i\). We set \(\rho=.7\), \(\gamma=.3\), and define
{\small
\begin{equation*}
    z_r=(\sqrt2S,A,.1F_1,\sqrt{.02}F_2), \enspace
    z_e=(\sqrt2S,Q,.1P_1,\sqrt{.02}P_2).
\end{equation*}
}
The scaling gives both views coordinate variances \((2,1,.01,.02)\), while the corresponding correlations are \((1,\rho,\gamma,\gamma)\). This separates variance from sharing: the weakly shared factors also have low variance. The audit objective is to remove \(S,A\) while
retaining \(F_1,F_2\).
We analytically compute optimal linear recovery from each retained representation (Table~\ref{tab:synthetic}(a)). The operators and population residuals are given in Appendix~\ref{app:synthetic_operators}, Table~\ref{tab:synthetic_operators}.

Table~\ref{tab:synthetic}(a) compares how different residualization operators recover $A, F_1,$ and $F_2$. The results highlight a fundamental trade-off in information removal. Exact erasure methods remove any component of $z_r$ whose covariance with $z_e$ is non-zero. As $\operatorname{Cov}(z_r,z_e)$ is full rank both $F_i$ are discarded. Shared/private decomposition methods yield different results: D-CCA keeps substantial information about \(A\) (\(R_A^2=.857\)), while JIVE with \(q=1\) and AJIVE with initial ranks \((4,4)\) retain \(A\) completely. With \(q=2\), JIVE’s reconstruction objective instead yields an individual component that discards the low-variance \(F_1\) while retaining \(A\), illustrating a difference between reconstruction and the specified audit objective. Instead, the SCCR operator removes \(S,A\) and preserves both \(F_i\) for \(\gamma<\tau<\rho\), despite their nonzero covariance with the expert proxies \(P_i\). This distinction is clinically important because complementary signals can remain correlated with expert views. For example, in predicting antidepressant treatment dropout from EHR data, an expert view may be restricted to psychiatric codes, while a raw model sees all diagnostic groups. Broader non-psychiatric signals $F_i$ (e.g., chronic pain or metabolic comorbidity) may be predictive of antidepressant treatment dropout and not explicitly part of the psychiatric expert view. However, these signals may still be weakly proxied in the expert view through $P_i$ due to overlapping comorbidity patterns, medication refills, or increased EHR contact for complex patients.

\subsection{Theoretical Properties}
\label{sec:theory}

These results concern the SCCR operator rather than the prediction head. We prove in Appendix~\ref{app:stability} that the operator removes exactly those canonical directions with correlations $s_i \ge \tau$, ensuring that no correlation above $\tau$ remains in the resulting residual embedding. We next show that the calibrated
frozen operator is stable to small covariance perturbations.

\begin{theorem}[Frozen SCCR is stable under calibration perturbations]
\label{thm:stability_main}
Let \(T\) be the population SCCR operator and let \(\widehat T\) be the frozen
operator computed from a calibration set using the same OAS-selected branchwise ridge
rule. Under the local OAS regularity, population conditioning, and strict
threshold-gap assumptions stated formally in Appendix~\ref{app:stability}, if the
calibration covariance error is sufficiently small, then the calibrated operator
removes the same number of canonical directions as the population operator and
{\small\[
\|\widehat T-T\|_2
\le
C\left(
\|\widehat S_{rr}-S_{rr}\|_2
+
\|\widehat S_{ee}-S_{ee}\|_2
+
\|\widehat S_{re}-S_{re}\|_2
\right)
\]}
for a constant \(C>0\). Consequently, \(\widehat T\to T\) whenever the calibration
covariance blocks converge.
\end{theorem}
The exact statement and proof are given in Appendix~\ref{app:stability}.

\begin{corollary}[Stability of the frozen inference map]
\label{cor:stability_main}
Under the assumptions of Theorem~\ref{thm:stability_main}, let
\(g(z)=(z-\mu_r)T\) and
\(\widehat g(z)=(z-\widehat\mu_r)\widehat T\), where
\(\mu_r\) and \(\widehat\mu_r\) are the population and calibration
raw-branch means. For every fixed \(z\in\mathbb{R}^{d_r}\),
\[
\begin{aligned}
\|\widehat g(z)-g(z)\|_2
&\le \|z-\mu_r\|_2\|\widehat T-T\|_2\\
&\quad+\|\widehat\mu_r-\mu_r\|_2\|\widehat T\|_2.
\end{aligned}
\]
Thus \(\widehat g(z)\to g(z)\) whenever the calibration means and
covariance blocks converge. The detailed bound and proof are given in
Corollary~\ref{cor:frozen_appendix} of Appendix~\ref{app:stability}.
\end{corollary}

This result supports freezing SCCR after calibration. Under the theorem's assumptions, including a gap around the threshold \(\tau\), sufficiently small calibration perturbations produce small changes in the residualization operator. Once frozen, the residual map is deterministic at inference.

\begin{table*}[!t]
\centering
\auditTableStyle
\setlength{\abovecaptionskip}{\baselineskip}
\setlength{\belowcaptionskip}{\baselineskip}
\caption[Linear Factor Recovery]{
Synthetic Linear Factor Recovery.
(a) Analytic values.
(b) Image experiments: mean $\pm$ sample SD across five seeds.
}
\label{tab:synthetic}
\label{tab:toy_tau_probe_results}

\begin{tabular*}{\textwidth}
{@{\extracolsep{\fill}}l*{9}{c}@{}}
\toprule
\textbf{(a)} & Raw
& Feature proj.
& OLS
& Exact LEACE
& D-CCA
& JIVE$_1$
& JIVE$_2$
& AJIVE$_{4,4}$
& \textbf{SCCR}\\
\midrule
$R_A^2\downarrow$
& 1 & 0 & .510 & 0 & .857 & 1 & 1 & 1 & \textbf{0}\\
$R_{F_1}^2\uparrow$
& 1 & 0 & .910 & 0 & .977 & 1 & 0 & 1 & \textbf{1}\\
$R_{F_2}^2\uparrow$
& 1 & 0 & .910 & 0 & .977 & 1 & 1 & 1 & \textbf{1}\\
\bottomrule
\end{tabular*}

\begingroup
\setlength{\tabcolsep}{1.4pt}
\begin{tabular*}{\textwidth}
{@{\extracolsep{\fill}}l*{10}{c}@{}}
\toprule
\textbf{(b)} & \multicolumn{5}{c}{Common frozen embeddings}
& \multicolumn{5}{c}{End-to-end learned embeddings}\\
\cmidrule(lr){2-6}\cmidrule(lr){7-11}
& Raw
& D-CCA
& JIVE
& AJIVE$_{\mathrm{gap}}$
& AJIVE$_{4,4}$
& OLS
& R-LEACE
& VCCA
& DSN-MMD
& \textbf{SCCR}\\
\midrule
$R_A^2\downarrow$
& \auditsd{.826}{.021}
& \auditsd{.451}{.039}
& NC
& \auditsd{.445}{.014}
& \auditsd{.013}{.014}
& \auditsd{.669}{.272}
& \auditsd{.000}{.001}
& \auditsd{.841}{.267}
& \auditsd{.068}{.091}
& \auditsd{\mathbf{.051}}{\mathbf{.092}}\\
$R_{F_1}^2\uparrow$
& \auditsd{.693}{.028}
& \auditsd{.690}{.029}
& NC
& \auditsd{-.000}{.001}
& \auditsd{.010}{.013}
& \auditsd{.628}{.022}
& \auditsd{.001}{.003}
& \auditsd{.110}{.085}
& \auditsd{.315}{.215}
& \auditsd{\mathbf{.708}}{\mathbf{.044}}\\
$R_{F_2}^2\uparrow$
& \auditsd{.634}{.032}
& \auditsd{.632}{.030}
& NC
& \auditsd{.002}{.004}
& \auditsd{.010}{.011}
& \auditsd{.593}{.052}
& \auditsd{.003}{.008}
& \auditsd{.049}{.041}
& \auditsd{.157}{.169}
& \auditsd{\mathbf{.644}}{\mathbf{.082}}\\
\bottomrule
\end{tabular*}
\par
\endgroup
\end{table*}

\subsection{SCCR on Synthetic Data}

\label{sec:toy}
We extend the construction in Section~\ref{sec:ped-example} to image inputs and learned embeddings, testing the same removal and retention objective in the full training pipeline.

\begin{figure}[!b]
\centering
\includegraphics[scale=0.9]{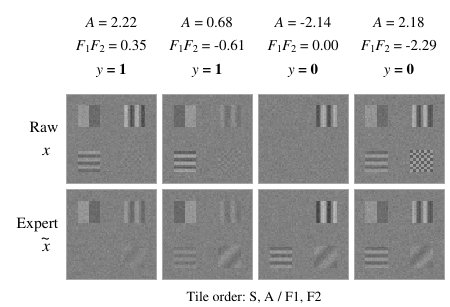}%
\caption{
Synthetic paired-view examples. The four raw patches encode
$S,A,F_1,F_2$, from upper left to lower right. The expert view contains
$S$ and noisy proxies of $A,F_1,F_2$, with correlations $.95,.30,.30$,
respectively, before independent pixel noise. %
}

\label{fig:synthetic_data_and_extra}
\end{figure}
\paragraph{Data generation.}
We use the same factors and expert-proxy definitions, now with \(\rho=.95\), \(\gamma=.3\), and rendering coefficients $c_r=(S,A,F_1,F_2)$ and $c_e=(S,Q,P_1,P_2).$ %
Unlike the theoretical example, all coefficients have unit variance. This avoids attenuating \(F_1,F_2\) in the images, while the larger \(\rho\) supplies a less noisy expert
proxy of \(A\). Each coefficient multiplies a fixed texture in a separate \(16\times16\) tile of a \(64\times64\) grayscale image. Thus, the factor values control the texture amplitudes. Independent pixel noise is added to each view (Figure~\ref{fig:synthetic_data_and_extra}). These rendering coefficients are distinct from the learned embeddings. The binary label is \(y=\mathbb I\{A+F_1F_2>0\}\), so prediction depends on both \(A\) and the interaction between \(F_1,F_2\). At each epoch, we generate 4,096 new raw--expert image pairs. We use separate, fixed sets of 4,096 pairs for operator calibration, 1,024 for checkpoint and probe selection, and 4,096 for final testing. SCCR is trained only on \(y\), without additional supervision on the underlying factors. %

\paragraph{Models and training.}
We train the dual-branch architecture from Section~\ref{sec:method} on the binary label \(y\).%
We use EfficientNetV2-S \cite{pmlr-v139-tan21a} initialized with ImageNet-pretrained weights \cite{5206848} for both branch encoders \(f_{\mathrm{raw}}\) and \(f_{\mathrm{expert}}\), with all encoder weights unfrozen and fine-tuned during training. %
We compare recovery of \(A,F_1,F_2\) with end-to-end OLS and regularized LEACE-style adaptations, unsupervised VCCA-private, and DSN-MMD. We also apply JIVE, AJIVE, and D-CCA to common frozen embeddings from the dual-branch baseline. For all models we train for 100 epochs with Adam, learning rate \(10^{-4}\), batch size 512, 64-dimensional embeddings, and \(\tau=0.4\) for SCCR. We note that these baselines were not originally formulated for the specific audit objective considered here: using an expert view as a reference to suppress strongly shared linear directions while retaining weakly shared predictive information. We adapt their retained representations to this objective and describe the implementation details in Appendix~\ref{apptrain}. 

\paragraph{Evaluation via linear probes.}
After training, we freeze the encoders and operators and fit linear and quadratic probes to recover \(A,F_1,F_2\), and \(F_1F_2\) from the raw, expert, and retained representations. We fit probes on training data, select their regularization on validation data, and report test-set \(R^2\). The audit objective is to reduce linear recovery of \(A\) while retaining \(F_1,F_2\). Quadratic probes test whether information remains accessible through nonlinear relationships.

\paragraph{Recovery Results.}
Table~\ref{tab:toy_tau_probe_results}(b) shows the expected pattern. Running for five different seeds, SCCR gives residual linear \(R_A^2=.051\pm.092\), while retaining \(R_{F_1}^2=.708\pm.044\) and \(R_{F_2}^2=.644\pm.082\); its native-head AUROC is \(0.9928\pm0.0025\). Values are mean \(\pm\) standard deviation across seeds. Quadratic recovery of \(A\) remains high (\(.818\pm.057\)), so these results support reduced linear recovery.%

We evaluate SCCR on dermatologist-annotated skin images and deidentified data obtained through hospital collaborations: longitudinal EHRs from four regional hospitals (Sites A–D) and coronary angiograms from Site A. Both hospital studies received IRB approval with a waiver of consent.
\subsection{Skin-lesion classification}
 
We use the SkinCon \citep{skincon} dataset to evaluate whether SCCR retains predictive information beyond an expert predictor. We then trace that remaining information to an omitted clinical concept. We examine three questions: whether the residual improves diagnosis prediction, whether an omitted concept explains part of this gain, and whether restoring that concept changes its linear accessibility after retraining.
\paragraph{Data and expert views}
SkinCon provides dermatologist annotations for 48 binary concepts on images from Fitzpatrick17k \citep{fitz} and DDI \citep{ddi}. Preprocessing details for the image can be found in Appendix \ref{app:skincon}.  We predict malignant versus benign or non-neoplastic lesions using the image as the raw view and 45 annotations as the expert view, withholding Crust, Erosion, and Scale. We then restore Crust to test whether a known clinical descriptor explains part of the residual’s predictive value. %
\begin{table*}[!t]
\centering
\auditTableStyle
\setlength{\abovecaptionskip}{\baselineskip}
\setlength{\belowcaptionskip}{\baselineskip}
\caption[Clinical Prediction and Audit Results]{
Clinical prediction and residual audits. (a) skin-lesion classification, SCCR $\tau=.40$; (b) antidepressant treatment dropout, held-out-site AUROC mean $\pm$ sample SD over five seeds; (c) coronary artery disease, one seed. Brackets: pointwise 95\% paired-bootstrap intervals conditional on fitted models.
}
\label{tab:clinical_audits}

\begin{tabular*}{\textwidth}
{@{\extracolsep{\fill}}l*{5}{c}@{}}
\toprule
\textbf{(a)} Seed
& Diagnosis gain $G$
& Gain over OLS
& Fixed residual $\Delta G_C$
& Retraining $\Delta D_C$
& External DDI gain $G$\\
\midrule
909
& \auditci{.073}{.027}{.114}
& \auditci{.105}{.039}{.174}
& \auditci{.020}{.008}{.034}
& \auditci{.016}{.002}{.029}
& \auditci{-.187}{-.266}{-.109}\\
910
& \auditci{.051}{.003}{.093}
& \auditci{.050}{.005}{.103}
& \auditci{.009}{-.001}{.020}
& \auditci{.045}{.003}{.088}
& \auditci{-.229}{-.309}{-.147}\\
\bottomrule
\end{tabular*}

\begin{tabular*}{\textwidth}
{@{\extracolsep{\fill}}l*{6}{c}@{}}
\toprule
\textbf{(b)} & \multicolumn{3}{c}{Dropout prediction}
& \multicolumn{3}{c}{Linear code recovery}\\
\cmidrule(lr){2-4}\cmidrule(lr){5-7}
Site
& Expert-only
& Raw-only
& SCCR ($\tau=.40$)
& Included: raw
& Included: residual
& Omitted: residual\\
\midrule
B
& \auditsd{.565}{.001}
& \auditsd{.619}{.001}
& \auditsd{.613}{.002}
& \auditsd{.885}{.021}
& \auditsd{.640}{.020}
& \auditsd{.920}{.050}\\
C
& \auditsd{.574}{.002}
& \auditsd{.611}{.003}
& \auditsd{.605}{.002}
& \auditsd{.874}{.017}
& \auditsd{.674}{.052}
& \auditsd{.933}{.063}\\
D
& \auditsd{.555}{.002}
& \auditsd{.585}{.006}
& \auditsd{.578}{.008}
& \auditsd{.927}{.016}
& \auditsd{.686}{.072}
& \auditsd{.864}{.103}\\
\bottomrule
\end{tabular*}

\begin{tabular*}{\textwidth}
{@{\extracolsep{\fill}}lccc@{}}
\toprule
\textbf{(c)} $\tau$
& Native AUROC
& Removed directions
& Residual gain over own expert embedding\\
\midrule
.10
& \auditci{.792}{.775}{.811}
& 4
& \auditci{-.0009}{-.0055}{.0040}\\
.30
& \auditci{.788}{.770}{.807}
& 0
& \auditci{-.0005}{-.0057}{.0049}\\
\bottomrule
\end{tabular*}
\end{table*}
\paragraph{Models and training.} We use an ImageNet-pretrained EfficientNetV2-S image encoder and a multilayer perceptron for the expert annotations, with 32-dimensional projected embeddings. SCCR is trained end to end on the diagnosis label with \(\tau=.4\), without an auxiliary concept-prediction loss. We compare SCCR with the same baseline methods as in Section~\ref{sec:toy}. We train for two seeds, training details can be found in Appendix \ref{app:skincon}.
\paragraph{Results: Predictive gain and a linear response to Crust restoration}
SCCR’s residual reduces diagnosis log loss by .073 and .051 across two seeds, exceeding the gains from end-to-end OLS (Table~\ref{tab:clinical_audits}). 
With the residual fixed, supplying the true Crust annotations reduces this gain more than supplying shuffled annotations, by $0.020$ and $0.009$; the interval includes zero for the second seed. This suggests that an omitted clinical descriptor explains part of the additional predictive value. After retraining, true rather than shuffled restoration also increases the drop in linear Crust recovery from raw to residual embeddings, although nonlinear probes still recover Crust. Residual gains are negative on DDI ($-0.187$ and $-0.229$). Thus, the audit provides partial evidence linking the predictive gain to Crust and shows a linear response to its restoration. The predictive benefit does not transfer to DDI. 
\subsection{Antidepressant Treatment Dropout from EHR Data}

\label{sec:trd_dropout}
\paragraph{Data and expert views} We study antidepressant treatment dropout using longitudinal EHRs from Sites A–D. The raw view contains counts of CPT, ICD, and RxNorm codes and demographics. The expert view contains the same demographics and a clinician-curated subset of psychiatric codes. 
\paragraph{Models and training}
We train and validate on Site A and test on Sites B–D, asking whether the broader coded record adds predictive information beyond the selected psychiatric codes. We also omit psychotherapy code CPT 90806, then restore its true or shuffled values to examine what the residual retains and whether this code explains its predictive gain. Model and training details are provided in Appendix~\ref{app:psyctraj-details}.

\paragraph{Results: Predictive gain across three held-out hospitals}
Across the three held-out hospitals, SCCR at \(\tau=.40\) improves AUROC over the psychiatric view by .048, .031, and .023, while remaining close to raw-only prediction (Table~\ref{tab:clinical_audits}). In the omission audit, the psychotherapy code remains recoverable from the residual (.864–.933 AUROC), while linear recovery of included psychiatric codes falls from .874–.927 to .640–.686. Nonlinear recovery remains substantial. Thus, SCCR retains predictive value beyond the psychiatric view across three held-out hospitals while reducing linear recovery of included codes. The omitted psychotherapy code remains recoverable, but adding it to the expert predictor with the residual fixed leaves the predictive gain largely unchanged.

\subsection{Coronary angiograms} %
\paragraph{Data and expert views} We study CAD prediction using coronary angiograms from Site A only. Each example pairs a single angiogram frame with its right-coronary-artery vessel mask, generated by a U-Net \citep{10.1007/978-3-319-24574-4_28} trained on ARCADE \citep{arcade}. The mask provides a summary of vessel morphology while omitting grayscale intensity. We ask how predictions depend on the mask and the full image. 
\begin{figure}[!b]
\centering
\includegraphics[width=\linewidth]{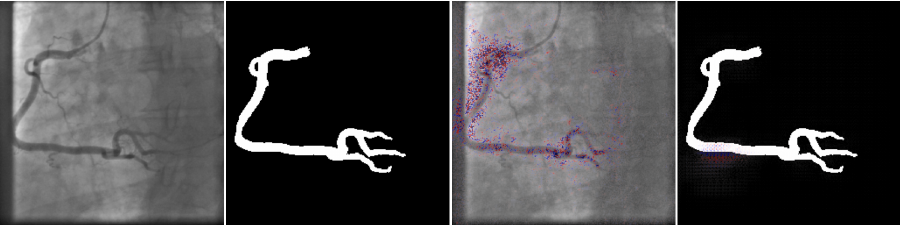}\\[2pt]
\includegraphics[width=\linewidth]{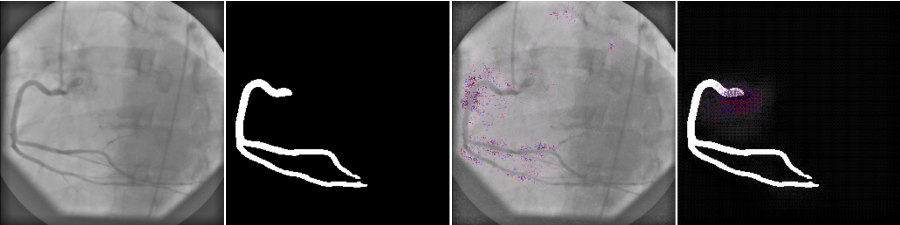}
\caption{Coronary angiogram examples. From left: raw image,
vessel mask, and Grad$\times$Input overlays for the raw and
mask inputs, respectively.} 
\label{fig:angio_saliency}
\end{figure}
\paragraph{Models and training.} For each threshold \(\tau\), we train a separate model and evaluate on a held-out video split. Branch perturbations and Grad×Input assess dependence on the raw image and mask; implementation details are provided in the Appendix \ref{app:angio-sccr}
\paragraph{Results: Separating branch reliance from predictive gain}
Reliance on the two branches differs across separately trained models. At \(\tau=.5\), AUROC is .790; zeroing the mask reduces it by .261, while zeroing the raw image has essentially no effect. At \(\tau=.7\), AUROC is .835; zeroing the raw image reduces it by .387, while zeroing the mask has negligible effect. These controls show a shift from mask-based to image-based prediction. However, the residual audits at \(\tau=.1\) and \(.3\) show no clear gain beyond the corresponding expert embeddings (Table~\ref{tab:clinical_audits}). The audit separates which input a model uses from whether its residual improves prediction beyond the expert view. %
Saliency maps (Figure~\ref{fig:angio_saliency}) suggest that the raw branch attends to grayscale intensity and contrast patterns not represented in the binary expert mask. These cues may reflect vessel visibility and local contrast, but because the model processes single frames, they describe static appearance rather than temporal mechanisms such as delayed filling or collateral flow. Their clinical relevance and predictive value beyond the expert view remain unresolved.

\section{Discussion and Conclusion}
SCCR audits raw representations by removing strongly shared linear directions while retaining weakly shared directions that may be necessary to maintain prediction quality. Once trained, the residual can be probed to test what predictive information remains beyond the expert predictor. We show that our approach is distinct in being able to do this compared to a large number of other baselines that fail to do both, and demonstrate value on both an open clinical dataset and two real, clinical problems. SCCR only requires one parameter, the threshold \(\tau\), which can be found by identifying the minimum value at which the prediction quality of our dual branch architecture is equivalent to simply training on the original raw data. In terms of future work, SCCR does not guarantee removal of arbitrary nonlinear dependence. 
SCCR does not guarantee removal of arbitrary nonlinear dependence. Its linear formulation provides an explicit removal rule, and probes of known concepts let us measure how their accessibility changes with the expert view. This makes linear removal useful for testing specific audit questions, while nonlinear probes reveal information that remains recoverable. Future work should extend removal to nonlinear dependence while preserving weakly shared predictive information.

\bibliographystyle{plainnat}  
\bibliography{example_paper}

\clearpage
\onecolumn
\appendix
\section*{Impact Statement}

This work is motivated by settings in which practitioners already have an expert view of what a model should use, for example in healthcare, credit, or other high-stakes decision problems. The main potential benefit is auditability: the method can help determine whether an expert view is sufficient and, when it is not, what complementary signal a model uses beyond it. This can help surface clinically meaningful information, but it can also reveal spurious shortcuts or hidden confounding that should not be trusted. At the same time, the method should not be taken as a guarantee that all unwanted information has been removed. It targets shared linear structure. Any deployment in high-stakes settings therefore requires domain oversight, task-specific validation, and careful review of the residual branch before drawing conclusions about fairness, safety, or causality.

\section{Residualization operators for the Gaussian example}
\label{app:synthetic_operators}

This appendix defines the transformations compared in
Table~\ref{tab:synthetic}(a) and derives their population residuals
and factor recovery. The methods use the same paired embeddings,
but solve different problems. We evaluate their retained outputs
against the stated audit objective: removing $S,A$ while retaining
$F_1,F_2$. No prediction label, encoder training, or fitted probe
is used in these analytic calculations.

\subsection{Construction and notation}
\label{app:gaussian_construction}

As in Section~\ref{sec:ped-example}, let
$S,A,F_1,F_2,N_A,N_1,N_2$ be mutually independent standard Gaussian
variables. Set $\rho=.7$, $\gamma=.3$, and
\[
Q=\rho A+\sqrt{1-\rho^2}N_A,\qquad
P_i=\gamma F_i+\sqrt{1-\gamma^2}N_i,\quad i=1,2.
\]
The raw and expert embeddings are row vectors:
\[
z_r=(\sqrt2S,A,.1F_1,\sqrt{.02}F_2),\qquad
z_e=(\sqrt2S,Q,.1P_1,\sqrt{.02}P_2).
\]
Thus $S$ is exactly shared, $Q$ is a strongly correlated proxy of
$A$, and the $P_i$ are weakly correlated proxies of the $F_i$.
The proxy noise is part of the expert embedding; it is not an
additional matrix perturbation that is removed before fitting.

Matrices $Z_r,Z_e$ contain paired embeddings as rows. Following
Section~\ref{sec:method}, their centered versions are
\[
Z_{r,c}=Z_r-\mathbf1\mu_r,\qquad
Z_{e,c}=Z_e-\mathbf1\mu_e.
\]
The population means are zero. In this appendix, the second-moment
symbols denote their population counterparts:
\[
S_{rr}=\mathbb E[z_r^\top z_r],\quad
S_{ee}=\mathbb E[z_e^\top z_e],\quad
S_{re}=\mathbb E[z_r^\top z_e],\quad S_{er}=S_{re}^\top.
\]
Because each proxy has unit variance and different factor pairs
are independent,
\begin{equation}
\begin{aligned}
S_{rr}=S_{ee}&=\operatorname{diag}(2,1,.01,.02),\\
S_{re}&=\operatorname{diag}(2,\rho,.01\gamma,.02\gamma)
       =\operatorname{diag}(2,.7,.003,.006).
\end{aligned}
\label{eq:gaussian_moments}
\end{equation}
All three matrices have rank four. The two blocks have the same
total variance, $\operatorname{tr}(S_{rr})=
\operatorname{tr}(S_{ee})=3.03$.

We write $\dagger$ for the Moore--Penrose pseudoinverse,
$\operatorname{col}(D)$ for a matrix's column space, and
$\mathcal T_k(D)$ for its rank-$k$ truncated singular value
decomposition (SVD). The Frobenius norm satisfies
$\|D\|_F^2=\sum_{ij}D_{ij}^2$.
Population calculations use exact moments and zero ridge.
In particular, $T$ denotes the population SCCR operator here;
$T^\star$ remains the frozen operator fitted on calibration data
in the main method.

\subsection{Optimal linear recovery}
\label{app:gaussian_recovery}

For a target $W\in\{A,F_1,F_2\}$ and a retained row representation
$z$, define
\begin{equation}
R_W^2(z)=1-
\frac{\min_{\beta_0,\beta}\mathbb E[(W-\beta_0-z\beta)^2]}
{\operatorname{Var}(W)}.
\label{eq:gaussian_recovery_definition}
\end{equation}
This measures recovery from the retained representation alone.
All targets have variance one. For centered $W,z$, the optimal
intercept is zero. Writing $S_{zz}=\mathbb E[z^\top z]$ and
$k_W=\mathbb E[z^\top W]$, expansion of the prediction error gives
\[
\mathbb E[(W-z\beta)^2]
=\operatorname{Var}(W)-2\beta^\top k_W+\beta^\top S_{zz}\beta.
\]
Every zero-variance linear combination of $z$ has zero covariance
with $W$, so $k_W$ lies in the range of $S_{zz}$.
The normal equations $S_{zz}\beta=k_W$ therefore admit the solution
$\beta=S_{zz}^{\dagger}k_W$, giving
\begin{equation}
R_W^2(z)=\frac{k_W^\top S_{zz}^{\dagger}k_W}
{\operatorname{Var}(W)}.
\label{eq:gaussian_recovery_covariance}
\end{equation}
This includes singular outputs: for $z=0$, the optimal prediction
is the mean and $R_W^2=0$.

If the only relevant residual coordinate is a nonzero multiple of
a scalar $u$ with positive variance, independent of all other
residual coordinates, then
\begin{equation}
R_W^2=\frac{\operatorname{Cov}(W,u)^2}
{\operatorname{Var}(W)\operatorname{Var}(u)}.
\label{eq:gaussian_scalar_recovery}
\end{equation}
The nonzero coordinate scale cancels. Thus the raw embedding has
$(R_A^2,R_{F_1}^2,R_{F_2}^2)=(1,1,1)$: the three factors are
recovered by dividing their coordinates by $1,.1,\sqrt{.02}$.
Small variance alone does not reduce optimal population recovery,
although it changes the cost of reconstruction errors.

Each population output below is a fixed linear function of the
Gaussian sources. Its conditional mean for each target is therefore
linear, so a nonlinear squared-error predictor cannot improve these
particular population values. This fact does not extend to the
learned nonlinear image embeddings or to predicting the interaction
$F_1F_2$.

\subsection{Feature-space projection}
\label{app:gaussian_feature_projection}

This baseline projects raw feature coordinates onto the orthogonal
complement of the expert feature span:
\begin{equation}
P_e=Z_{e,c}^{\dagger}Z_{e,c},\qquad
Z_{\mathrm{FP}}=Z_{r,c}(I-P_e).
\label{eq:gaussian_feature_projection}
\end{equation}
Here $P_e$ acts in feature space, not sample space. This construction
assumes the displayed raw and expert feature coordinates have been
identified; it is not a general alignment method for unrelated
embedding coordinates.

Since $S_{ee}$ is positive definite, the population expert support
spans all four feature directions. Therefore $P_e=I_4$, and
\[
z_{\mathrm{FP}}=(0,0,0,0),\qquad
(R_A^2,R_{F_1}^2,R_{F_2}^2)=(0,0,0).
\]
A centered Gaussian sample with at least five observations also has
full column rank almost surely. This feature-space projection is
different from the sample-space regression projection used by OLS.

\subsection{Full OLS residualization}
\label{app:gaussian_ols}

\paragraph{Definition and transformation.}
OLS predicts the raw embedding from the expert embedding by solving
\[
\min_{B_{\mathrm{OLS}}}
\|Z_{r,c}-Z_{e,c}B_{\mathrm{OLS}}\|_F^2.
\]
The minimum-norm coefficient and fitted residual are
\begin{equation}
\begin{aligned}
B_{\mathrm{OLS}}
 &=(Z_{e,c}^\top Z_{e,c})^\dagger Z_{e,c}^\top Z_{r,c},\\
Z_{\mathrm{OLS}}
 &=Z_{r,c}-Z_{e,c}B_{\mathrm{OLS}}
  =(I-Z_{e,c}Z_{e,c}^{\dagger})Z_{r,c}.
\end{aligned}
\label{eq:gaussian_ols_map}
\end{equation}
The population normal equations give
\[
B_{\mathrm{OLS}}^{\mathrm{pop}}
=S_{ee}^{-1}S_{er}=\operatorname{diag}(1,\rho,\gamma,\gamma).
\]
A new residual uses both observations,
$z_{\mathrm{OLS}}=z_r-z_eB_{\mathrm{OLS}}^{\mathrm{pop}}$, yielding
\begin{equation}
z_{\mathrm{OLS}}
=(0,A-\rho Q,.1(F_1-\gamma P_1),
\sqrt{.02}(F_2-\gamma P_2)).
\label{eq:gaussian_ols_output}
\end{equation}

\paragraph{Recovery and interpretation.}
For a standardized pair $u,v$ with correlation $0\le c<1$, let
$u_{\mathrm{OLS}}=u-cv$. Then
\[
\operatorname{Var}(u_{\mathrm{OLS}})=1-c^2,\qquad
\operatorname{Cov}(u,u_{\mathrm{OLS}})=1-c^2,
\]
so Equation~\eqref{eq:gaussian_scalar_recovery} gives $R_u^2=1-c^2$.
Different factor pairs remain independent, and therefore
\[
(R_A^2,R_{F_1}^2,R_{F_2}^2)
=(1-\rho^2,1-\gamma^2,1-\gamma^2)=(.510,.910,.910).
\]
OLS makes its residual uncorrelated with the expert view:
$\mathbb E[z_{\mathrm{OLS}}^\top z_e]
=S_{re}-(B_{\mathrm{OLS}}^{\mathrm{pop}})^\top S_{ee}=0$.
It does not erase the underlying factor $A$, because the expert
contains only its noisy proxy. It also retains most of each $F_i$;
this is partial attenuation, not complete loss. If the expert view
is supplied alongside the residual, the original raw embedding can
be reconstructed as $z_r=z_{\mathrm{OLS}}+
 z_eB_{\mathrm{OLS}}^{\mathrm{pop}}$. Thus these residual-only
recovery values do not establish a failure of combined prediction.

\subsection{Exact LEACE}
\label{app:gaussian_leace}

\paragraph{Definition and transformation.}
We use the continuous-concept covariance formulation of
LEACE~\citep{10.5555/3666122.3669006}, treating the entire expert
embedding as the concept to erase. For centered row embeddings,
its minimum-distortion problem is
\begin{equation}
\min_{T_{\mathrm{LEACE}}}
\mathbb E\|z_rT_{\mathrm{LEACE}}-z_r\|_2^2
\quad\text{subject to}\quad
T_{\mathrm{LEACE}}^\top S_{re}=0.
\label{eq:gaussian_leace_objective}
\end{equation}
The constraint makes every expert coordinate linearly unpredictable
from the transformed raw embedding under squared error. Unlike OLS,
once fitted this map acts on the raw observation alone.

For positive definite $S_{rr}$, use its Cholesky factor and define
\[
S_{rr}=L_rL_r^\top,\qquad C_{\mathrm L}=L_r^{-1}S_{re},\qquad
\Pi_{\mathrm L}=C_{\mathrm L}C_{\mathrm L}^{\dagger}.
\]
Here $\Pi_{\mathrm L}$ projects onto the whitened cross-covariance
column space. The row-action LEACE map is
\begin{equation}
T_{\mathrm{LEACE}}
=L_r^{-\top}(I-\Pi_{\mathrm L})L_r^\top.
\label{eq:gaussian_leace_map}
\end{equation}
To see this, set $D=L_r^\top T_{\mathrm{LEACE}}$. The objective becomes
$\|D-L_r^\top\|_F^2$ and the constraint becomes
$D^\top C_{\mathrm L}=0$, or equivalently $\Pi_{\mathrm L}D=0$.
Projection of each column of $L_r^\top$ onto this feasible subspace
gives $D=(I-\Pi_{\mathrm L})L_r^\top$, proving the formula.
The corresponding column-action map is its transpose; in general
it is not a symmetric Euclidean projection.

\paragraph{Residual and recovery.}
Here $S_{re}$ is full rank, so $C_{\mathrm L}$ is full rank and
$\Pi_{\mathrm L}=I_4$. Consequently,
\[
T_{\mathrm{LEACE}}=0,\qquad z_{\mathrm{LEACE}}=0,\qquad
(R_A^2,R_{F_1}^2,R_{F_2}^2)=(0,0,0).
\]
Equivalently, the constraint $T_{\mathrm{LEACE}}^\top S_{re}=0$
forces $T_{\mathrm{LEACE}}=0$ because $S_{re}$ is invertible.
The nonzero covariance of each $F_i$ with its proxy is sufficient
for erasure here, even though that covariance is small. These are
exact, unregularized LEACE values, not the regularized LEACE-style
adaptation used with learned embeddings.

\subsection{Decomposition-based CCA}
\label{app:gaussian_dcca}

\paragraph{Definition and common scores.}
D-CCA~\citep{Shu02012020} separates common and distinctive variation;
it is different from Deep CCA. We apply its population signal
decomposition to all four canonical pairs, treating the observed
embeddings, including the proxy noise, as the signal vectors.
We do not add covariance denoising or empirical signal-rank selection.

Unregularized whitening of Equation~\eqref{eq:gaussian_moments}
gives canonical pairs $(S,S),(A,Q),(F_1,P_1),(F_2,P_2)$ with
correlations $(1,\rho,\gamma,\gamma)$. For a standardized scalar
pair $(u,v)$ with correlation $c\in(0,1)$, D-CCA first selects a
unit-variance common direction maximizing
\[
\operatorname{Corr}(u,h)^2+\operatorname{Corr}(v,h)^2.
\]
The sum and difference of $u,v$ are orthogonal, with variances
$2(1+c)$ and $2(1-c)$. The maximizing direction is therefore
$h=(u+v)/\sqrt{2(1+c)}$.
Write the common score as $a_c(u+v)$ and the distinctive scores as
\[
d_r=u-a_c(u+v),\qquad d_e=v-a_c(u+v).
\]
D-CCA requires the two distinctive scores to be uncorrelated:
\[
0=\operatorname{Cov}(d_r,d_e)
=c-2a_c(1+c)+2a_c^2(1+c).
\]
Solving this quadratic and choosing the branch whose common variance
increases from zero with $c$ gives
\begin{equation}
a_c=\frac12\left(1-\sqrt{\frac{1-c}{1+c}}\right),\qquad
 d_c(u,v)=u-a_c(u+v).
\label{eq:gaussian_dcca_score}
\end{equation}
The common variance is $1-\sqrt{1-c^2}$.

\paragraph{Reconstruction and residual.}
To state the full matrix transformation, use the unregularized
Cholesky factors $S_{rr}=L_rL_r^\top$, $S_{ee}=L_eL_e^\top$ and
the SVD $L_r^{-1}S_{re}L_e^{-\top}=U\operatorname{diag}(s)V^\top$.
The standardized canonical score matrices are
\[
\Xi_r=Z_{r,c}L_r^{-\top}U,\qquad
\Xi_e=Z_{e,c}L_e^{-\top}V.
\]
The raw reconstruction loading matrix is $L_rU$, giving
\[
C_r=(\Xi_r+\Xi_e)\operatorname{diag}(a_{s_j})U^\top L_r^\top,
\qquad Z_{\mathrm{DCCA}}=Z_{r,c}-C_r.
\]
These expressions use population moments and their sample analogues
when fitting; no ridge is used in the analytic calculation.
In the displayed coordinate system this reduces to
\[
C_r=(Z_{r,c}+Z_{e,c})
\operatorname{diag}(\tfrac12,a_\rho,a_\gamma,a_\gamma).
\]
Define $D_A=d_\rho(A,Q)$ and $D_i=d_\gamma(F_i,P_i)$. Then
\begin{equation}
z_{\mathrm{DCCA}}=(0,D_A,.1D_1,\sqrt{.02}D_2).
\label{eq:gaussian_dcca_output}
\end{equation}
As in OLS, this paired construction uses the expert observation
when forming a new distinctive output.

\paragraph{Recovery and interpretation.}
For $0\le c<1$, substitution into Equation~\eqref{eq:gaussian_dcca_score}
gives
\[
\operatorname{Var}(d_r)=1-c,\qquad
\operatorname{Cov}(u,d_r)=\frac{1-c+\sqrt{1-c^2}}2.
\]
It follows that
\begin{equation}
R_u^2(d_r)
=\frac{(1-c+\sqrt{1-c^2})^2}{4(1-c)}
=\frac{1+\sqrt{1-c^2}}2.
\label{eq:gaussian_dcca_recovery}
\end{equation}
Therefore
\[
R_A^2=\frac{1+\sqrt{.51}}2\approx.857,\qquad
R_{F_1}^2=R_{F_2}^2=\frac{1+\sqrt{.91}}2\approx.977.
\]
At $c=1$, the views coincide and the distinctive output is exactly
zero; Equation~\eqref{eq:gaussian_dcca_recovery} cannot be used at
that endpoint because its derivation divides by $1-c$.
This removes $S$ completely.

D-CCA makes $d_r$ and $d_e$ uncorrelated, not $d_r$ and the original
expert score $v$. In fact,
\[
\operatorname{Cov}(d_r,v)
=\frac{c-1+\sqrt{1-c^2}}2>0\quad(0<c<1).
\]
Its substantial recovery of $A$ is therefore consistent with its
objective. It is not a failure to satisfy the D-CCA constraint.

\subsection{JIVE with prescribed ranks}
\label{app:gaussian_jive}

\paragraph{Definition and transformation.}
JIVE~\citep{10.1214/12-AOAS597} approximates each centered block by
joint and individual components. With samples as rows, write
\[
Z_{r,c}=J_r+I_r+E_r,\qquad Z_{e,c}=J_e+I_e+E_e,
\]
where $E_r,E_e$ are reconstruction errors, and let $J=[J_r\ J_e]$.
For joint rank $q$ and individual ranks $(4-q,4-q)$, we use the
reconstruction problem
\begin{equation}
\begin{aligned}
\min_{J_r,J_e,I_r,I_e}\quad
 &\|Z_{r,c}-J_r-I_r\|_F^2+
   \|Z_{e,c}-J_e-I_e\|_F^2,\\
\text{subject to}\quad
 &\operatorname{rank}(J)\le q,\qquad
   \operatorname{rank}(I_r),\operatorname{rank}(I_e)\le4-q,\\
 &J^\top I_r=0,\qquad J^\top I_e=0.
\end{aligned}
\label{eq:gaussian_jive_objective}
\end{equation}
The last constraints separate joint and individual sample-score
spaces. They do not require $I_r^\top I_e=0$.
For a fixed joint space with orthogonal projector $H_J$, the
least-squares joint reconstruction and individual output are
\begin{equation}
J_r=H_JZ_{r,c},\qquad
I_r=\mathcal T_{4-q}((I-H_J)Z_{r,c}),
\label{eq:gaussian_jive_map}
\end{equation}
with analogous expressions for the expert block. The table evaluates
$I_r$, not the full remainder $Z_{r,c}-J_r$.
The derivations below concern global minima, not a guarantee for
an alternating SVD solver. There is no columnwise standardization
or whitening. Equal whole-block normalization multiplies the
population objective by $1/3.03$ and leaves its minimizers unchanged.

\paragraph{Population score spaces.}
A scalar score is a linear combination of the seven Gaussian
sources. Its inner product with another score is their covariance.
The block score spaces are
\[
\mathcal U_r=\operatorname{span}\{S,A,F_1,F_2\},\qquad
\mathcal U_e=\operatorname{span}\{S,Q,P_1,P_2\}.
\]
Their intersection is $\operatorname{span}\{S\}$. Indeed, equality
of a raw and an expert linear combination forces the coefficients
of $N_A,N_1,N_2$ to vanish; since $\rho,\gamma<1$, no other
expert score lies exactly in the raw score space.

\paragraph{Joint rank one.}
For $q=1$, the following population decomposition has zero error:
\[
\begin{aligned}
j_r=j_e&=(\sqrt2S,0,0,0),\\
i_r&=(0,A,.1F_1,\sqrt{.02}F_2),\\
i_e&=(0,Q,.1P_1,\sqrt{.02}P_2).
\end{aligned}
\]
It satisfies joint rank one, individual ranks three, and
joint--individual orthogonality. Zero is the global minimum of
the nonnegative reconstruction objective. At any zero-error solution,
each rank-four block is the sum of orthogonal components of ranks
one and three. Its joint space must consequently belong to both
block score spaces. Their intersection forces it to be the $S$
direction. Thus every global optimum has this raw individual output,
and
\[
(R_A^2,R_{F_1}^2,R_{F_2}^2)=(1,1,1).
\]

\paragraph{Joint rank two: a global lower bound.}
For $q=2$, the pooled fitted matrix has rank at most
$2+2+2=6$. To obtain a population bound, write
$\eta=(S,A,F_1,F_2,N_A,N_1,N_2)$ and define the coefficient matrix
$\Lambda\in\mathbb R^{7\times8}$ by $[z_r\ z_e]=\eta\Lambda$.
Since $\mathbb E[\eta^\top\eta]=I_7$, mean squared reconstruction
error equals squared Frobenius error between coefficient matrices.

For a pair of unit-variance scores $u,v$ with correlation $c$, the
sum and difference directions have pooled eigenvalues $1+c$ and
$1-c$. Applying their coordinate scales, the seven nonzero pooled
covariance eigenvalues are
\begin{equation}
\begin{aligned}
&4,\ 1+\rho,\ 1-\rho,\ .02(1+\gamma),\ .02(1-\gamma),\
.01(1+\gamma),\ .01(1-\gamma)\\
&\hspace{2em}=4,1.70,.30,.026,.014,.013,.007.
\end{aligned}
\label{eq:gaussian_jive_spectrum}
\end{equation}
The smallest corresponds to $F_1-P_1$. A rank-six fitted coefficient
matrix must lose at least its energy, $.007$. More explicitly,
let $\widehat\Lambda$ have rank at most six and take a unit vector
$h$ in the null space of $\widehat\Lambda^\top$. Then
\[
\|\Lambda-\widehat\Lambda\|_F^2
\ge\|h^\top\Lambda\|_2^2
=h^\top\Lambda\Lambda^\top h\ge.007.
\]
The smallest eigenvalue is simple. Equality therefore determines
the unique pooled fit obtained by deleting only that contrast.

\paragraph{Joint rank two: a feasible optimum.}
The bound is attained by
\begin{equation}
\begin{aligned}
j_r=j_e&=(\sqrt2S,0,.1(F_1+P_1)/2,0),\\
i_r&=(0,A,0,\sqrt{.02}F_2),\\
i_e&=(0,Q,0,\sqrt{.02}P_2).
\end{aligned}
\label{eq:gaussian_jive_q2_components}
\end{equation}
The joint scores are $S,F_1+P_1$, and the two individual spaces
are spanned by $A,F_2$ and $Q,P_2$. They are orthogonal to the
joint scores and have the prescribed ranks. The only errors are
$+.1(F_1-P_1)/2$ and $-.1(F_1-P_1)/2$ in the third coordinate,
with total error
\[
2\frac{.01}{4}\mathbb E[(F_1-P_1)^2]
=.01(1-\gamma)=.007.
\]
This proves global optimality. The unique pooled fit has rank-four
blocks with score-space intersection
$\operatorname{span}\{S,F_1+P_1\}$. The rank and orthogonality
constraints force both joint components to span this intersection,
so their individual components are also fixed across global minima.
Consequently,
\[
(R_A^2,R_{F_1}^2,R_{F_2}^2)=(1,0,1).
\]
Thus reconstruction discards the low-variance $F_1$ contrast while
retaining $A$. The fit need not satisfy a generative model with
independent reconstruction errors; the result concerns the stated
reconstruction objective.

\paragraph{Individual output, full remainder, and preprocessing.}
At the rank-two optimum, joint subtraction alone gives
\[
z_r-j_r=(0,A,.1(F_1-P_1)/2,\sqrt{.02}F_2).
\]
Its $F_1$ recovery is $(1-\gamma)/2=.35$, not zero. Rank-two
individual truncation removes this remaining contrast. The table's
zero therefore concerns the individual component specifically.

Whitening also changes the result. With all marginal variances set
to one, the smallest pooled eigenvalue is $1-\rho=.3$, associated
with $A-Q$. A rank-two joint fit can then use $S,(A+Q)/2$, leaving
raw individual output $(0,0,F_1,F_2)$ and recovery $(0,1,1)$.
This follows from the same rank-six lower bound and a feasible
fit attaining error $.3$.
A different rank and output can also succeed without whitening:
under the aggregate joint-rank constraint in
Equation~\eqref{eq:gaussian_jive_objective}, $q=3$ and individual
ranks $(1,1)$ admit the joint components
$(\sqrt2S,A,0,0)$ and $(\sqrt2S,Q,0,0)$, with individual components
containing only $F_2,P_2$. This fit attains the rank-five lower
bound $.013+.007=.020$, while its full raw-minus-joint remainder
retains both $F_i$ and removes $A$.
Each block's joint rank is two in this last control, so a convention
requiring joint rank three in every block excludes it.
These controls limit the claim to the stated ranks, feature scales,
and choice of retained output.

\subsection{AJIVE with initial ranks $(4,4)$}
\label{app:gaussian_ajive}

\paragraph{Initial signal extraction.}
AJIVE~\citep{10.1016/j.jmva.2018.03.008} first extracts each block's
signal space, estimates joint directions from their angles, and
then reconstructs joint and individual components. For block
$j\in\{r,e\}$ and initial rank $r_j$, write
\[
Z_{j,c}\approx U_{j,r_j}\Sigma_{j,r_j}V_{j,r_j}^\top,
\qquad Q_j=U_{j,r_j}.
\]
The columns of $Q_j$ are orthonormal sample-score directions.
Let $\sigma_{j,k}$ denote the block singular values. We use the
initial threshold
$t_j=(\sigma_{j,r_j}+\sigma_{j,r_j+1})/2$ and reuse it for the
final individual truncation.
Here both initial ranks are four, so all nonzero block directions
are retained. In an exact-moment realization, after normalization
by $\sqrt{n-1}$ their singular values are
$\sqrt2,1,\sqrt{.02},.1$. We set $\sigma_{j,5}=0$ when no fifth
singular value exists. The midpoint rule therefore gives
$t_r=t_e=.05$ on this normalized scale.

\paragraph{Joint directions and rank selection.}
Form the concatenated score matrix
\[
M_{\mathrm J}=[Q_r\ Q_e].
\]
This is separate from the whitened cross-covariance $M$ used by SCCR.
If $c_i$ are the singular values of $Q_r^\top Q_e$, the squared
singular values of $M_{\mathrm J}$ are $1+c_i$ and $1-c_i$.
Indeed, they are the eigenvalues of
\[
M_{\mathrm J}^\top M_{\mathrm J}
=\begin{pmatrix}I&Q_r^\top Q_e\\Q_e^\top Q_r&I\end{pmatrix}.
\]
The $c_i$ are cosines of principal angles between the retained
score spaces. Under exact population moments here they are
$(1,\rho,\gamma,\gamma)$, yielding
\begin{equation}
\sigma_i(M_{\mathrm J})^2
\in\{2,1.7,1.3,1.3,.7,.7,.3,0\}.
\label{eq:gaussian_ajive_spectrum}
\end{equation}

AJIVE compares these values with Wedin perturbation and random
score-space bounds. For resampling draw $b$, let
$\widetilde Q_{j,b}$ and $\widetilde V_{j,b}$ have four orthonormal
columns in the complements of the retained sample-score and feature
loading spaces, respectively. The perturbation estimate used here is
\[
\beta_{j,b}=\min\!\left\{1,
\frac{\max\{\|Z_{j,c}^\top\widetilde Q_{j,b}\|_2,
                 \|Z_{j,c}\widetilde V_{j,b}\|_2\}}
{\sigma_{j,4}}\right\},
\]
where $\|\cdot\|_2$ is the spectral norm. In the two-block rule,
each draw gives the Wedin squared-singular-value bound
$\max\{2-\beta_{r,b}^2-\beta_{e,b}^2,1\}$. The random bound
comes from the largest squared singular value of two independent
random orthonormal score matrices. The joint cutoff is the larger
of the chosen Wedin and random-bound quantiles (the fifth and
95th percentiles, respectively, in our implementation).

This row uses a zero-perturbation population interpretation: every
block direction, including the proxy noise within $Q,P_i$, is
retained in the initial rank-four signal space. There are no
nonzero discarded singular values. In an exact-moment matrix
realization, products of the data with orthogonal complements of
its retained loading or score space are therefore zero, giving
$\beta_{r,b}=\beta_{e,b}=0$ for every draw. The Wedin bound is
exactly two. The random
bound is at most two since
$\sigma_{\max}([Q_r^{\mathrm{rand}}\ Q_e^{\mathrm{rand}}])^2
=1+\|(Q_r^{\mathrm{rand}})^\top Q_e^{\mathrm{rand}}\|_2\le2$.
Thus the joint cutoff is two.

We explicitly include equality at this boundary: a squared singular
value of two is accepted. The local implementation uses
$\sigma_i^2+\varepsilon>2$ with $\varepsilon=10^{-10}$, which has
the same outcome for this separated spectrum. A strict $\sigma_i^2>2$
rule without this convention would reject even the exactly shared
$S$. Under the stated convention, only $S$ is selected; $A$ and
both $F_i$ remain outside the joint space.

The complement-sampling step also needs enough ambient dimensions.
A literal four-feature rank-four block has no feature complement.
For a numerical realization of this population rule, we either
explicitly declare verified zero perturbation or isometrically
embed each four-coordinate block into at least eight feature
dimensions, using at least nine observations to allow four
complement directions within centered sample space.
Such padding preserves the nonzero singular values, sample-score
spaces, and factor recovery. It is not extra observation noise.

\paragraph{Reconstruction and recovery.}
After joint selection, AJIVE checks that each joint score has
sufficient projected energy in both blocks. The selected score
$S$ has loading norm $\sqrt2$ in each block, above $.05$, so it
passes. Let $H_J$ project onto this joint score space. The joint
and individual reconstructions are
\[
J_r=H_JZ_{r,c},\qquad
I_r=\sum_{\sigma_k((I-H_J)Z_{r,c})>t_r}
\sigma_k u_kv_k^\top,
\]
where the sum uses the SVD of the remainder and thresholds on the
same scale as the original block singular values.
In the population calculation,
\[
z_r-j_r=(0,A,.1F_1,\sqrt{.02}F_2).
\]
Its nonzero normalized singular values $1,\sqrt{.02},.1$ all exceed
$.05$. The final individual component is therefore the whole
remainder, equivalently $\mathcal T_3((I-H_J)Z_{r,c})$, and
\[
(R_A^2,R_{F_1}^2,R_{F_2}^2)=(1,1,1).
\]
This conclusion depends on the initial ranks, perturbation model,
and boundary convention. It does not establish that AJIVE always
retains $A$: different initial ranks or nonzero discarded variation
can change its joint-rank decision.

\subsection{SCCR and matched top-2 CCA removal}
\label{app:gaussian_sccr}

\paragraph{Definition and transformation.}
For this analytic comparison we use the SCCR construction in
Section~\ref{sec:method} with $\lambda_r=\lambda_e=0$, without OAS.
Thus $R_{rr}=S_{rr}$ and $R_{ee}=S_{ee}$, and their Cholesky factors
are
\[
L_r=L_e=\operatorname{diag}(\sqrt2,1,.1,\sqrt{.02}).
\]
The whitened cross-covariance is
\[
M=L_r^{-1}S_{re}L_e^{-\top}
=\operatorname{diag}(1,\rho,\gamma,\gamma).
\]
In the SVD $M=U\operatorname{diag}(s)V^\top$, we may take
$U=V=I_4$ and $s=(1,\rho,\gamma,\gamma)$. A shared rotation within
the repeated $\gamma$ subspace does not change the retained space.
For $H=\{i:s_i\ge\tau\}$, SCCR forms
\[
K=I-U_HU_H^\top,\qquad
T=L_r^{-\top}KL_r^\top,\qquad Z_\perp=Z_{r,c}T.
\]
Here $K$ is an orthogonal projector in whitened raw coordinates.
The raw-coordinate map $T$ is generally an oblique projector; it
is diagonal in this particular construction.

\paragraph{Residual and recovery.}
For $\gamma<\tau<\rho$, the selected directions are $S,A$. Hence
\[
K=T=\operatorname{diag}(0,0,1,1),\qquad
z_\perp=(0,0,.1F_1,\sqrt{.02}F_2).
\]
The factor $A$ is independent of the retained output, while both
$F_i$ are recovered exactly by undoing their nonzero scales:
\[
(R_A^2,R_{F_1}^2,R_{F_2}^2)=(0,1,1).
\]
The result matches the stated audit objective, although the retained
coordinates still have nonzero covariance with $P_1,P_2$.
SCCR therefore does not impose exact covariance erasure against
the whole expert view.

For completeness, because the selection rule uses $s_i\ge\tau$,
the same result also holds at $\tau=\rho$. For
$0<\tau\le\gamma$, all four directions are removed; for
$\rho<\tau<1$, only $S$ is removed and all three evaluated factors
remain recoverable. In the reported interval, SCCR is exactly
removal of the top two raw canonical directions. Thresholding
rather than specifying this rank does not distinguish their
operators on these fixed embeddings.

\subsection{Summary and scope of the comparison}
\label{app:gaussian_summary}

Table~\ref{tab:synthetic_operators} collects the transformations and
population outputs. The recovery values follow from
Equation~\eqref{eq:gaussian_recovery_covariance}. For brevity in the
table, write $D_A=d_\rho(A,Q)$ and $D_i=d_\gamma(F_i,P_i)$ as
in Equation~\eqref{eq:gaussian_dcca_score}.

\begin{table*}[t]
\centering
\caption{Residualization transformations and population outputs.
JIVE uses global optima with joint rank $q$ and individual ranks
$(4-q,4-q)$. AJIVE uses initial ranks $(4,4)$ and the zero-perturbation
boundary convention in Appendix~\ref{app:gaussian_ajive}.
SCCR uses zero ridge and $\gamma<\tau<\rho$.}
\label{tab:synthetic_operators}
\setlength{\tabcolsep}{3pt}
\begin{tabularx}{\textwidth}{@{}p{.20\textwidth}p{.31\textwidth}X@{}}
\toprule
Method / output & Transformation & Population output \\
\midrule
Raw embedding & $Z_{r,c}$
& $(\sqrt2S,A,.1F_1,\sqrt{.02}F_2)$ \\
Feature-space projection & $Z_{r,c}(I-Z_{e,c}^\dagger Z_{e,c})$
& $(0,0,0,0)$ \\
OLS residual & $Z_{r,c}-Z_{e,c}B_{\mathrm{OLS}}$
& $\begin{aligned}(0,\,&A-\rho Q,\\
&.1(F_1-\gamma P_1),\\
&\sqrt{.02}(F_2-\gamma P_2))\end{aligned}$ \\
Exact LEACE & $Z_{r,c}T_{\mathrm{LEACE}}$ & $(0,0,0,0)$ \\
D-CCA distinctive & $Z_{r,c}-C_r$
& $(0,D_A,.1D_1,\sqrt{.02}D_2)$ \\
JIVE, $q=1$, individual & $\mathcal T_3((I-H_J)Z_{r,c})$
& $(0,A,.1F_1,\sqrt{.02}F_2)$ \\
JIVE, $q=2$, individual & $\mathcal T_2((I-H_J)Z_{r,c})$
& $(0,A,0,\sqrt{.02}F_2)$ \\
AJIVE $(4,4)$, individual & $\mathcal T_3((I-H_J)Z_{r,c})$
& $(0,A,.1F_1,\sqrt{.02}F_2)$ \\
SCCR & $Z_\perp=Z_{r,c}T$ & $(0,0,.1F_1,\sqrt{.02}F_2)$ \\
\bottomrule
\end{tabularx}
\end{table*}

\paragraph{What the example establishes.}
Exact LEACE and feature-space projection remove the weak factors
in this full-rank construction. OLS and D-CCA retain them but also
retain information about $A$. The stated JIVE individual components
follow reconstruction ranks and variance, while the stated AJIVE
output follows its initial signal spaces and perturbation cutoff.
SCCR achieves the chosen suppression--retention pattern by removing
only the two strongly shared canonical directions. These are
comparisons under specified formulations, not claims that all
other rank choices, preprocessing choices, or outputs must fail.

The outputs also differ in their use of paired observations. OLS
and the displayed D-CCA map require both views for each new output.
LEACE and SCCR act on the raw view alone after fitting. JIVE and
AJIVE first decompose paired data in sample-score space; using a
fitted decomposition for new observations requires an additional
inference rule. In the population cases reported here, their raw
individual outputs happen to be coordinate selections of $z_r$,
so a raw-only linear extension can reproduce them exactly. This
does not make finite-sample paired and raw-only extensions identical.

The analytic example demonstrates the operator mechanism. It does
not establish that neural encoders learn these factors, that a
learned expert embedding exposes $A$, or that linear suppression
implies nonlinear independence. Those questions require separate
learning and recovery measurements. Likewise, choosing this audit
objective does not establish that retaining $F_i$ or removing $A$
is clinically useful in any particular application.

\section{Theory and proofs}
\label{app:stability}
\subsection{Stability of the Frozen SCCR Operator}
Here we prove the stability of the frozen SCCR operator used after calibration. 
\paragraph{Perturbation results used in the proof.}
We use two standard Weyl perturbation inequalities.

First, for Hermitian matrices \(A,E\in\mathbb{R}^{d\times d}\), with eigenvalues
ordered nonincreasingly,
\[
\lambda_1(A)\ge\cdots\ge\lambda_d(A),
\]
Weyl's eigenvalue perturbation inequality gives
\[
|\lambda_i(A+E)-\lambda_i(A)|
\le
\|E\|_2,
\qquad
i=1,\ldots,d.
\]
Equivalently,
\[
\lambda_{\min}(A+E)
\ge
\lambda_{\min}(A)-\|E\|_2.
\]
We use this form to show that empirical positive definiteness of
\(\widehat R_j\) follows from population positive definiteness of \(R_j\) when
\(\|\widehat R_j-R_j\|_2\) is small.

Second, for arbitrary rectangular matrices \(A,E\), with singular values ordered
nonincreasingly,
\[
\sigma_1(A)\ge\cdots\ge\sigma_m(A),
\qquad
m=\min(d_1,d_2),
\]
Weyl's singular-value perturbation inequality gives
\[
|\sigma_i(A+E)-\sigma_i(A)|
\le
\|E\|_2,
\qquad
i=1,\ldots,m.
\]
We use this form with \(A=M\) and \(E=\widehat M-M\) to show that the empirical
singular values cannot cross the threshold \(\tau\) when
\(\|\widehat M-M\|_2\) is smaller than the threshold gap. References can be found in \cite{hornjohnson2013matrix, bhatia1997matrix, stewart1990matrix}

\paragraph{The OAS-to-ridge map.}
We first make explicit the scalar map that converts an auto-covariance matrix into the
ridge level used by SCCR. For a symmetric positive semidefinite matrix
\(S\in\mathbb{R}^{d\times d}\) and effective sample size \(B\), define
\[
m(S)=\frac{\operatorname{tr}(S)}{d},
\]
\[
a(S)=
\left(1-\frac{2}{d}\right)\operatorname{tr}(S^2)
+
\operatorname{tr}(S)^2,
\]
and
\[
b(S)=
\left(B+1-\frac{2}{d}\right)
\left(
\operatorname{tr}(S^2)
-
\frac{\operatorname{tr}(S)^2}{d}
\right).
\]
The OAS shrinkage coefficient is the safeguarded and clipped version of
\[
\frac{a(S)}{b(S)}.
\]
Concretely, the implementation avoids division by zero and clips the coefficient below
one. We write the resulting coefficient as
\[
\rho_{\mathrm{OAS}}(S;B,d)\in[0,\rho_{\max}],
\qquad
\rho_{\max}<1.
\]
The equivalent ridge level is then
\[
\lambda_{\mathrm{OAS}}(S;B,d)
=
\frac{\rho_{\mathrm{OAS}}(S;B,d)m(S)}
{1-\rho_{\mathrm{OAS}}(S;B,d)}.
\]
Thus the OAS ridge rule is simply the deterministic map
\[
S\mapsto \lambda_{\mathrm{OAS}}(S;B,d).
\]
In the stability result below, we assume this map is locally Lipschitz at the population
auto-covariances. That is, for \(j\in\{r,e\}\), there are constants
\(\ell_j<\infty\) and radii \(\delta_j>0\) such that whenever
\(\|\widehat S_{jj}-S_{jj}\|_2\le\delta_j\),
\[
\left|
\lambda_{\mathrm{OAS}}(\widehat S_{jj};B,d_j)
-
\lambda_{\mathrm{OAS}}(S_{jj};B,d_j)
\right|
\le
\ell_j\|\widehat S_{jj}-S_{jj}\|_2.
\]
This assumption concerns only the scalar OAS ridge rule. It is separate from, and weaker
than, assuming empirical conditioning directly.%

\paragraph{Symmetric-square-root form.}
The methods uses Cholesky decomposition \(R_j=L_jL_j^\top\). For the proof, it is cleaner to
use symmetric square roots
\[
Q_j=R_j^{1/2}.
\]
This does not change the unwhitened SCCR operator. Indeed, if
\(R_j=L_jL_j^\top=Q_j^2\), then \(L_j=Q_jO_j\) for an orthogonal matrix \(O_j\). The
Cholesky-whitened cross-covariance is therefore an orthogonal rotation of the
symmetric-square-root-whitened cross-covariance. The selected left singular projector is
rotated by the same orthogonal matrix, and the unwhitened operator is unchanged:
\[
L_r^{-\top}(I-P_L)L_r^\top
=
Q_r^{-1}(I-P)Q_r.
\]
We therefore prove stability for the equivalent operator
\[
T=Q_r^{-1}(I-P)Q_r.
\]

\begin{theorem}[Detailed stability bound for OAS-calibrated SCCR]
\label{thm:frozen_appendix}
Let \(S_{rr},S_{ee},S_{re}\) be the population covariance blocks of the centered
post-training raw and expert embeddings. Define the population OAS-selected ridge levels
\[
\lambda_r=\lambda_{\mathrm{OAS}}(S_{rr};B,d_r),
\qquad
\lambda_e=\lambda_{\mathrm{OAS}}(S_{ee};B,d_e),
\]
and the population regularized auto-covariances
\[
R_r=S_{rr}+\lambda_r I,
\qquad
R_e=S_{ee}+\lambda_e I.
\]
In addition to the local OAS Lipschitz assumption and the threshold-gap assumption below,
assume only population, not empirical, conditioning:
\[
R_r\succ0,
\qquad
R_e\succ0.
\]
Let
\[
\gamma_r=\lambda_{\min}(R_r)>0,
\qquad
\gamma_e=\lambda_{\min}(R_e)>0.
\]
Define
\[
Q_r=R_r^{1/2},
\qquad
Q_e=R_e^{1/2},
\qquad
M=Q_r^{-1}S_{re}Q_e^{-1}.
\]
Let
\[
M=U\operatorname{diag}(s_1,\ldots,s_m)V^\top,
\qquad
s_1\ge\cdots\ge s_m\ge0,
\qquad
m=\min(d_r,d_e).
\]
Assume that \(q\) population singular values are above threshold $\tau$ and that there
is a strict threshold gap:
\[
s_q>\tau>s_{q+1},
\qquad
1\le q<m.
\]
Let
\[
P=U_{1:q}U_{1:q}^\top,
\qquad
T=Q_r^{-1}(I-P)Q_r.
\]

Let \(\widehat S_{rr},\widehat S_{ee},\widehat S_{re}\) be the corresponding calibration
covariance blocks. Define
\[
\widehat\lambda_r
=
\lambda_{\mathrm{OAS}}(\widehat S_{rr};B,d_r),
\qquad
\widehat\lambda_e
=
\lambda_{\mathrm{OAS}}(\widehat S_{ee};B,d_e),
\]
\[
\widehat R_r
=
\widehat S_{rr}+\widehat\lambda_r I,
\qquad
\widehat R_e
=
\widehat S_{ee}+\widehat\lambda_e I,
\]
\[
\widehat Q_r=\widehat R_r^{1/2},
\qquad
\widehat Q_e=\widehat R_e^{1/2},
\qquad
\widehat M=\widehat Q_r^{-1}\widehat S_{re}\widehat Q_e^{-1}.
\]
Let \(\widehat P\) be the empirical left singular projector selected by the same
threshold rule \(s_i\ge\tau\), and define
\[
\widehat T=\widehat Q_r^{-1}(I-\widehat P)\widehat Q_r.
\]

Then, for sufficiently small calibration error, the empirical thresholding step selects
the same rank \(q\), and there exists a constant \(C>0\) such that
\[
\|\widehat T-T\|_2
\le
C\left(
\|\widehat S_{rr}-S_{rr}\|_2
+
\|\widehat S_{ee}-S_{ee}\|_2
+
\|\widehat S_{re}-S_{re}\|_2
\right).
\]
Consequently, \(\widehat T\to T\) whenever the calibration covariance blocks converge.
\end{theorem}

\paragraph{Meaning of ``sufficiently small calibration error.''}
The phrase means the following. Define
\[
\varepsilon_{\mathrm{cal}}
=
\|\widehat S_{rr}-S_{rr}\|_2
+
\|\widehat S_{ee}-S_{ee}\|_2
+
\|\widehat S_{re}-S_{re}\|_2.
\]
There exists a number \(\varepsilon_0>0\), depending on the population conditioning
\(\gamma_r,\gamma_e\), the local Lipschitz constants of the OAS ridge maps, and the
threshold gap \(s_q>\tau>s_{q+1}\), such that all conclusions of the theorem hold
whenever
\[
\varepsilon_{\mathrm{cal}}\le\varepsilon_0.
\]
In the proof, the smallness requirement is used for three concrete purposes:
\begin{enumerate}[leftmargin=*]
    \item to ensure the OAS ridge maps are evaluated inside their local Lipschitz
    neighborhoods;
    \item to ensure empirical conditioning follows from population conditioning, e.g.
    \[
    \widehat R_r\succeq \frac{\gamma_r}{2}I,
    \qquad
    \widehat R_e\succeq \frac{\gamma_e}{2}I;
    \]
    \item to ensure the empirical singular values do not cross the threshold \(\tau\).
\end{enumerate}
No empirical conditioning assumption is imposed separately; it is derived locally
from \(R_r\succ0\) and \(R_e\succ0\).

\begin{proof}
All norms below are spectral norms. Define
\[
\varepsilon_{rr}
=
\|\widehat S_{rr}-S_{rr}\|_2,
\qquad
\varepsilon_{ee}
=
\|\widehat S_{ee}-S_{ee}\|_2,
\qquad
\varepsilon_{re}
=
\|\widehat S_{re}-S_{re}\|_2.
\]
Thus
\[
\varepsilon_{\mathrm{cal}}
=
\varepsilon_{rr}+\varepsilon_{ee}+\varepsilon_{re}.
\]

\paragraph{Step 1: perturbation of the OAS-selected ridge levels.}
By local Lipschitzness of the OAS ridge maps, for sufficiently small
\(\varepsilon_{rr}\) and \(\varepsilon_{ee}\),
\[
|\widehat\lambda_r-\lambda_r|
\le
\ell_r\varepsilon_{rr},
\qquad
|\widehat\lambda_e-\lambda_e|
\le
\ell_e\varepsilon_{ee}.
\]
Now
\[
\widehat R_r-R_r
=
(\widehat S_{rr}-S_{rr})
+
(\widehat\lambda_r-\lambda_r)I.
\]
Therefore by the triangle inequality:
\[
\|\widehat R_r-R_r\|_2
\le
\|\widehat S_{rr}-S_{rr}\|_2
+
|\widehat\lambda_r-\lambda_r|
\le
(1+\ell_r)\varepsilon_{rr}.
\]
Similarly for $e$,
\[
\|\widehat R_e-R_e\|_2
\le
(1+\ell_e)\varepsilon_{ee}.
\]
For shorthand, define
\[
\eta_r=(1+\ell_r)\varepsilon_{rr},
\qquad
\eta_e=(1+\ell_e)\varepsilon_{ee}.
\]
Then
\[
\|\widehat R_r-R_r\|_2\le \eta_r,
\qquad
\|\widehat R_e-R_e\|_2\le \eta_e.
\]

\paragraph{Step 2: empirical conditioning follows from population conditioning.}
By assumption,
\[
\lambda_{\min}(R_r)=\gamma_r>0,
\qquad
\lambda_{\min}(R_e)=\gamma_e>0.
\]
Suppose the calibration error is small enough that
\[
\eta_r\le \frac{\gamma_r}{2},
\qquad
\eta_e\le \frac{\gamma_e}{2},
\]
where
\[
\eta_r=(1+\ell_r)\varepsilon_{rr},
\qquad
\eta_e=(1+\ell_e)\varepsilon_{ee}.
\]
Equivalently, it is enough to require
\[
\varepsilon_{rr}\le \frac{\gamma_r}{2(1+\ell_r)},
\qquad
\varepsilon_{ee}\le \frac{\gamma_e}{2(1+\ell_e)}.
\]

We now apply Weyl's eigenvalue perturbation inequality. Since
\[
\widehat R_r=R_r+(\widehat R_r-R_r),
\]
we have
\[
\lambda_{\min}(\widehat R_r)
\ge
\lambda_{\min}(R_r)-\|\widehat R_r-R_r\|_2
\ge
\gamma_r-\eta_r
\ge
\frac{\gamma_r}{2}.
\]
Thus
\[
\widehat R_r\succeq \frac{\gamma_r}{2}I.
\]
Similarly,
\[
\widehat R_e\succeq \frac{\gamma_e}{2}I.
\]

Define
\[
\alpha_r=\sqrt{\frac{\gamma_r}{2}},
\qquad
\alpha_e=\sqrt{\frac{\gamma_e}{2}}.
\]
We next translate these lower bounds from \(R_j\) and \(\widehat R_j\) to their
symmetric square roots.

Recall that \(A\succeq B\) means \(A-B\) is positive semidefinite. Since
\(R_r\succ0\), the spectral theorem gives an eigendecomposition
\[
R_r=V_r\operatorname{diag}(\lambda_1,\ldots,\lambda_{d_r})V_r^\top,
\]
where each eigenvalue satisfies
\[
\lambda_k\ge \gamma_r.
\]
The symmetric square root is therefore
\[
Q_r=R_r^{1/2}
=
V_r\operatorname{diag}(\sqrt{\lambda_1},\ldots,\sqrt{\lambda_{d_r}})V_r^\top.
\]
Because \(\lambda_k\ge\gamma_r\), we have
\[
\sqrt{\lambda_k}\ge \sqrt{\gamma_r}
\qquad
\text{for every }k.
\]
Hence
\[
Q_r-\sqrt{\gamma_r}I
=
V_r\operatorname{diag}
\bigl(
\sqrt{\lambda_1}-\sqrt{\gamma_r},
\ldots,
\sqrt{\lambda_{d_r}}-\sqrt{\gamma_r}
\bigr)
V_r^\top
\succeq 0.
\]
Therefore
\[
Q_r\succeq \sqrt{\gamma_r}I.
\]
Since
\[
\sqrt{\gamma_r}\ge \sqrt{\frac{\gamma_r}{2}}=\alpha_r,
\]
we also have
\[
Q_r\succeq \alpha_r I.
\]

The same argument applies to the empirical matrix. From
\[
\widehat R_r\succeq \frac{\gamma_r}{2}I,
\]
all eigenvalues of \(\widehat R_r\) are at least \(\gamma_r/2\). Therefore all
eigenvalues of
\[
\widehat Q_r=\widehat R_r^{1/2}
\]
are at least \(\sqrt{\gamma_r/2}=\alpha_r\), and hence
\[
\widehat Q_r\succeq \alpha_r I.
\]
Repeating the same reasoning for the expert branch gives
\[
Q_e\succeq \sqrt{\gamma_e}I\succeq \alpha_e I,
\qquad
\widehat Q_e\succeq \alpha_e I.
\]

Thus the lower bounds
\[
Q_r\succeq \alpha_r I,
\qquad
\widehat Q_r\succeq \alpha_r I,
\qquad
Q_e\succeq \alpha_e I,
\qquad
\widehat Q_e\succeq \alpha_e I
\]
are not separate assumptions. They follow from population conditioning,
local OAS ridge stability, and sufficiently small calibration error.

\paragraph{Step 3: perturbation of the square roots.}
We next bound
\[
\|\widehat Q_j-Q_j\|_2,
\qquad j\in\{r,e\}.
\]
Let
\[
X_j=\widehat Q_j-Q_j,
\qquad
E_j=\widehat R_j-R_j.
\]
Since
\[
\widehat Q_j^2=\widehat R_j,
\qquad
Q_j^2=R_j,
\]
we have
\[
\widehat Q_j^2-Q_j^2=E_j.
\]
Moreover,
\[
\widehat Q_jX_j+X_jQ_j
=
\widehat Q_j(\widehat Q_j-Q_j)
+
(\widehat Q_j-Q_j)Q_j
=
\widehat Q_j^2-Q_j^2
=
E_j.
\]
Thus \(X_j\) solves the Sylvester equation
\[
\widehat Q_jX_j+X_jQ_j=E_j.
\]
Since \(\widehat Q_j\succeq \alpha_j I\) and \(Q_j\succeq \alpha_j I\), this Sylvester
equation has the integral representation \citep[Theorem VII.2.3]{bhatia1997matrix}
\[
X_j
=
\int_0^\infty
e^{-t\widehat Q_j}
E_j
e^{-tQ_j}
\,dt.
\]
Taking norms gives
\[
\|X_j\|_2
\le
\int_0^\infty
\|e^{-t\widehat Q_j}\|_2
\|E_j\|_2
\|e^{-tQ_j}\|_2
\,dt.
\]
Because the smallest eigenvalues of \(\widehat Q_j\) and \(Q_j\) are both at least
\(\alpha_j\),
\[
\|e^{-t\widehat Q_j}\|_2\le e^{-t\alpha_j},
\qquad
\|e^{-tQ_j}\|_2\le e^{-t\alpha_j}.
\]
Therefore
\[
\|\widehat Q_j-Q_j\|_2
=
\|X_j\|_2
\le
\int_0^\infty e^{-2t\alpha_j}\,dt\,\|E_j\|_2
=
\frac{\|E_j\|_2}{2\alpha_j}.
\]
Applying the bounds from Step 1,
\[
\|\widehat Q_r-Q_r\|_2
\le
\frac{\eta_r}{2\alpha_r},
\qquad
\|\widehat Q_e-Q_e\|_2
\le
\frac{\eta_e}{2\alpha_e}.
\]

\paragraph{Step 4: perturbation of the inverse square roots.}
We use the identity
\[
\widehat Q_j^{-1}-Q_j^{-1}
=
\widehat Q_j^{-1}(Q_j-\widehat Q_j)Q_j^{-1}.
\]
Taking norms,
\[
\|\widehat Q_j^{-1}-Q_j^{-1}\|_2
\le
\|\widehat Q_j^{-1}\|_2
\|\widehat Q_j-Q_j\|_2
\|Q_j^{-1}\|_2.
\]
From Step 2,
\[
\|\widehat Q_j^{-1}\|_2\le \frac{1}{\alpha_j},
\qquad
\|Q_j^{-1}\|_2\le \frac{1}{\alpha_j}.
\]
Combining these inequalities with the square-root bound from Step 3 gives
\[
\|\widehat Q_j^{-1}-Q_j^{-1}\|_2
\le
\frac{1}{\alpha_j^2}
\cdot
\frac{\|E_j\|_2}{2\alpha_j}
=
\frac{\|E_j\|_2}{2\alpha_j^3}.
\]
Hence
\[
\|\widehat Q_r^{-1}-Q_r^{-1}\|_2
\le
\frac{\eta_r}{2\alpha_r^3},
\qquad
\|\widehat Q_e^{-1}-Q_e^{-1}\|_2
\le
\frac{\eta_e}{2\alpha_e^3}.
\]

\paragraph{Step 5: perturbation of the whitened cross-covariance.}
Recall
\[
M=Q_r^{-1}S_{re}Q_e^{-1},
\qquad
\widehat M=\widehat Q_r^{-1}\widehat S_{re}\widehat Q_e^{-1}.
\]
Add and subtract intermediate terms:
\[
\begin{aligned}
\widehat M-M
&=
(\widehat Q_r^{-1}-Q_r^{-1})
\widehat S_{re}
\widehat Q_e^{-1}
\\
&\quad+
Q_r^{-1}
(\widehat S_{re}-S_{re})
\widehat Q_e^{-1}
\\
&\quad+
Q_r^{-1}
S_{re}
(\widehat Q_e^{-1}-Q_e^{-1}).
\end{aligned}
\]
Taking norms and using submultiplicativity,
\[
\begin{aligned}
\|\widehat M-M\|_2
&\le
\|\widehat Q_r^{-1}-Q_r^{-1}\|_2
\|\widehat S_{re}\|_2
\|\widehat Q_e^{-1}\|_2
\\
&\quad+
\|Q_r^{-1}\|_2
\|\widehat S_{re}-S_{re}\|_2
\|\widehat Q_e^{-1}\|_2
\\
&\quad+
\|Q_r^{-1}\|_2
\|S_{re}\|_2
\|\widehat Q_e^{-1}-Q_e^{-1}\|_2.
\end{aligned}
\]
Also,
\[
\|\widehat S_{re}\|_2
\le
\|S_{re}\|_2+\varepsilon_{re}.
\]
For sufficiently small calibration error, we may assume \(\varepsilon_{re}\le1\), giving
\[
\|\widehat S_{re}\|_2
\le
\|S_{re}\|_2+1.
\]
Using the inverse bounds from Steps 2 and 4, we obtain
\[
\begin{aligned}
\|\widehat M-M\|_2
&\le
\frac{\eta_r}{2\alpha_r^3}
\cdot
(\|S_{re}\|_2+1)
\cdot
\frac{1}{\alpha_e}
\\
&\quad+
\frac{1}{\alpha_r}
\cdot
\varepsilon_{re}
\cdot
\frac{1}{\alpha_e}
\\
&\quad+
\frac{1}{\alpha_r}
\cdot
\|S_{re}\|_2
\cdot
\frac{\eta_e}{2\alpha_e^3}.
\end{aligned}
\]
Equivalently,
\[
\|\widehat M-M\|_2
\le
C_{M,r}\varepsilon_{rr}
+
C_{M,re}\varepsilon_{re}
+
C_{M,e}\varepsilon_{ee},
\]
where
\[
C_{M,r}
=
\frac{(1+\ell_r)(\|S_{re}\|_2+1)}
{2\alpha_r^3\alpha_e},
\]
\[
C_{M,re}
=
\frac{1}{\alpha_r\alpha_e},
\]
and
\[
C_{M,e}
=
\frac{(1+\ell_e)\|S_{re}\|_2}
{2\alpha_r\alpha_e^3}.
\]
Thus, with
\[
C_M=C_{M,r}+C_{M,re}+C_{M,e},
\]
we have
\[
\|\widehat M-M\|_2
\le
C_M\varepsilon_{\mathrm{cal}}.
\]

\paragraph{Step 6: the threshold selects the same rank.}
Let
\[
\Delta_\tau
=
\min\{s_q-\tau,\tau-s_{q+1}\}>0.
\]
By Weyl's inequality for singular values,
\[
|\widehat s_i-s_i|
\le
\|\widehat M-M\|_2
\qquad
\text{for every } i.
\]
Therefore, if
\[
\|\widehat M-M\|_2<\Delta_\tau,
\]
then
\[
\widehat s_q
\ge
s_q-\|\widehat M-M\|_2
>
\tau,
\]
and
\[
\widehat s_{q+1}
\le
s_{q+1}+\|\widehat M-M\|_2
<
\tau.
\]
Thus the empirical SCCR thresholding step selects exactly \(q\) directions.
Using Step 5, it is enough to require, for example,
\[
C_M\varepsilon_{\mathrm{cal}}\le \frac{\Delta_\tau}{2}.
\]

\paragraph{Step 7: perturbation of the selected singular subspace.}
Let
\[
\delta=s_q-s_{q+1}>0.
\]

By further decreasing the allowed calibration error if necessary, assume

\[
\|\widehat M-M\|_2\le \frac{\delta}{4}.
\]

Equivalently, using Step 5, it is enough to require

\[
C_M\varepsilon_{\mathrm{cal}}\le \frac{\delta}{4}.
\]

Then Weyl's inequality gives
\[
\widehat s_q-\widehat s_{q+1}
\ge
(s_q-\|\widehat M-M\|_2)
-
(s_{q+1}+\|\widehat M-M\|_2)
\ge
\frac{\delta}{2}.
\]

Thus the empirical selected and unselected singular clusters remain separated.

By Wedin's \(\sin\Theta\) theorem for singular subspaces \cite[Theorem V.4.4]{stewart1990matrix}
(see also \cite{wedin1972perturbation}), and using
\(\|\widehat P-P\|_2=\|\sin\Theta(\widehat{\mathcal U},\mathcal U)\|_2\)
or orthogonal projectors onto subspaces of the same dimension, there is a universal constant \(C_W>0\) such that
\(C_W>0\) such that
\[
\|\widehat P-P\|_2
\le
C_W\frac{\|\widehat M-M\|_2}{\delta}.
\]
Combining this with the bound from Step 5 gives
\[
\|\widehat P-P\|_2
\le
\frac{C_WC_M}{\delta}
\varepsilon_{\mathrm{cal}}.
\]
Define
\[
C_P=\frac{C_WC_M}{\delta}.
\]
Then
\[
\|\widehat P-P\|_2
\le
C_P\varepsilon_{\mathrm{cal}}.
\]

\paragraph{Step 8: perturbation of the SCCR operator.}
Recall
\[
T=Q_r^{-1}(I-P)Q_r,
\qquad
\widehat T=\widehat Q_r^{-1}(I-\widehat P)\widehat Q_r.
\]
Add and subtract
\[
\widehat Q_r^{-1}(I-P)\widehat Q_r
\qquad
\text{and}
\qquad
\widehat Q_r^{-1}(I-P)Q_r.
\]
Then
\[
\begin{aligned}
\widehat T-T
&=
\widehat Q_r^{-1}(P-\widehat P)\widehat Q_r
\\
&\quad+
\widehat Q_r^{-1}(I-P)(\widehat Q_r-Q_r)
\\
&\quad+
(\widehat Q_r^{-1}-Q_r^{-1})(I-P)Q_r.
\end{aligned}
\]
Taking norms and using \(\|I-P\|_2=1\),
\[
\begin{aligned}
\|\widehat T-T\|_2
&\le
\|\widehat Q_r^{-1}\|_2
\|\widehat P-P\|_2
\|\widehat Q_r\|_2
\\
&\quad+
\|\widehat Q_r^{-1}\|_2
\|\widehat Q_r-Q_r\|_2
\\
&\quad+
\|\widehat Q_r^{-1}-Q_r^{-1}\|_2
\|Q_r\|_2.
\end{aligned}
\]
We already have
\[
\|\widehat Q_r^{-1}\|_2\le \frac{1}{\alpha_r},
\]
\[
\|\widehat Q_r-Q_r\|_2
\le
\frac{\eta_r}{2\alpha_r},
\]
and
\[
\|\widehat Q_r^{-1}-Q_r^{-1}\|_2
\le
\frac{\eta_r}{2\alpha_r^3}.
\]
It remains only to bound \(\|\widehat Q_r\|_2\). By the triangle inequality,
\[
\|\widehat Q_r\|_2
\le
\|Q_r\|_2+\|\widehat Q_r-Q_r\|_2.
\]
Using Step 3 and the small-error condition \(\eta_r\le\gamma_r/2\),
\[
\|\widehat Q_r-Q_r\|_2
\le
\frac{\eta_r}{2\alpha_r}
\le
\frac{\gamma_r}{4\alpha_r}.
\]
Define the finite population-dependent constant
\[
\beta_r
=
\|Q_r\|_2+\frac{\gamma_r}{4\alpha_r}.
\]
Then
\[
\|\widehat Q_r\|_2\le\beta_r.
\]
Substituting all bounds into the operator expansion yields
\[
\|\widehat T-T\|_2
\le
\frac{\beta_r}{\alpha_r}
\|\widehat P-P\|_2
+
\frac{1}{\alpha_r}
\frac{\eta_r}{2\alpha_r}
+
\frac{\eta_r}{2\alpha_r^3}
\|Q_r\|_2.
\]
Using
\[
\|\widehat P-P\|_2
\le
C_P\varepsilon_{\mathrm{cal}}
\]
and
\[
\eta_r=(1+\ell_r)\varepsilon_{rr}
\le
(1+\ell_r)\varepsilon_{\mathrm{cal}},
\]
we obtain
\[
\|\widehat T-T\|_2
\le
\left[
\frac{\beta_r}{\alpha_r}C_P
+
(1+\ell_r)
\left(
\frac{1}{2\alpha_r^2}
+
\frac{\|Q_r\|_2}{2\alpha_r^3}
\right)
\right]
\varepsilon_{\mathrm{cal}}.
\]
Therefore, defining
\[
C
=
\frac{\beta_r}{\alpha_r}C_P
+
(1+\ell_r)
\left(
\frac{1}{2\alpha_r^2}
+
\frac{\|Q_r\|_2}{2\alpha_r^3}
\right),
\]
we have
\[
\|\widehat T-T\|_2
\le
C\varepsilon_{\mathrm{cal}}.
\]
That is,
\[
\|\widehat T-T\|_2
\le
C\left(
\|\widehat S_{rr}-S_{rr}\|_2
+
\|\widehat S_{ee}-S_{ee}\|_2
+
\|\widehat S_{re}-S_{re}\|_2
\right).
\]
This proves the theorem. Finally, if the calibration covariance blocks converge to the population covariance
blocks, then \(\varepsilon_{\mathrm{cal}}\to0\), and the displayed bound implies
\[
\|\widehat T-T\|_2\to0.
\]
\end{proof}

\begin{remark}[A concrete sufficient small-error condition]
The proof constructs an explicit type of small-error condition. Let \(\delta_r,\delta_e\)
be radii on which the OAS ridge maps are locally Lipschitz. One sufficient condition is
\[
\varepsilon_{\mathrm{cal}}
\le
\varepsilon_0,
\]
where \(\varepsilon_0>0\) is any number satisfying
\[
\varepsilon_0
\le
\min\left\{
\delta_r,\delta_e,
1,
\frac{\gamma_r}{2(1+\ell_r)},
\frac{\gamma_e}{2(1+\ell_e)},
\frac{\Delta_\tau}{2C_M},
\frac{\delta}{4C_M}
\right\}.
\]
The first two terms keep the OAS map inside its local Lipschitz neighborhood. The next
term \(\varepsilon_0\le1\) is only used to simplify the bound
\(\|\widehat S_{re}\|_2\le\|S_{re}\|_2+1\). The two terms involving
\(\gamma_r,\gamma_e\) ensure empirical positive definiteness of
\(\widehat R_r,\widehat R_e\). The final term prevents singular values from crossing the
threshold \(\tau\). This is the precise meaning of ``for sufficiently small calibration
error'' in Theorem~\ref{thm:frozen_appendix}. 
\label{remark:_aap}
\end{remark}

\begin{remark}[Boundary cases]
The theorem is stated for \(1\le q<m\) to avoid extra notation. If \(q=0\), then no
canonical direction is removed and \(P=0\). Once the calibration error is small enough
that \(\widehat s_1<\tau\), we also have \(\widehat P=0\), so the Wedin subspace term
vanishes. If all available singular directions are removed, the same argument applies
with the relevant gap between the selected singular subspace and its complement.
\end{remark}

\begin{corollary}[Detailed stability bound for the frozen inference map]
\label{cor:frozen_appendix}
Let the assumptions of Theorem~\ref{thm:frozen_appendix} hold, and suppose the
calibration error is sufficiently small in the sense of the small-error condition in
Remark~\ref{remark:_aap}. Define
\[
g(z)=(z-\mu_r)T,
\qquad
\widehat g(z)=(z-\widehat\mu_r)\widehat T,
\]
where \(z\in\mathbb{R}^{d_r}\), \(\mu_r\) is the population raw-branch mean, and
\(\widehat\mu_r\) is the calibration raw-branch mean. Then, for every fixed
\(z\in\mathbb{R}^{d_r}\),
\[
\|\widehat g(z)-g(z)\|_2
\le
\|z-\mu_r\|_2\,\|\widehat T-T\|_2
+
\|\widehat\mu_r-\mu_r\|_2\,\|\widehat T\|_2.
\]
Moreover, with
\[
\alpha_r=\sqrt{\frac{\gamma_r}{2}}
\]
and
\[
\beta_r
=
\|Q_r\|_2+\frac{\gamma_r}{4\alpha_r},
\]
as in the proof of Theorem~\ref{thm:frozen_appendix}, we have
\[
\|\widehat T\|_2
\le
\frac{\beta_r}{\alpha_r}.
\]
Consequently, using the operator bound from Theorem~\ref{thm:frozen_appendix},
\[
\|\widehat g(z)-g(z)\|_2
\le
C\|z-\mu_r\|_2
\left(
\|\widehat S_{rr}-S_{rr}\|_2
+
\|\widehat S_{ee}-S_{ee}\|_2
+
\|\widehat S_{re}-S_{re}\|_2
\right)
+
\frac{\beta_r}{\alpha_r}
\|\widehat\mu_r-\mu_r\|_2.
\]
Hence, if
\[
\widehat\mu_r\to\mu_r
\]
and the calibration covariance blocks converge to their population counterparts, then
\[
\widehat g(z)\to g(z)
\qquad
\text{for every fixed } z\in\mathbb{R}^{d_r}.
\]
\end{corollary}

\begin{proof}
We have
\[
\widehat g(z)-g(z)
=
(z-\mu_r)(\widehat T-T)
+
(\mu_r-\widehat\mu_r)\widehat T.
\]
Taking norms yields
\[
\|\widehat g(z)-g(z)\|_2
\le
\|z-\mu_r\|_2\,\|\widehat T-T\|_2
+
\|\widehat\mu_r-\mu_r\|_2\,\|\widehat T\|_2.
\]
For the operator norm of \(\widehat T\),
\[
\|\widehat T\|_2
=
\|\widehat Q_r^{-1}(I-\widehat P)\widehat Q_r\|_2
\le
\|\widehat Q_r^{-1}\|_2\,\|I-\widehat P\|_2\,\|\widehat Q_r\|_2
\le
\frac{\beta_r}{\alpha_r}.
\]
Here \(I-\widehat P\) is an orthogonal projector, so
\(\|I-\widehat P\|_2\le1\), and the proof of
Theorem~\ref{thm:frozen_appendix} gives
\(\|\widehat Q_r^{-1}\|_2\le1/\alpha_r\) and
\(\|\widehat Q_r\|_2\le\beta_r\).
The final convergence claim follows from Theorem~\ref{thm:frozen_appendix}.
\end{proof}

\begin{proposition}[SCCR removes exactly the thresholded canonical correlations]
\label{prop:exact_main}
Let \(Z_{r,c}\in\mathbb{R}^{n\times d_r}\) and
\(Z_{e,c}\in\mathbb{R}^{n\times d_e}\) be centered reference embeddings, and let
\(B=\max(n-1,1)\). Define
\[
S_{re}=\frac{1}{B}Z_{r,c}^{\top}Z_{e,c}.
\]
Let
\[
R_r=L_rL_r^\top,
\qquad
R_e=L_eL_e^\top
\]
be the OAS-regularized raw and expert auto-covariances used by SCCR, and define the
regularized whitened cross-covariance
\[
M=L_r^{-1}S_{re}L_e^{-\top}.
\]
Let
\[
M=U\operatorname{diag}(s)V^\top
=
\sum_i s_i u_i v_i^\top
\]
be a thin SVD. For a threshold \(\tau\), define
\[
H=\{i:s_i\ge \tau\},
\qquad
K=I-U_HU_H^\top,
\qquad
T=L_r^{-\top}KL_r^\top.
\]
Let
\[
Z_\perp=Z_{r,c}T,
\qquad
S_{\perp e}=\frac{1}{B}Z_\perp^\top Z_{e,c}.
\]
Then
\[
L_r^{-1}S_{\perp e}L_e^{-\top}
=
(I-U_HU_H^\top)M
=
\sum_{i\notin H}s_i u_i v_i^\top.
\]
Consequently, in the fixed SCCR regularized CCA metric, the retained singular values
of the residual cross-covariance are exactly the sub-threshold values
\(\{s_i:i\notin H\}\). In particular, with the convention \(\max\emptyset=0\),
\[
\left\|L_r^{-1}S_{\perp e}L_e^{-\top}\right\|_2
=
\max_{i\notin H}s_i
<
\tau.
\]
\end{proposition}
\begin{proof}
By definition,
\[
Z_\perp=Z_{r,c}T.
\]
Therefore,
\[
S_{\perp e}
=
\frac{1}{B}Z_\perp^\top Z_{e,c}
=
\frac{1}{B}(Z_{r,c}T)^\top Z_{e,c}
=
T^\top
\left(
\frac{1}{B}Z_{r,c}^\top Z_{e,c}
\right)
=
T^\top S_{re}.
\]
Since
\[
T=L_r^{-\top}KL_r^\top
\]
and \(K=K^\top\), we have
\[
T^\top
=
(L_r^{-\top}KL_r^\top)^\top
=
L_rK L_r^{-1}.
\]
Thus
\[
\begin{aligned}
L_r^{-1}S_{\perp e}L_e^{-\top}
&=
L_r^{-1}T^\top S_{re}L_e^{-\top} \\
&=
L_r^{-1}
\left(
L_rK L_r^{-1}
\right)
S_{re}L_e^{-\top} \\
&=
K L_r^{-1}S_{re}L_e^{-\top} \\
&=
KM.
\end{aligned}
\]
Using the SVD
\[
M=\sum_i s_i u_i v_i^\top,
\]
we obtain
\[
KM
=
\sum_i s_i (Ku_i)v_i^\top.
\]
By construction,
\[
K=I-U_HU_H^\top
\]
is the orthogonal projector onto the complement of the span of the selected left
singular vectors \(\{u_i:i\in H\}\). Hence
\[
Ku_i=0
\quad\text{for } i\in H,
\qquad
Ku_i=u_i
\quad\text{for } i\notin H.
\]
Therefore
\[
KM
=
\sum_{i\notin H}s_i u_i v_i^\top.
\]
This proves the displayed decomposition.

Because the vectors \(\{u_i\}\) and \(\{v_i\}\) are orthonormal singular-vectors, the nonzero singular values of
\[
\sum_{i\notin H}s_i u_i v_i^\top
\]
are the retained nonzero values among \(\{s_i:i\notin H\}\). Since
\[
H=\{i:s_i\ge\tau\},
\]
every index \(i\notin H\) satisfies \(s_i<\tau\). Hence
\[
\left\|L_r^{-1}S_{\perp e}L_e^{-\top}\right\|_2
=
\max_{i\notin H}s_i
<
\tau,
\]
with the convention that the maximum over an empty set is \(0\). This completes the
proof.
\end{proof}
\newpage

\section{Algorithm}
\label{app:algorithm}
The algorithm below uses the OAS-selected branchwise ridge from
Section~\ref{sec:method}. We compute separate ridge levels from the
current raw and expert embeddings, and recompute them from global
training moments during calibration. We keep $S_{re}$ unchanged.
The clipping bound $\rho_{\max}<1$ and the numerical safeguards are
the same during training and calibration.

\begin{algorithm}[H]
\caption{Shared Canonical Component Removal with OAS-Selected Branchwise Ridge}
\label{alg:sccr}
\DontPrintSemicolon

\KwIn{Raw embedding batch \(Z_r\in\R^{n\times d_r}\), expert embedding batch \(Z_e\in\R^{n\times d_e}\)}
\KwIn{Threshold \(\tau\in(0,1)\), shrinkage bound \(\rho_{\max}\in(0,1)\)}
\KwOut{Residual embedding \(Z_{\perp}\), optional removed component \(Z_{\parallel}\), and operator \(T\)}

\BlankLine
\textbf{Training-time operator on a minibatch}\;

Compute batch means \(\mu_r \leftarrow \tfrac{1}{n}\mathbf{1}^\top Z_r\), \(\mu_e \leftarrow \tfrac{1}{n}\mathbf{1}^\top Z_e\)\;
Center:
\(Z_{r,c} \leftarrow Z_r-\mathbf{1}\mu_r\),
\(Z_{e,c} \leftarrow Z_e-\mathbf{1}\mu_e\)\;
\(B \leftarrow \max(n-1,1)\)\;
Empirical second moments:
\[
S_{rr}\leftarrow\tfrac{1}{B}Z_{r,c}^\top Z_{r,c},\quad
S_{ee}\leftarrow\tfrac{1}{B}Z_{e,c}^\top Z_{e,c},\quad
S_{re}\leftarrow\tfrac{1}{B}Z_{r,c}^\top Z_{e,c}
\]
\;
\ForEach{branch \(j\in\{r,e\}\)}{
Compute \(m_j\leftarrow\operatorname{tr}(S_{jj})/d_j\)\;
\eIf{the denominator in Equation~\eqref{eq:oas_rho} is numerically zero}{
Set \(\rho_j\leftarrow\rho_{\max}\)\;
}{
Compute \(\rho_j\) using Equation~\eqref{eq:oas_rho} with \(S_{jj},d_j,B\) and clip it to \([0,\rho_{\max}]\)\;
}
Set \(\lambda_j\leftarrow\rho_jm_j/(1-\rho_j)\) and \(R_{jj}\leftarrow S_{jj}+\lambda_j I_{d_j}\)\;
}
Cholesky, with diagonal jitter if needed:
\(L_rL_r^\top=R_{rr}\), \(L_eL_e^\top=R_{ee}\)\;
Whitened cross-covariance:
\(M\leftarrow L_r^{-1}S_{re}L_e^{-\top}\)\;
Thin SVD:
\(M=U\,\operatorname{diag}(s)\,V^\top\)\tcp*{\(s_i\) are regularized canonical correlations}
High-correlation set:
\(H\leftarrow\{i:s_i\ge\tau\}\),
\(U_H\leftarrow U[:,H]\)\;
Projector in whitened raw space:
\(K\leftarrow I_{d_r}-U_HU_H^\top\)\;
Shared canonical component removal operator:
\(T\leftarrow L_r^{-\top}KL_r^\top\)\;
Residual:
\(Z_{\perp}\leftarrow Z_{r,c}T\)\;
Optional removed component:
\(Z_{\parallel}\leftarrow Z_{r,c}-Z_{\perp}\)\;

\BlankLine
\textbf{Post-training calibration for deterministic inference}\;

\Indp
\emph{Input:} training-set embeddings \(\{(z_{r,i},z_{e,i})\}_{i=1}^N\) from the frozen encoders/projectors\;
Compute global means \(\mu_r^\star,\mu_e^\star\) and centered matrices \(Z_{r,c}^\star,Z_{e,c}^\star\)\;
Set \(B^\star\leftarrow\max(N-1,1)\) and compute global empirical moments
\(S_{rr}^\star,S_{ee}^\star,S_{re}^\star\) as above\;
Recompute OAS coefficients \(\rho_r^\star,\rho_e^\star\) and ridge levels \(\lambda_r^\star,\lambda_e^\star\) using the global moments and \(B^\star\)\;
Form \(R_{rr}^\star=S_{rr}^\star+\lambda_r^\star I_{d_r}\), \(R_{ee}^\star=S_{ee}^\star+\lambda_e^\star I_{d_e}\)\;
Compute frozen \(T^\star\) by repeating the Cholesky/SVD/projector steps with
\((R_{rr}^\star,R_{ee}^\star,S_{re}^\star)\) and the same threshold \(\tau\)\;
Store \((\mu_r^\star,\mu_e^\star,T^\star)\)\;
\Indm

\BlankLine
\textbf{Inference}\;

For a test example, compute embeddings \((z_r,z_e)\)\;
Residual:
\(z_{\perp}\leftarrow(z_r-\mu_r^\star)T^\star\)\;
Centered expert embedding:
\(z_e^{\mathrm{ctr}}\leftarrow z_e-\mu_e^\star\)\;
Predict:
\(\hat y\leftarrow h(z_{\perp},z_e^{\mathrm{ctr}})\)\;

\end{algorithm}
\section{Experimental Details: Synthetic}
\label{apptrain}

This appendix describes the image experiment in Section~\ref{sec:toy}.
We report seeds 907--911, separating operators fitted to common frozen
embeddings from models that learn their own embeddings. The comparisons
measure recovery of the generating variables and prediction of the task
label separately. The tables include the raw and expert learning controls,
linear and quadratic probes, native prediction heads, rank sensitivities,
and unsuccessful decomposition fits.

\subsection{Data generation and audit objective}
\label{app:freq-amp-data}

We use the variables defined in Section~\ref{sec:ped-example}, but give all
four image coefficients unit variance and set the expert correlation for
$A$ to $.95$. Specifically, let
$S,A,F_1,F_2,N_A,N_1,N_2$ be independent standard Gaussian variables and
write
\begin{align}
 Q&=.95A+\sqrt{.0975}N_A,&
 P_i&=.30F_i+\sqrt{.91}N_i,\quad i=1,2,\\
 c_r&=(S,A,F_1,F_2),&c_e&=(S,Q,P_1,P_2),\\
 y&=\mathbb I\{A+F_1F_2>0\}.&&
\end{align}
Here $c_r,c_e$ are rendering coefficients, not learned embeddings.
Their corresponding correlations before pixel noise are
$(1,.95,.30,.30)$, and every coefficient has variance one. In particular,
the raw coefficients are not $(S,.95A,.30F_1,.30F_2)$.
The unit amplitudes give the weak factors the same coefficient scale
as $A$; the unequal variances in Section~\ref{sec:ped-example} instead isolate
the effect of variance on the population decomposition.

The nuisance $S$ is shared but does not enter the label. The label depends
on both $A$ and the interaction $F_1F_2$, making both $A$ and the interaction relevant to task training. The audit objective is to reduce
linear recovery of $A$ from the residual while retaining $F_1,F_2$ and
their interaction. Recovery of $S$ is an additional diagnostic. The
expert image is informative about $A$, but the learned expert embedding
need not expose it. We therefore measure raw and expert accessibility
before interpreting the residual. None of the generating variables is
an additional training target.

\paragraph{Image rendering.}
Each view is a $64\times64$ grayscale image containing four disjoint
$16\times16$ tiles. Their zero-based row/column origins are
$(8,8),(8,40),(40,8),(40,40)$. For tile coordinates
$u,w\in\{0,\ldots,15\}$, the unnormalized patterns are
\[
\begin{aligned}
 t_1(u,w)&=\begin{cases}-1,&w<8,\\1,&w\ge8,\end{cases}
 &t_2(u,w)&=\sin(4\pi w/16),\\
 t_3(u,w)&=\sin(6\pi u/16),
 &t_{r,4}(u,w)&=2\bigl((\lfloor u/2\rfloor+\lfloor w/2\rfloor)\bmod2\bigr)-1,\\
 &&t_{e,4}(u,w)&=\sin(2\pi(u+w)/16).
\end{aligned}
\]
We center each pattern, normalize it to unit root-mean-square amplitude
within its tile, and set it to zero elsewhere. The first three templates
are the same across views; the fourth gives the corresponding raw and
expert factors different appearances. For view $v\in\{r,e\}$ and pixel
$p$, the rendered intensity is
\[
 I_v(p)=.5+.06\sum_{j=1}^4 c_{v,j}t_{v,j}(p)+\epsilon_v(p),
 \qquad \epsilon_v(p)\sim\mathcal N(0,.0192^2).
\]
Pixel noise is independent across pixels, images, and views, including
the background. It is distinct from the proxy noise $N_A,N_1,N_2$:
proxy noise controls sharing of the generating variables, whereas pixel
noise affects observation of both views. The chosen pixel scale gives
standard deviation $.02$ in a coefficient estimate obtained using the
known template, because $.0192=.02\times.06\times16$.
Disjoint tiles make the variables separately observable in the pixels;
they do not force a neural encoder to use separate coordinates.

Images remain unclipped floating-point values. We repeat the grayscale
channel three times and use ImageNet channel means
$(.485,.456,.406)$ and standard deviations $(.229,.224,.225)$.
Inputs remain $64\times64$; there is no resizing to 224 and no additional
image augmentation.

\subsection{Training, calibration, validation, and test data}
\label{app:training}

The data seed is 909 for all methods and neural initializations.
Each epoch uses 4,096 new paired images. The epoch-indexed image streams
are shared across methods and model seeds, while minibatch order is
shuffled using the model seed. Calibration uses 4,096 fixed independent
pairs, and validation uses 1,024 fixed independent pairs. These
calibration images come from the training distribution but are separate
from the images used for gradient updates. For the dual models, after each epoch we fit the
calibrated operator using all calibration embeddings with the network
in evaluation mode, and evaluate that frozen operator on validation.

All 35 neural runs in this five-seed panel completed 100 epochs. There
is no early stopping. We select the first strict minimum of the
validation objective over those epochs. The supervised arms use
classification cross-entropy; VCCA-private uses its reconstruction and
KL objective. VCCA validation uses a fixed Monte Carlo draw sequence
across epochs, so changes in its validation objective do not arise from
changing that sequence. Table~\ref{tab:synimg_checkpoints} lists the
selected epochs and native SCCR removal ranks.

After selecting each network, we freeze it and fit audit predictors on
the 4,096 epoch-zero training pairs, which were also seen during network
training. Validation selects the probe penalty independently for each
representation and target. Calibration, probe fitting, and validation
therefore use different observations. We evaluate on the same separate
4,096 test pairs for every method and initialization. Test images use
an independent random stream with seed 202609280909; no encoder, operator, scaler, or predictor
is refitted on them.

The reported seeds 907--911 are a retrospective subset of seven completed
initializations. The original seeds 909--913 were locked before the test
pairs were generated. Seeds 907 and 908 were added after those test
results had been inspected and reuse the same test cohort. The five-seed
tables therefore do not constitute a prospectively selected replication
on a new cohort. They retain seed 907's partial suppression and do not
establish success for every initialization in the larger study.

\begin{table}[htbp]
\caption{Selected checkpoint epochs after completing 100 epochs in every run. Epochs are numbered from 1. The final row gives the number of directions removed by each end-to-end SCCR model's calibrated operator at its selected checkpoint; these differ from the ranks fitted to common embeddings.}
\label{tab:synimg_checkpoints}
\centering
\setlength{\tabcolsep}{2pt}
\begin{tabular*}{\textwidth}{@{\extracolsep{\fill}}lccccc@{}}
\toprule
Model or diagnostic & 907 & 908 & 909 & 910 & 911\\
\midrule
Dual baseline & 86 & 99 & 98 & 91 & 100\\
Raw only & 83 & 83 & 82 & 75 & 83\\
SCCR & 100 & 100 & 100 & 99 & 100\\
OLS & 100 & 84 & 91 & 90 & 90\\
Regularized LEACE-style & 76 & 94 & 78 & 42 & 84\\
VCCA-private & 100 & 100 & 100 & 100 & 100\\
DSN-MMD & 94 & 81 & 66 & 81 & 36\\
\midrule
Native SCCR removed rank & 3 & 6 & 5 & 5 & 5\\
\bottomrule
\end{tabular*}
\end{table}

\subsection{Neural architecture and optimization}
\label{app:architecture}
\label{app:prediction-head}

All image encoders use EfficientNetV2-S~\citep{pmlr-v139-tan21a}
initialized with ImageNet-pretrained weights~\citep{5206848}. We remove
the original classifier and use its 1,280-dimensional pooled features.
All backbone weights are fine-tuned. In the raw-only and dual models,
we $\ell_2$-normalize these features and project them through
\[
 1280\ \longrightarrow\ 64\ \xrightarrow{\mathrm{SiLU},\ \mathrm{dropout}(.2)}\ 64.
\]
The arrows with dimensions denote linear layers, including biases.
Each dual model has separate raw and expert encoders and projectors.
Its two-logit classifier has branch-specific LayerNorm, a $64\to64$
linear layer, SiLU, and dropout $.2$. Writing the resulting hidden
vectors as $u_r,u_e$, its logits are
\[
 \ell=W_ru_r+W_eu_e+W_\times(u_r\odot u_e)+b,
 \qquad
 u_r=\phi_r(z_\perp),\quad u_e=\phi_e(z_e^{\mathrm{ctr}}).
\]
Thus the classifier can use expert evidence, residual evidence, and their
interaction. The raw-only classifier uses LayerNorm, a $64\to64$ layer,
SiLU, dropout $.2$, and a $64\to2$ output layer.

All neural arms use Adam with learning rate $10^{-4}$,
$(\beta_1,\beta_2)=(.9,.999)$, numerical constant $10^{-8}$, and zero
weight decay. Batch size 512 gives eight updates per epoch and 800
updates in total. Evaluation batch size is 128. We use no learning-rate
schedule or gradient clipping. The raw-only and dual arms minimize two-class softmax
cross-entropy. VCCA-private and DSN use the losses described below.
The synthetic experiments run on a MacBook Pro with an Apple M5 Max
chip and 128\,GB of RAM. Neural operations use MPS; differentiable covariance calculations use
CPU float64, retaining gradients through the device transfers.

\subsection{End-to-end residualization models}
\label{app:sccr}

The dual baseline uses centered raw embeddings without removal. SCCR,
OLS, and regularized LEACE-style replace this operation while retaining
the same encoders, projectors, supervised head, and task loss. For these
models, removal is part of training; the encoders are not frozen while
the residualizer is learned. Training uses minibatch means and moments.
After each epoch, the complete calibration set gives new means and a
frozen operator for validation. The selected checkpoint stores its own
calibrated state, which is reused for test prediction.

\paragraph{SCCR.}
For a batch of $n$ embeddings, we center $Z_r,Z_e$ and form the second
moments with denominator $n-1$. Branchwise OAS coefficients give the
ridge levels in Section~\ref{sec:method}. With
$R_{rr}=S_{rr}+\lambda_r I=L_rL_r^\top$ and
$R_{ee}=S_{ee}+\lambda_e I=L_eL_e^\top$, we compute
\[
 M=L_r^{-1}S_{re}L_e^{-\top}
   =U\operatorname{diag}(s)V^\top,
 \qquad H=\{i:s_i\ge .40\},
\]
\[
 T=L_r^{-\top}(I-U_HU_H^\top)L_r^\top,
 \qquad Z_\perp=Z_{r,c}T.
\]
OAS is recalculated on each minibatch and again on the complete
calibration set. The shrinkage coefficient is clipped to $[0,1-10^{-6}]$;
the implementation uses $10^{-12}$ in its denominator safeguard and
floors $1-\rho_j$ at $10^{-6}$ when forming the ridge.
Diagonal jitter is added only when needed for Cholesky factorization,
starting at $10^{-6}$ and increasing by a factor of ten if necessary. The cross-covariance is not separately shrunk. The map is
identity when no direction is selected and zero when all raw directions
are selected.

The implementation differentiates the selected subspace, avoiding
individual singular-vector derivatives within a selected or retained
subspace. More precisely, if $MM^\top=U\operatorname{diag}(\nu)U^\top$
and $b_i$ indicates removal, its projector derivative uses
$(b_i-b_j)/(\nu_i-\nu_j)$ only between the two groups, and zero within
each group. The rank decision itself is discrete. This does not remove
the need for a gap at the threshold and does not guarantee stable
end-to-end optimization. At inference, the saved means and $T^\star$
are fixed; the residual map requires only the raw embedding, while the
classifier also uses the expert embedding.

\paragraph{OLS.}
On the centered batch we fit
\[
 B_{\mathrm{OLS}}=Z_{e,c}^{\dagger}Z_{r,c},
 \qquad Z_\perp=Z_{r,c}-Z_{e,c}B_{\mathrm{OLS}}.
\]
The pseudoinverse uses a relative singular-value cutoff of $10^{-6}$,
without an additional ridge. The batch fit is differentiable; global
calibration replaces it for validation and inference. OLS requires both
views to form a new residual. Given its fixed map, $(z_e^{\mathrm{ctr}},
z_\perp)$ reconstructs $z_r^{\mathrm{ctr}}$, so residual-only suppression
need not imply loss of information from the combined representation.

\paragraph{Regularized LEACE-style removal.}
We whiten the raw branch using its OAS-regularized covariance and take
an SVD of $L_r^{-1}S_{re}$. Let $U_H$ contain the left singular vectors
whose singular values exceed $10^{-6}$ times the largest one. The
row-vector map is
\[
 T_{\mathrm{R\text{-}LEACE}}
 =L_r^{-\top}(I-U_HU_H^\top)L_r^\top.
\]
This arm uses the same batch fitting, calibration, and supervised
training protocol as SCCR. It differs from SCCR by targeting every
numerically supported cross-covariance direction, without expert
whitening or a correlation threshold. It also differs from exact LEACE:
the raw covariance is regularized and the support cutoff is $10^{-6}$.
We label it R-LEACE in the tables rather than treating it as an exact
implementation of the unregularized erasure guarantee.

\subsection{Shared/private neural models}
\label{app:baselines}

These baselines were not designed for the specific audit objective in
this paper. We evaluate their raw private code as the retained output,
and report shared and full raw codes as learning controls. They use the
same images, 100-epoch budget, optimizer, and trainable pretrained
backbones, but their objectives and numbers of encoders differ. The 1,280-dimensional backbone features are also $\ell_2$-normalized
before their latent-code layers. Equal epoch budgets therefore do not
imply matched capacity or computation.

\paragraph{VCCA-private.}
Our adaptation of VCCA-private~\citep{wang2017deepvariationalcanonicalcorrelation}
has three independent encoders defining diagonal Gaussian posteriors
$q(c\mid x)$, $q(p_r\mid x)$, and $q(p_e\mid\tilde x)$ for a shared
code and two private codes. Each uses a 64-unit hidden layer with SiLU
and dropout $.2$, followed by separate 64-dimensional mean and
log-variance outputs. Log variances are clipped to $[-12,8]$.
The shared posterior takes the raw image as input; there is no second
shared posterior inferred from the expert image. With
$\widetilde I_v=(I_v-.5)/.06$, we minimize
\[
\begin{aligned}
 \mathcal L_{\mathrm{VCCA}}={}&
 \mathbb E_q[-\log p(\widetilde I_r\mid c,p_r)
              -\log p(\widetilde I_e\mid c,p_e)]\\
 &+D_{\mathrm{KL}}(q(c\mid x)\|\mathcal N(0,I))
  +D_{\mathrm{KL}}(q(p_r\mid x)\|\mathcal N(0,I))
  \\
 &+D_{\mathrm{KL}}(q(p_e\mid\tilde x)\|\mathcal N(0,I)).
\end{aligned}
\]
Both Gaussian observation models have unit variance in these scaled
pixel units. Reconstruction losses sum over pixels and average over
images; KL terms sum over latent coordinates and average over images.
All five terms have weight one. We use one reparameterized draw per
posterior during training and validation. There is no supervised task
head or auxiliary factor loss. Validation negative ELBO selects the
checkpoint. Audit predictors use posterior means: shared and private
codes have 64 coordinates each, and their concatenation has 128.

Two convolutional decoders reconstruct raw and expert images from the
corresponding concatenated shared/private codes. Each maps its input
to $128\times4\times4$ features, then uses four bilinear upsampling
stages with $3\times3$ convolutions and SiLU, with channel widths
$64,32,16,16$. A final convolution gives an unconstrained one-channel
$64\times64$ image.

\paragraph{Domain Separation Networks.}
Our DSN-MMD adaptation~\citep{NIPS2016_45fbc6d3} has one shared encoder
used on both views and independent raw and expert private encoders,
all producing 64-dimensional codes. A single decoder with the above
architecture reconstructs each view from the \emph{sum} of its shared
and private codes. The native classifier predicts $y$ from the raw
shared code only; expert images act as an unlabeled target domain.
The loss is
\[
 \mathcal L_{\mathrm{DSN}}=
 \mathcal L_{\mathrm{CE}}+\mathcal L_{\mathrm{recon}}
 +.1\mathcal L_{\mathrm{diff}}+.1\mathcal L_{\mathrm{MMD}}.
\]
For reconstruction errors $e_{v,i,p}$ in scaled pixel units and $P=4096$
pixels, we use
\[
 \mathcal L_{\mathrm{recon}}=
 \sum_{v\in\{r,e\}}\frac1n\sum_{i=1}^n
 \left[\frac1P\sum_p e_{v,i,p}^2-
       \left(\frac1P\sum_p e_{v,i,p}\right)^2\right].
\]
After column centering and row $\ell_2$ normalization of shared and
private code matrices $H_{c,v},H_{p,v}$, the difference penalty is
\[
 \mathcal L_{\mathrm{diff}}=
 \sum_{v\in\{r,e\}}\left\|H_{c,v}^\top H_{p,v}/n\right\|_F^2.
\]
The similarity penalty is the biased squared MMD between raw and expert
shared codes, including diagonal terms:
\[
 \mathcal L_{\mathrm{MMD}}=\frac1{n^2}\sum_{i,j}
 \bigl[k(c_{r,i},c_{r,j})+k(c_{e,i},c_{e,j})-2k(c_{r,i},c_{e,j})\bigr],
\]
where $k(a,b)$ averages $\exp(-\|a-b\|^2/(2\sigma^2))$ over
$\sigma\in\{.125,.25,.5,1,2,4,8\}$. There is no paired-code equality
loss. Validation classification loss selects the checkpoint. The raw
private code is the audit output; the shared code and the 128-dimensional
shared/private concatenation are controls.

\subsection{Operators fitted to common frozen embeddings}
\label{app:synimg_common_methods}

For each seed, the selected dual baseline supplies the same frozen
64-dimensional raw and expert embeddings to all operators below.
Every decomposition is fitted to calibration embeddings without label
or factor supervision. Its map is then fixed for probe fitting,
validation, and test evaluation. These measurements isolate the
operator comparison from changes in encoder learning.

\paragraph{SCCR, matched top-$q$, and exact LEACE.}
Common-embedding SCCR uses OAS, $\tau=.40$, and the same residual map
as above. Its symmetric covariance-root implementation is equivalent
to the Cholesky convention. The selected ranks are $2,2,2,3,3$ for
seeds 907--911. Removing the top $q$ canonical directions with this
same rank and spectrum gives identical saved test scores. This is an
equivalence control, not a separate algorithmic advantage of thresholding.

Exact empirical LEACE~\citep{10.5555/3666122.3669006} uses the
unregularized raw covariance. With $W=S_{rr}^{\dagger/2}$ and
$P_{\mathrm{cross}}$ the orthogonal projector onto the column space of
$WS_{re}$, the row-vector eraser is
\[
 T_{\mathrm{LEACE}}=I-WP_{\mathrm{cross}}S_{rr}^{1/2}.
\]
Covariance support and cross-covariance support use relative cutoff
$10^{-12}$. The selected ranks are $63,64,64,64,64$.
Residual scores close to zero can be affected by finite precision and
probe fitting; we report measured values rather than replacing them by
analytic zeros. Raw and zero outputs are additional controls. Feature-space
projection has no separately measured image row in this suite.

\paragraph{D-CCA.}
We use the empirical signal-covariance version of decomposition-based
CCA~\citep{Shu02012020}, without S-POET denoising. This is not Deep CCA.
Let $W_r=S_{rr}^{\dagger/2}$ and $W_e=S_{ee}^{\dagger/2}$ on their
numerical supports, using covariance denominator $n$ in this implementation.
Let $r_r,r_e$ be their support ranks and retain $q=\min(r_r,r_e)$ pairs
from $W_rS_{re}W_e=U\operatorname{diag}(s)V^\top$. For the paired canonical
scores $C_r=Z_{r,c}W_rU_q$, $C_e=Z_{e,c}W_eV_q$, define
\[
 a(s)=\frac12\left(1-\sqrt{\frac{1-s}{1+s}}\right),
 \qquad C_{\mathrm{common}}=(C_r+C_e)\operatorname{diag}(a(s)).
\]
The distinctive raw output is
\[
 Z_\perp=Z_{r,c}-C_{\mathrm{common}}U_q^\top S_{rr}^{1/2}.
\]
The $10^{-12}$ relative cutoff applies to covariance eigenvalues,
not to the canonical correlations. All $q$ pairs are used; unpaired
raw modes are left unchanged. There is no additional statistical rank
selector or ridge. Unlike fitted
SCCR and LEACE, this output requires both views at inference.

\paragraph{JIVE.}
JIVE~\citep{10.1214/12-AOAS597} approximates each centered, block-scaled
matrix $\bar Z_v=Z_{v,c}/\|Z_{v,c}\|_F$ by a joint component $J_v$
and an individual component $I_v$. The fixed-rank objective is
\[
 \min_{J_r,J_e,I_r,I_e}
 \sum_{v\in\{r,e\}}\|\bar Z_v-J_v-I_v\|_F^2,
 \qquad
 \operatorname{rank}[J_r\ J_e]\le q,\quad
 \operatorname{rank}(I_v)\le k-q,
\]
with joint/individual sample-score orthogonality. The implementation
alternates truncated SVD updates and applies a final orthogonalization
that preserves the fitted sum. This procedure does not certify a global
optimum. We use $k\in\{4,8,16\}$ and
$q\in\{0,1,2\}$, with no expansion of the individual ranks. The primary
setting is $k=4,q=2$. This is a fixed-rank adaptation; JIVE's permutation
rank selector is not used. Block Frobenius scaling is not feature
whitening.

The original limit is 5,000 iterations. Convergence requires the change
in the concatenated fitted matrix, divided by the Frobenius norm of the
concatenated input, to be below $10^{-9}$. All 30 positive-$q$ fits
across these five seeds missed this criterion. Their individual and
remainder outputs, with both inference extensions below, are recorded
as nonconverged and excluded from scoring. They are not assigned zero
performance. The $q=0$ controls involve no joint removal. A later increase
to 50,000 iterations is reported separately in
Section~\ref{sec:synimg_common_jive_followup}.

\paragraph{AJIVE.}
AJIVE~\citep{10.1016/j.jmva.2018.03.008} first retains initial sample-score
bases from truncated SVDs of the two centered matrices,
\[
 Z_{v,c}\approx Q_vD_vV_v^\top,
 \qquad Q_v\in\mathbb R^{n\times r_v},\quad v\in\{r,e\}.
\]
The initial rank $r_v$ determines which directions enter the alignment
calculation, rather than directly specifying the joint rank. The leading
squared singular values of $[Q_r\ Q_e]$ are
$1+\sigma_i(Q_r^\top Q_e)$. AJIVE tests these alignments against Wedin
perturbation and random-direction bounds.

We use 1,000 resamples for each bound. The Wedin calculation samples
subspaces in the complements of the retained sample scores and feature
loadings. Dividing the larger projected residual operator norm ($\|\cdot\|_2$) by the last retained
singular value, and clipping to one, gives a bound $b_v$. Its joint
squared-singular-value samples are
$\max\{2-b_r^2-b_e^2,1\}$. The cutoff is the larger of their fifth
percentile and the 95th percentile of the leading squared singular value
from independent random score spaces of ranks $r_r,r_e$. The implementation
uses a $10^{-10}$ numerical boundary tolerance and the subsequent
per-view signal-strength check. If $Q_J$ contains the accepted joint
scores, $J_v=Q_JQ_J^\top Z_{v,c}$. The individual component truncates
$Z_{v,c}-J_v$ using the original per-view singular-value threshold,
which is the midpoint between the last initially retained and first
discarded singular values. There is no additional cap of $r_v-q$.

The primary initial-rank heuristic selects the largest adjacent
singular-value ratio over ranks 1--8 separately in each view. It selects
$(2,2)$ and joint rank zero in every included seed. We also report fixed
initial ranks $(4,4)$. Both settings use the paired individual output
defined below. No feature whitening or block scaling is applied to AJIVE.
We report the initial, joint, and final individual ranks for both settings.

\paragraph{Prediction on new images.}
JIVE and AJIVE decompose calibration sample scores. To evaluate new
images, we fit explicit least-squares extensions using only calibration
rows. Let $s_v=\|Z_{v,c}\|_F$ for JIVE and $s_v=1$ for AJIVE,
and write $\bar Z_v=Z_{v,c}/s_v$. AJIVE uses the paired matrix
$V=[\bar Z_r\ \bar Z_e]$ and reports the individual component.
For JIVE, we also evaluate the raw-only matrix $V=\bar Z_r$ and the
raw-minus-joint remainder. For calibration components $\bar I_r,\bar J_r$
in these same coordinates, set
\[
 B_I=V^\dagger\bar I_r,\qquad B_J=V^\dagger\bar J_r.
\]
For a new centered and scaled raw row $\bar z_r$ and its corresponding
raw-only or paired row $v$, the outputs in raw embedding units are
\[
 r_{\mathrm{individual}}=s_rvB_I,\qquad
 r_{\mathrm{remainder}}=s_r(\bar z_r-vB_J).
\]
The pseudoinverse cutoff is $10^{-12}$. The individual
output excludes unmodeled variation; the remainder retains it.
A paired extension uses both views even when its downstream predictor
uses the residual alone. A raw-only extension uses no expert observation
at inference. The AJIVE tables use only the paired individual output;
the additional extensions and remainder outputs are JIVE controls.
These are audit adaptations, not a fresh decomposition on test rows.
D-CCA also requires paired observations, while the fitted SCCR residual
is a map of the raw embedding alone.

\subsection{Probe fitting and evaluation}
\label{app:evaluation}

For a retained code $r$, linear probes use its standardized coordinates;
quadratic probes additionally use every product $r_ir_j$, $i\le j$.
All feature transformations are fitted on the probe training set.
We drop columns whose range is below $10^{-12}$ times the larger of one
and their maximum absolute value. The input standard-deviation floor
is $\max\{10^{-6},.01\max_j\operatorname{SD}(r_j)\}$.
After quadratic expansion, we center and scale again using an absolute
floor of one and the same relative floor of $.01$ times the largest
column standard deviation. These floors limit amplification of very
small numerical residuals.

Separate ridge regressions predict $S,A,F_1,F_2,F_1F_2$.
Logistic regressions predict $y$. Both use the penalty grid
$\eta\in\{10^{-6},10^{-4},.01,1,100\}$, with ridge parameter
$\alpha=n\eta$ and logistic parameter $C=1/(n\eta)$, where $n=4096$.
Models include an intercept. Ridge uses a Cholesky solver; logistic
regression uses L-BFGS with at most 10,000 iterations and tolerance
$10^{-6}$. Validation MSE selects each factor predictor independently;
validation cross-entropy selects each label predictor. The training-mean
and training-class-prior predictors are candidates and win exact ties,
followed by the larger penalty. Failed or nonconverged candidates cannot
be selected.

For test values $a_i$ and predictions $\widehat a_i$, we report
\[
 R_A^2=1-\frac{\sum_i(a_i-\widehat a_i)^2}
                  {\sum_i(a_i-\bar a_{\mathrm{test}})^2},
\]
with the same definition for each other target. Negative values are
retained. These are finite-sample, validation-selected probe scores,
not the analytic population recovery values in
Section~\ref{sec:ped-example}. AUROC uses continuous class probabilities;
accuracy uses threshold $.5$; log loss is mean binary cross-entropy.
The native-head metrics evaluate the trained classifier without fitting
a new readout. They are separate from linear and quadratic task probes.
VCCA-private has no native task head; DSN's native head uses its shared,
not its private, raw code.

For the dual arms we report their own expert code and own expert plus
residual, as well as combinations with the same-seed dual baseline's
expert code. The latter supplies a common expert reference. VCCA-private
and DSN expert-side codes were not exported for this audit; their saved
expert fields refer to that common reference. We therefore do not present
those fields as measurements of their own expert encoders. Full and
shared raw codes provide their recorded learning controls.

Unless a table states otherwise, every aggregate is the mean and sample
standard deviation over all five seeds, with SD denominator $5-1$.
The seeds share the generating streams and test observations, so these
SDs measure initialization variation conditional on the data, not
confidence intervals over independent datasets. A low probe score
establishes poor accessibility under that probe family; it does not
establish absence of information, nonlinear independence, or causal
relevance.

\subsection{What the image results establish}

Table~\ref{tab:synimg_native_factors_linear} shows the raw, expert, and retained
codes before comparing methods. In end-to-end SCCR, mean linear $A$
recovery falls from $.890$ in the raw code to $.051$ in the residual,
while $F_1,F_2$ recovery remains $.708$ and $.644$. Seed 907 has partial
suppression ($.8570\to.2149$), while residual linear $A$ recovery in the
other four seeds is $.0052$--$.0163$. Quadratic residual $A$ recovery
remains $.818\pm.057$. This supports reduced linear accessibility,
not complete removal of $A$. Native-head AUROC is $.9928\pm.0025$.

The end-to-end OLS comparison also changes what the expert encoder
learns: mean own-expert linear $A$ recovery is $.144\pm.308$, compared
with $.809\pm.128$ for SCCR. Its larger residual $A$ recovery therefore
cannot be attributed to the fixed operator alone. The AJIVE comparisons
concern the paired individual output under the two stated initial-rank
settings; they do not establish failure across other rank or output
choices.

For VCCA-private, weak-factor recovery is limited even in the full code,
so low private recovery combines learning and allocation limits. DSN's
shared code predicts the interaction, while its private code has little
interaction recovery. Its good native prediction therefore does not
imply a useful private audit representation. R-LEACE also has weak
raw-code interaction recovery. These controls distinguish failure to
learn an accessible variable from its loss in a retained output.

The learned image embeddings do not reproduce every premise of the
Gaussian calculation. In particular, the common expert code exposes
$A$ but little of the individual weak factors under either probe family.
The unequal-variance image follow-up using the coefficients of
Section~\ref{sec:ped-example} did not reliably learn those weak factors,
so it did not establish the same suppression-and-retention result after
image encoding. In a separate unit-amplitude $\rho=.70$, seed-42 pilot,
the weak factors were accessible, but SCCR's learned expert exposed
almost no $A$ under either probe family. Linear raw-to-residual $A$
recovery then remained $.877\to.876$. These separate exploratory
experiments show why factor learning and expert quality must be checked;
they are not included in the five-seed aggregates below.

Finally, higher residual task AUROC is not itself the audit objective:
the label depends on $A$ as well as $F_1F_2$, and retaining $A$ can improve
prediction. The relevant result is the joint pattern of suppression,
retention, expert accessibility, and combined prediction under the
stated representations and inference rules.

\subsection{Complete results for learned representations}
\label{app:synimg_learned_results}

\paragraph{Learned-representation results.}
The following tables concern separately trained neural arms.
``Raw'' is the embedding before residualization and ``own expert'' is the
expert embedding learned by that same arm. For VCCA-private and DSN-MMD,
``full'' concatenates the raw-view shared and private codes.
Their own expert-side codes were not exported for this audit; the saved
expert-reference fields are therefore not reported as their own expert
representations. R-LEACE denotes the regularized trainable LEACE-style arm.
The dual baseline's residual is centered raw, and the raw-only residual
is the raw code; these duplicate representations are shown once.
Factor tables retain negative test-set $R^2$. Aggregate entries are
mean $\pm$ sample standard deviation over seeds 907--911; they describe
initialization variation on the shared assessment cohort.

\begingroup
\setlength{\tabcolsep}{2pt}
\begin{longtable}{@{\extracolsep{\fill}}lllccccc@{}}
\caption{Factor recovery from separately trained representations. Test-set $R^2$ is mean $\pm$ sample SD over seeds 907--911; each code and target has its own development-selected readout. L/Q denote linear/quadratic probes. Full/shared/private are raw-view codes; own expert refers to the same supervised arm.}\label{tab:synimg_native_factors_linear}
\label{tab:synimg_native_factors_quadratic} \\
\toprule
Method & Code & Probe & $R_S^2$ & $R_A^2$ & $R_{F_1}^2$ & $R_{F_2}^2$ & $R_{F_1F_2}^2$ \\
\midrule
\endfirsthead
\caption[]{Factor recovery from separately trained representations (continued).} \\
\toprule
Method & Code & Probe & $R_S^2$ & $R_A^2$ & $R_{F_1}^2$ & $R_{F_2}^2$ & $R_{F_1F_2}^2$ \\
\midrule
\endhead
\midrule
\multicolumn{8}{r}{Continued on next page} \\
\endfoot
\bottomrule
\endlastfoot
Dual baseline & Raw & L & $.053\!\pm\!.023$ & $.826\!\pm\!.021$ & $.693\!\pm\!.028$ & $.634\!\pm\!.032$ & $.726\!\pm\!.022$ \\
Dual baseline & Own expert & L & $.166\!\pm\!.059$ & $.798\!\pm\!.025$ & $-.001\!\pm\!.001$ & $-.000\!\pm\!.002$ & $-.001\!\pm\!.001$ \\
Raw only & Raw & L & $.049\!\pm\!.023$ & $.802\!\pm\!.012$ & $.717\!\pm\!.020$ & $.638\!\pm\!.017$ & $.721\!\pm\!.015$ \\
SCCR & Raw & L & $.122\!\pm\!.041$ & $.890\!\pm\!.023$ & $.710\!\pm\!.043$ & $.647\!\pm\!.084$ & $.702\!\pm\!.031$ \\
SCCR & Own expert & L & $.399\!\pm\!.051$ & $.809\!\pm\!.128$ & $-.000\!\pm\!.002$ & $.002\!\pm\!.012$ & $-.000\!\pm\!.000$ \\
SCCR & Residual & L & $.115\!\pm\!.044$ & $.051\!\pm\!.092$ & $.708\!\pm\!.044$ & $.644\!\pm\!.082$ & $.699\!\pm\!.033$ \\
OLS & Raw & L & $.043\!\pm\!.017$ & $.824\!\pm\!.025$ & $.662\!\pm\!.025$ & $.609\!\pm\!.054$ & $.743\!\pm\!.022$ \\
OLS & Own expert & L & $.491\!\pm\!.069$ & $.144\!\pm\!.308$ & $.018\!\pm\!.011$ & $-.002\!\pm\!.008$ & $-.000\!\pm\!.000$ \\
OLS & Residual & L & $.007\!\pm\!.005$ & $.669\!\pm\!.272$ & $.628\!\pm\!.022$ & $.593\!\pm\!.052$ & $.734\!\pm\!.021$ \\
R-LEACE & Raw & L & $.459\!\pm\!.135$ & $.134\!\pm\!.077$ & $.202\!\pm\!.116$ & $.074\!\pm\!.065$ & $-.003\!\pm\!.005$ \\
R-LEACE & Own expert & L & $.160\!\pm\!.046$ & $.869\!\pm\!.009$ & $.001\!\pm\!.003$ & $-.003\!\pm\!.001$ & $-.002\!\pm\!.003$ \\
R-LEACE & Residual & L & $-.000\!\pm\!.002$ & $.000\!\pm\!.001$ & $.001\!\pm\!.003$ & $.003\!\pm\!.008$ & $-.000\!\pm\!.000$ \\
VCCA-private & Full & L & $.998\!\pm\!.001$ & $.868\!\pm\!.242$ & $.157\!\pm\!.103$ & $.076\!\pm\!.041$ & $-.002\!\pm\!.003$ \\
VCCA-private & Shared & L & $.990\!\pm\!.007$ & $.526\!\pm\!.445$ & $.050\!\pm\!.029$ & $.020\!\pm\!.023$ & $-.002\!\pm\!.003$ \\
VCCA-private & Private & L & $.879\!\pm\!.170$ & $.841\!\pm\!.267$ & $.110\!\pm\!.085$ & $.049\!\pm\!.041$ & $-.001\!\pm\!.001$ \\
DSN-MMD & Full & L & $.415\!\pm\!.231$ & $.915\!\pm\!.054$ & $.683\!\pm\!.058$ & $.629\!\pm\!.044$ & $.733\!\pm\!.051$ \\
DSN-MMD & Shared & L & $.015\!\pm\!.012$ & $.915\!\pm\!.054$ & $.625\!\pm\!.018$ & $.591\!\pm\!.036$ & $.733\!\pm\!.047$ \\
DSN-MMD & Private & L & $.402\!\pm\!.235$ & $.068\!\pm\!.091$ & $.315\!\pm\!.215$ & $.157\!\pm\!.169$ & $.001\!\pm\!.002$ \\
Dual baseline & Raw & Q & $.111\!\pm\!.057$ & $.919\!\pm\!.012$ & $.887\!\pm\!.022$ & $.812\!\pm\!.040$ & $.898\!\pm\!.013$ \\
Dual baseline & Own expert & Q & $.217\!\pm\!.059$ & $.817\!\pm\!.023$ & $-.002\!\pm\!.004$ & $-.001\!\pm\!.001$ & $-.001\!\pm\!.001$ \\
Raw only & Raw & Q & $.153\!\pm\!.058$ & $.913\!\pm\!.011$ & $.892\!\pm\!.009$ & $.831\!\pm\!.025$ & $.891\!\pm\!.010$ \\
SCCR & Raw & Q & $.255\!\pm\!.070$ & $.960\!\pm\!.014$ & $.878\!\pm\!.030$ & $.823\!\pm\!.071$ & $.876\!\pm\!.015$ \\
SCCR & Own expert & Q & $.540\!\pm\!.076$ & $.822\!\pm\!.128$ & $-.001\!\pm\!.004$ & $-.002\!\pm\!.001$ & $-.000\!\pm\!.000$ \\
SCCR & Residual & Q & $.162\!\pm\!.077$ & $.818\!\pm\!.057$ & $.856\!\pm\!.030$ & $.795\!\pm\!.067$ & $.874\!\pm\!.018$ \\
OLS & Raw & Q & $.131\!\pm\!.028$ & $.917\!\pm\!.012$ & $.866\!\pm\!.012$ & $.785\!\pm\!.035$ & $.887\!\pm\!.017$ \\
OLS & Own expert & Q & $.551\!\pm\!.110$ & $.147\!\pm\!.326$ & $.007\!\pm\!.007$ & $.000\!\pm\!.001$ & $-.001\!\pm\!.001$ \\
OLS & Residual & Q & $.085\!\pm\!.032$ & $.873\!\pm\!.007$ & $.837\!\pm\!.014$ & $.750\!\pm\!.044$ & $.871\!\pm\!.019$ \\
R-LEACE & Raw & Q & $.516\!\pm\!.182$ & $.154\!\pm\!.068$ & $.234\!\pm\!.194$ & $.063\!\pm\!.086$ & $-.000\!\pm\!.000$ \\
R-LEACE & Own expert & Q & $.262\!\pm\!.044$ & $.879\!\pm\!.009$ & $-.001\!\pm\!.000$ & $-.002\!\pm\!.000$ & $-.000\!\pm\!.000$ \\
R-LEACE & Residual & Q & $.001\!\pm\!.003$ & $-.000\!\pm\!.000$ & $.001\!\pm\!.003$ & $.004\!\pm\!.009$ & $-.000\!\pm\!.000$ \\
VCCA-private & Full & Q & $.998\!\pm\!.000$ & $.923\!\pm\!.152$ & $.253\!\pm\!.093$ & $.159\!\pm\!.073$ & $-.003\!\pm\!.006$ \\
VCCA-private & Shared & Q & $.997\!\pm\!.002$ & $.632\!\pm\!.433$ & $.100\!\pm\!.051$ & $.050\!\pm\!.050$ & $-.007\!\pm\!.009$ \\
VCCA-private & Private & Q & $.934\!\pm\!.096$ & $.876\!\pm\!.217$ & $.164\!\pm\!.106$ & $.057\!\pm\!.127$ & $-.001\!\pm\!.001$ \\
DSN-MMD & Full & Q & $.500\!\pm\!.246$ & $.905\!\pm\!.091$ & $.668\!\pm\!.111$ & $.507\!\pm\!.210$ & $.709\!\pm\!.096$ \\
DSN-MMD & Shared & Q & $.073\!\pm\!.024$ & $.944\!\pm\!.039$ & $.745\!\pm\!.027$ & $.697\!\pm\!.038$ & $.822\!\pm\!.037$ \\
DSN-MMD & Private & Q & $.497\!\pm\!.249$ & $.096\!\pm\!.127$ & $.388\!\pm\!.231$ & $.205\!\pm\!.220$ & $.009\!\pm\!.014$ \\
\end{longtable}
\endgroup

\begingroup
\small
\setlength{\tabcolsep}{2pt}
\begin{longtable}{@{\extracolsep{\fill}}llcccccccccc@{}}
\caption{Per-seed factor recovery: supervised dual and raw-only models. Entries are test-set $R^2$; linear and quadratic readouts for each target were selected separately. Raw is before removal; own expert is learned by the same arm. Full concatenates raw-view shared and private codes. Duplicate raw/residual controls are shown once; shared/private own-expert codes were not exported. Dual, Raw, VCCA, and DSN denote Dual baseline, Raw only, VCCA-private, and DSN-MMD.}\label{tab:synimg_native_identity_seeds}
\label{tab:synimg_native_raw_only_seeds}
\label{tab:synimg_native_sccr_seeds}
\label{tab:synimg_native_ols_seeds} \\
\toprule
\multicolumn{2}{l}{} & \multicolumn{5}{c}{Linear probes} & \multicolumn{5}{c}{Quadratic probes} \\
\cmidrule(lr){3-7}\cmidrule(l){8-12}
Method/seed & Code & $R_S^2$ & $R_A^2$ & $R_{F_1}^2$ & $R_{F_2}^2$ & $R_{F_1F_2}^2$ & $R_S^2$ & $R_A^2$ & $R_{F_1}^2$ & $R_{F_2}^2$ & $R_{F_1F_2}^2$ \\
\midrule
\endfirsthead
\caption[]{Per-seed factor recovery: supervised dual and raw-only models (continued).} \\
\toprule
\multicolumn{2}{l}{} & \multicolumn{5}{c}{Linear probes} & \multicolumn{5}{c}{Quadratic probes} \\
\cmidrule(lr){3-7}\cmidrule(l){8-12}
Method/seed & Code & $R_S^2$ & $R_A^2$ & $R_{F_1}^2$ & $R_{F_2}^2$ & $R_{F_1F_2}^2$ & $R_S^2$ & $R_A^2$ & $R_{F_1}^2$ & $R_{F_2}^2$ & $R_{F_1F_2}^2$ \\
\midrule
\endhead
\midrule
\multicolumn{12}{r}{Continued on next page} \\
\endfoot
\bottomrule
\endlastfoot
Dual 907 & Raw & $.0528$ & $.8247$ & $.7020$ & $.6302$ & $.7442$ & $.0810$ & $.9108$ & $.8635$ & $.7901$ & $.8996$ \\
Dual 907 & Own expert & $.2349$ & $.8014$ & $.0013$ & $-.0014$ & $-.0002$ & $.2767$ & $.8195$ & $.0006$ & $-.0014$ & $-.0007$ \\
Dual 908 & Raw & $.0244$ & $.8471$ & $.6451$ & $.6573$ & $.7385$ & $.0495$ & $.9291$ & $.8850$ & $.8528$ & $.9096$ \\
Dual 908 & Own expert & $.1100$ & $.7642$ & $-.0026$ & $.0028$ & $-.0032$ & $.1186$ & $.7885$ & $-.0006$ & $-.0001$ & $-.0021$ \\
Dual 909 & Raw & $.0568$ & $.8387$ & $.7168$ & $.6581$ & $.7266$ & $.0952$ & $.9271$ & $.9178$ & $.8559$ & $.9074$ \\
Dual 909 & Own expert & $.1382$ & $.7812$ & $-.0009$ & $-.0000$ & $-.0003$ & $.2234$ & $.8001$ & $-.0006$ & $-.0009$ & $-.0003$ \\
Dual 910 & Raw & $.0869$ & $.8277$ & $.6978$ & $.5807$ & $.6888$ & $.1999$ & $.9268$ & $.8698$ & $.7675$ & $.8761$ \\
Dual 910 & Own expert & $.2246$ & $.8152$ & $-.0006$ & $-.0010$ & $-.0012$ & $.2327$ & $.8300$ & $-.0006$ & $-.0014$ & $-.0003$ \\
Dual 911 & Raw & $.0435$ & $.7924$ & $.7026$ & $.6439$ & $.7295$ & $.1279$ & $.9036$ & $.8964$ & $.7947$ & $.8966$ \\
Dual 911 & Own expert & $.1206$ & $.8266$ & $-.0006$ & $-.0010$ & $-.0003$ & $.2346$ & $.8475$ & $-.0097$ & $-.0010$ & $-.0003$ \\
Raw 907 & Raw & $.0438$ & $.8081$ & $.7305$ & $.6442$ & $.7010$ & $.1575$ & $.9186$ & $.9026$ & $.8489$ & $.8850$ \\
Raw 908 & Raw & $.0605$ & $.7897$ & $.7424$ & $.6621$ & $.7256$ & $.1827$ & $.9031$ & $.8906$ & $.8443$ & $.9011$ \\
Raw 909 & Raw & $.0746$ & $.8018$ & $.6919$ & $.6344$ & $.7325$ & $.2319$ & $.9243$ & $.8806$ & $.8292$ & $.9024$ \\
Raw 910 & Raw & $.0523$ & $.7903$ & $.7040$ & $.6361$ & $.7124$ & $.0847$ & $.8988$ & $.8976$ & $.8461$ & $.8858$ \\
Raw 911 & Raw & $.0140$ & $.8179$ & $.7161$ & $.6151$ & $.7359$ & $.1106$ & $.9177$ & $.8866$ & $.7883$ & $.8821$ \\
SCCR 907 & Raw & $.1274$ & $.8570$ & $.7336$ & $.5273$ & $.7403$ & $.2952$ & $.9360$ & $.8692$ & $.7030$ & $.8683$ \\
SCCR 907 & Own expert & $.4026$ & $.5803$ & $-.0006$ & $-.0010$ & $-.0003$ & $.4316$ & $.5932$ & $-.0072$ & $-.0008$ & $-.0003$ \\
SCCR 907 & Residual & $.1258$ & $.2149$ & $.7329$ & $.5271$ & $.7390$ & $.2468$ & $.9070$ & $.8505$ & $.6826$ & $.8622$ \\
SCCR 908 & Raw & $.1303$ & $.8964$ & $.6751$ & $.5940$ & $.6910$ & $.2863$ & $.9691$ & $.8849$ & $.8146$ & $.8937$ \\
SCCR 908 & Own expert & $.3168$ & $.8746$ & $-.0030$ & $-.0012$ & $-.0002$ & $.4999$ & $.8836$ & $-.0002$ & $-.0018$ & $-.0003$ \\
SCCR 908 & Residual & $.1306$ & $.0088$ & $.6713$ & $.5903$ & $.6844$ & $.1573$ & $.7717$ & $.8715$ & $.7880$ & $.8959$ \\
SCCR 909 & Raw & $.1763$ & $.8784$ & $.6532$ & $.7186$ & $.6594$ & $.3036$ & $.9612$ & $.8335$ & $.8596$ & $.8550$ \\
SCCR 909 & Own expert & $.4143$ & $.8674$ & $-.0007$ & $.0228$ & $-.0003$ & $.5597$ & $.8791$ & $-.0006$ & $-.0011$ & $-.0003$ \\
SCCR 909 & Residual & $.1702$ & $.0105$ & $.6505$ & $.7125$ & $.6542$ & $.2166$ & $.7978$ & $.8057$ & $.8283$ & $.8519$ \\
SCCR 910 & Raw & $.0638$ & $.9008$ & $.7490$ & $.6825$ & $.7208$ & $.1343$ & $.9654$ & $.9119$ & $.8671$ & $.8818$ \\
SCCR 910 & Own expert & $.4563$ & $.8708$ & $-.0006$ & $-.0066$ & $-.0002$ & $.6276$ & $.8815$ & $-.0005$ & $-.0026$ & $-.0003$ \\
SCCR 910 & Residual & $.0511$ & $.0052$ & $.7457$ & $.6805$ & $.7206$ & $.0476$ & $.7738$ & $.8799$ & $.8296$ & $.8847$ \\
SCCR 911 & Raw & $.1100$ & $.9156$ & $.7392$ & $.7144$ & $.6992$ & $.2571$ & $.9685$ & $.8929$ & $.8711$ & $.8831$ \\
SCCR 911 & Own expert & $.4035$ & $.8535$ & $.0032$ & $-.0025$ & $-.0003$ & $.5833$ & $.8705$ & $.0021$ & $-.0023$ & $-.0003$ \\
SCCR 911 & Residual & $.0953$ & $.0163$ & $.7393$ & $.7111$ & $.6965$ & $.1402$ & $.8409$ & $.8742$ & $.8486$ & $.8775$ \\
OLS 907 & Raw & $.0547$ & $.7876$ & $.6767$ & $.5884$ & $.7644$ & $.1626$ & $.9058$ & $.8629$ & $.7874$ & $.9095$ \\
OLS 907 & Own expert & $.5685$ & $-.0014$ & $.0309$ & $-.0162$ & $-.0004$ & $.6585$ & $.0010$ & $.0183$ & $.0016$ & $-.0021$ \\
OLS 907 & Residual & $.0134$ & $.7666$ & $.6274$ & $.5731$ & $.7566$ & $.1246$ & $.8704$ & $.8427$ & $.7551$ & $.8941$ \\
OLS 908 & Raw & $.0536$ & $.8459$ & $.6491$ & $.6742$ & $.7515$ & $.1528$ & $.9122$ & $.8575$ & $.8240$ & $.8891$ \\
OLS 908 & Own expert & $.5534$ & $.0003$ & $.0103$ & $.0007$ & $-.0002$ & $.6165$ & $.0006$ & $.0051$ & $-.0000$ & $-.0003$ \\
OLS 908 & Residual & $.0088$ & $.8192$ & $.6254$ & $.6533$ & $.7445$ & $.1033$ & $.8771$ & $.8299$ & $.7964$ & $.8760$ \\
OLS 909 & Raw & $.0231$ & $.8095$ & $.6383$ & $.6493$ & $.7593$ & $.1339$ & $.9106$ & $.8652$ & $.8138$ & $.8951$ \\
OLS 909 & Own expert & $.4679$ & $.0042$ & $.0196$ & $-.0000$ & $-.0003$ & $.5517$ & $.0030$ & $.0066$ & $-.0000$ & $-.0003$ \\
OLS 909 & Residual & $-.0017$ & $.7824$ & $.6006$ & $.6386$ & $.7464$ & $.0880$ & $.8789$ & $.8380$ & $.7878$ & $.8794$ \\
OLS 910 & Raw & $.0257$ & $.8322$ & $.6996$ & $.5954$ & $.7143$ & $.1058$ & $.9228$ & $.8874$ & $.7504$ & $.8707$ \\
OLS 910 & Own expert & $.4019$ & $.0205$ & $.0244$ & $.0054$ & $-.0003$ & $.3705$ & $.0008$ & $.0053$ & $-.0000$ & $-.0003$ \\
OLS 910 & Residual & $.0061$ & $.7917$ & $.6628$ & $.5733$ & $.7051$ & $.0687$ & $.8759$ & $.8555$ & $.7125$ & $.8543$ \\
OLS 911 & Raw & $.0571$ & $.8447$ & $.6481$ & $.5386$ & $.7241$ & $.0994$ & $.9354$ & $.8576$ & $.7490$ & $.8701$ \\
OLS 911 & Own expert & $.4624$ & $.6947$ & $.0039$ & $.0024$ & $-.0003$ & $.5594$ & $.7312$ & $.0005$ & $.0002$ & $-.0003$ \\
OLS 911 & Residual & $.0071$ & $.1841$ & $.6247$ & $.5270$ & $.7187$ & $.0400$ & $.8616$ & $.8192$ & $.6989$ & $.8494$ \\
\end{longtable}
\endgroup

\begingroup
\small
\setlength{\tabcolsep}{2pt}
\begin{longtable}{@{\extracolsep{\fill}}llcccccccccc@{}}
\caption{Per-seed factor recovery: regularized erasure and shared/private models. Entries are test-set $R^2$; linear and quadratic readouts for each target were selected separately. Raw is before removal; own expert is learned by the same arm. Full concatenates raw-view shared and private codes. Duplicate raw/residual controls are shown once; shared/private own-expert codes were not exported. Dual, Raw, VCCA, and DSN denote Dual baseline, Raw only, VCCA-private, and DSN-MMD.}\label{tab:synimg_native_leace_style_seeds}
\label{tab:synimg_native_vcca_private_seeds}
\label{tab:synimg_native_dsn_mmd_seeds} \\
\toprule
\multicolumn{2}{l}{} & \multicolumn{5}{c}{Linear probes} & \multicolumn{5}{c}{Quadratic probes} \\
\cmidrule(lr){3-7}\cmidrule(l){8-12}
Method/seed & Code & $R_S^2$ & $R_A^2$ & $R_{F_1}^2$ & $R_{F_2}^2$ & $R_{F_1F_2}^2$ & $R_S^2$ & $R_A^2$ & $R_{F_1}^2$ & $R_{F_2}^2$ & $R_{F_1F_2}^2$ \\
\midrule
\endfirsthead
\caption[]{Per-seed factor recovery: regularized erasure and shared/private models (continued).} \\
\toprule
\multicolumn{2}{l}{} & \multicolumn{5}{c}{Linear probes} & \multicolumn{5}{c}{Quadratic probes} \\
\cmidrule(lr){3-7}\cmidrule(l){8-12}
Method/seed & Code & $R_S^2$ & $R_A^2$ & $R_{F_1}^2$ & $R_{F_2}^2$ & $R_{F_1F_2}^2$ & $R_S^2$ & $R_A^2$ & $R_{F_1}^2$ & $R_{F_2}^2$ & $R_{F_1F_2}^2$ \\
\midrule
\endhead
\midrule
\multicolumn{12}{r}{Continued on next page} \\
\endfoot
\bottomrule
\endlastfoot
R-LEACE 907 & Raw & $.4642$ & $.0792$ & $.1990$ & $.0546$ & $-.0002$ & $.5119$ & $.1248$ & $.1876$ & $.0368$ & $-.0003$ \\
R-LEACE 907 & Own expert & $.1234$ & $.8686$ & $-.0004$ & $-.0020$ & $-.0003$ & $.2229$ & $.8816$ & $-.0006$ & $-.0014$ & $-.0003$ \\
R-LEACE 907 & Residual & $-.0013$ & $-.0002$ & $-.0006$ & $-.0000$ & $-.0003$ & $-.0013$ & $-.0002$ & $-.0006$ & $-.0000$ & $-.0003$ \\
R-LEACE 908 & Raw & $.5817$ & $.2656$ & $.3546$ & $.0887$ & $-.0002$ & $.6860$ & $.2609$ & $.5373$ & $.0596$ & $-.0003$ \\
R-LEACE 908 & Own expert & $.1935$ & $.8716$ & $-.0006$ & $-.0026$ & $-.0069$ & $.3155$ & $.8827$ & $-.0006$ & $-.0017$ & $-.0003$ \\
R-LEACE 908 & Residual & $-.0006$ & $-.0002$ & $.0067$ & $.0002$ & $-.0001$ & $.0033$ & $-.0002$ & $.0062$ & $-.0007$ & $-.0005$ \\
R-LEACE 909 & Raw & $.5657$ & $.0769$ & $.0945$ & $.0009$ & $-.0029$ & $.6372$ & $.1052$ & $.0919$ & $.0044$ & $-.0004$ \\
R-LEACE 909 & Own expert & $.1073$ & $.8764$ & $.0058$ & $-.0040$ & $-.0003$ & $.2109$ & $.8819$ & $-.0004$ & $-.0017$ & $-.0003$ \\
R-LEACE 909 & Residual & $-.0013$ & $-.0002$ & $-.0006$ & $-.0000$ & $-.0003$ & $-.0013$ & $-.0002$ & $-.0006$ & $-.0000$ & $-.0003$ \\
R-LEACE 910 & Raw & $.2436$ & $.1201$ & $.0851$ & $.0492$ & $-.0003$ & $.2170$ & $.0997$ & $.0541$ & $.0036$ & $-.0001$ \\
R-LEACE 910 & Own expert & $.2167$ & $.8541$ & $-.0006$ & $-.0029$ & $-.0003$ & $.2887$ & $.8634$ & $-.0006$ & $-.0020$ & $-.0003$ \\
R-LEACE 910 & Residual & $-.0013$ & $-.0002$ & $-.0006$ & $-.0000$ & $-.0003$ & $-.0013$ & $-.0002$ & $-.0006$ & $-.0000$ & $-.0003$ \\
R-LEACE 911 & Raw & $.4420$ & $.1282$ & $.2744$ & $.1756$ & $-.0110$ & $.5292$ & $.1797$ & $.2979$ & $.2116$ & $-.0003$ \\
R-LEACE 911 & Own expert & $.1612$ & $.8764$ & $-.0016$ & $-.0037$ & $-.0003$ & $.2717$ & $.8836$ & $-.0006$ & $-.0016$ & $-.0003$ \\
R-LEACE 911 & Residual & $.0024$ & $.0009$ & $-.0006$ & $.0170$ & $-.0003$ & $.0036$ & $.0000$ & $-.0006$ & $.0195$ & $-.0003$ \\
VCCA 907 & Full & $.9973$ & $.9142$ & $.2613$ & $.0897$ & $-.0060$ & $.9983$ & $.9702$ & $.3716$ & $.1457$ & $-.0003$ \\
VCCA 907 & Shared & $.9834$ & $.6489$ & $.0822$ & $.0008$ & $-.0006$ & $.9950$ & $.8701$ & $.1713$ & $.0113$ & $-.0101$ \\
VCCA 907 & Private & $.9791$ & $.8596$ & $.1834$ & $.0772$ & $-.0003$ & $.9874$ & $.8998$ & $.2540$ & $.1211$ & $-.0003$ \\
VCCA 908 & Full & $.9979$ & $.9943$ & $.0724$ & $.1305$ & $-.0008$ & $.9986$ & $.9959$ & $.1675$ & $.2874$ & $-.0003$ \\
VCCA 908 & Shared & $.9976$ & $.0621$ & $.0365$ & $.0161$ & $-.0018$ & $.9984$ & $.1866$ & $.1069$ & $.0474$ & $-.0018$ \\
VCCA 908 & Private & $.5910$ & $.9926$ & $.0339$ & $.1069$ & $-.0005$ & $.7687$ & $.9948$ & $.0567$ & $.2097$ & $-.0014$ \\
VCCA 909 & Full & $.9966$ & $.9957$ & $.0990$ & $.0312$ & $-.0004$ & $.9982$ & $.9969$ & $.1755$ & $.1132$ & $-.0003$ \\
VCCA 909 & Shared & $.9850$ & $.9694$ & $.0228$ & $.0003$ & $-.0003$ & $.9940$ & $.9855$ & $.0422$ & $-.0005$ & $-.0003$ \\
VCCA 909 & Private & $.9762$ & $.9856$ & $.0702$ & $.0271$ & $-.0034$ & $.9897$ & $.9929$ & $.1151$ & $.0713$ & $-.0003$ \\
VCCA 910 & Full & $.9975$ & $.9960$ & $.0759$ & $.0399$ & $-.0008$ & $.9985$ & $.9984$ & $.2202$ & $.1215$ & $-.0003$ \\
VCCA 910 & Shared & $.9870$ & $.8994$ & $.0286$ & $.0244$ & $-.0002$ & $.9970$ & $.9828$ & $.0589$ & $.0683$ & $-.0003$ \\
VCCA 910 & Private & $.8542$ & $.9935$ & $.0438$ & $.0095$ & $-.0007$ & $.9273$ & $.9959$ & $.0971$ & $.0124$ & $-.0004$ \\
VCCA 911 & Full & $.9986$ & $.4401$ & $.2767$ & $.0863$ & $-.0042$ & $.9985$ & $.6518$ & $.3317$ & $.1270$ & $-.0128$ \\
VCCA 911 & Shared & $.9975$ & $.0498$ & $.0815$ & $.0565$ & $-.0066$ & $.9982$ & $.1362$ & $.1204$ & $.1245$ & $-.0221$ \\
VCCA 911 & Private & $.9937$ & $.3753$ & $.2182$ & $.0234$ & $-.0022$ & $.9953$ & $.4945$ & $.2995$ & $-.1311$ & $-.0037$ \\
DSN 907 & Full & $.8061$ & $.9388$ & $.7051$ & $.6670$ & $.6622$ & $.8805$ & $.9478$ & $.7682$ & $.7167$ & $.7280$ \\
DSN 907 & Shared & $.0108$ & $.9390$ & $.6125$ & $.5809$ & $.6680$ & $.0507$ & $.9572$ & $.7535$ & $.7004$ & $.7908$ \\
DSN 907 & Private & $.8011$ & $.0636$ & $.4200$ & $.4155$ & $.0033$ & $.8761$ & $.0920$ & $.5060$ & $.5274$ & $.0339$ \\
DSN 908 & Full & $.3449$ & $.9406$ & $.6789$ & $.6571$ & $.7731$ & $.3932$ & $.9319$ & $.6013$ & $.4816$ & $.6635$ \\
DSN 908 & Shared & $.0354$ & $.9411$ & $.6345$ & $.6548$ & $.7735$ & $.1129$ & $.9594$ & $.7380$ & $.7578$ & $.8115$ \\
DSN 908 & Private & $.3076$ & $.0187$ & $.3369$ & $.0565$ & $.0031$ & $.3878$ & $.0402$ & $.4187$ & $.0382$ & $.0070$ \\
DSN 909 & Full & $.3592$ & $.9321$ & $.6508$ & $.6504$ & $.7698$ & $.4329$ & $.9640$ & $.6868$ & $.6384$ & $.8484$ \\
DSN 909 & Shared & $.0082$ & $.9313$ & $.6386$ & $.5725$ & $.7690$ & $.0584$ & $.9618$ & $.7533$ & $.6824$ & $.8752$ \\
DSN 909 & Private & $.3533$ & $.2255$ & $.1106$ & $.2366$ & $.0000$ & $.4391$ & $.3160$ & $.1709$ & $.3287$ & $.0029$ \\
DSN 910 & Full & $.3742$ & $.9456$ & $.7656$ & $.6098$ & $.7632$ & $.5685$ & $.9387$ & $.7687$ & $.5286$ & $.7174$ \\
DSN 910 & Shared & $.0038$ & $.9462$ & $.6406$ & $.5820$ & $.7571$ & $.0788$ & $.9664$ & $.7761$ & $.6883$ & $.8448$ \\
DSN 910 & Private & $.3673$ & $.0139$ & $.6077$ & $.0782$ & $.0000$ & $.5727$ & $.0145$ & $.6950$ & $.1307$ & $-.0003$ \\
DSN 911 & Full & $.1918$ & $.8190$ & $.6131$ & $.5608$ & $.6950$ & $.2233$ & $.7437$ & $.5130$ & $.1683$ & $.5880$ \\
DSN 911 & Shared & $.0157$ & $.8196$ & $.5989$ & $.5656$ & $.6975$ & $.0660$ & $.8744$ & $.7031$ & $.6540$ & $.7895$ \\
DSN 911 & Private & $.1804$ & $.0166$ & $.0998$ & $-.0000$ & $-.0003$ & $.2085$ & $.0147$ & $.1490$ & $.0011$ & $-.0003$ \\
\end{longtable}
\endgroup

\paragraph{Native task heads.}
These metrics evaluate the supervised classifier trained with each network.
They differ from the post-hoc task readouts below. The dual baseline,
SCCR, OLS, and R-LEACE heads combine their own expert and retained codes;
the raw-only head uses raw, and DSN-MMD uses its raw-view shared code.
VCCA-private has no native supervised task head, so it has no entries in this table.

\begin{table}[H]
\centering
\caption{Native task-head prediction for each seed and its five-seed mean $\pm$ sample SD. Accuracy uses probability threshold $.5$; AUROC uses continuous scores; log loss is cross-entropy. All scores use the same fixed assessment cohort. These are the trained heads, not post-hoc readouts.}
\label{tab:synimg_native_head_summary}
\label{tab:synimg_native_head_seeds_1}
\label{tab:synimg_native_head_seeds_2}
\setlength{\tabcolsep}{2pt}
\begin{tabular*}{\textwidth}{@{\extracolsep{\fill}}llcccccc@{}}
\toprule
Method & Metric & 907 & 908 & 909 & 910 & 911 & Mean $\pm$ SD \\
\midrule
Dual baseline & Accuracy & $.9839$ & $.9822$ & $.9839$ & $.9846$ & $.9841$ & $.984\!\pm\!.001$ \\
 & AUROC & $.9992$ & $.9992$ & $.9991$ & $.9992$ & $.9992$ & $.9992\!\pm\!.0000$ \\
 & Log loss & $.0365$ & $.0380$ & $.0391$ & $.0361$ & $.0375$ & $.037\!\pm\!.001$ \\
Raw only & Accuracy & $.9873$ & $.9854$ & $.9846$ & $.9832$ & $.9841$ & $.985\!\pm\!.002$ \\
 & AUROC & $.9994$ & $.9994$ & $.9993$ & $.9989$ & $.9993$ & $.9993\!\pm\!.0002$ \\
 & Log loss & $.0325$ & $.0327$ & $.0335$ & $.0424$ & $.0357$ & $.035\!\pm\!.004$ \\
SCCR & Accuracy & $.9543$ & $.9719$ & $.9487$ & $.9617$ & $.9580$ & $.959\!\pm\!.009$ \\
 & AUROC & $.9920$ & $.9968$ & $.9899$ & $.9927$ & $.9927$ & $.9928\!\pm\!.0025$ \\
 & Log loss & $.1173$ & $.0739$ & $.1298$ & $.1041$ & $.1100$ & $.107\!\pm\!.021$ \\
OLS & Accuracy & $.9805$ & $.9805$ & $.9839$ & $.9829$ & $.9800$ & $.982\!\pm\!.002$ \\
 & AUROC & $.9988$ & $.9987$ & $.9991$ & $.9993$ & $.9988$ & $.9989\!\pm\!.0003$ \\
 & Log loss & $.0451$ & $.0490$ & $.0400$ & $.0369$ & $.0459$ & $.043\!\pm\!.005$ \\
R-LEACE & Accuracy & $.7708$ & $.7744$ & $.7747$ & $.7703$ & $.7756$ & $.773\!\pm\!.002$ \\
 & AUROC & $.8572$ & $.8578$ & $.8582$ & $.8558$ & $.8588$ & $.8575\!\pm\!.0011$ \\
 & Log loss & $.4722$ & $.4718$ & $.4706$ & $.4740$ & $.4708$ & $.472\!\pm\!.001$ \\
DSN-MMD & Accuracy & $.9675$ & $.9465$ & $.9480$ & $.9473$ & $.9587$ & $.954\!\pm\!.009$ \\
 & AUROC & $.9971$ & $.9920$ & $.9939$ & $.9938$ & $.9959$ & $.9946\!\pm\!.0020$ \\
 & Log loss & $.0795$ & $.1425$ & $.1359$ & $.1374$ & $.0969$ & $.118\!\pm\!.028$ \\
\bottomrule
\end{tabular*}
\end{table}

\paragraph{Post-hoc task readouts.}
The following tables fit linear or complete quadratic classifiers to frozen
learned codes and evaluate the binary task. Here $E$ denotes the arm's own
expert code, $E_0$ the same-seed dual baseline's expert code, and $R$ the
arm's retained output (private for VCCA-private/DSN-MMD; raw for the raw-only
arm). Thus $(E,R)$ and $(E_0,R)$ are distinct comparisons.
For the dual baseline, $E=E_0$, so that combined input is shown once.
The VCCA-private and DSN-MMD combined readouts use $E_0$; no own-expert
readout is available for those models. These classifiers and their
regularization were selected on development data, separately from the
native heads. Higher task AUROC is not by itself the removal objective,
because the label also depends on $A$.

\begingroup
\small
\setlength{\tabcolsep}{2pt}
\begin{longtable}{@{\extracolsep{\fill}}llcccccc@{}}
\caption{Post-hoc task readouts on learned codes: mean $\pm$ sample SD over seeds 907--911. $E$ is the own expert and $E_0$ the same-seed dual baseline expert, so $(E,R)$ and $(E_0,R)$ remain distinct. Accuracy uses threshold $.5$. Each representation and readout family is selected separately on development data; these are not native heads.}\label{tab:synimg_task_readouts_linear}
\label{tab:synimg_task_readouts_quadratic} \\
\toprule
\multicolumn{2}{l}{} & \multicolumn{3}{c}{Linear readouts} & \multicolumn{3}{c}{Quadratic readouts} \\
\cmidrule(lr){3-5}\cmidrule(l){6-8}
Method & Code & Accuracy & AUROC & Log loss & Accuracy & AUROC & Log loss \\
\midrule
\endfirsthead
\caption[]{Post-hoc task readouts on learned codes: mean $\pm$ sample SD over seeds 907--911 (continued).} \\
\toprule
\multicolumn{2}{l}{} & \multicolumn{3}{c}{Linear readouts} & \multicolumn{3}{c}{Quadratic readouts} \\
\cmidrule(lr){3-5}\cmidrule(l){6-8}
Method & Code & Accuracy & AUROC & Log loss & Accuracy & AUROC & Log loss \\
\midrule
\endhead
\midrule
\multicolumn{8}{r}{Continued on next page} \\
\endfoot
\bottomrule
\endlastfoot
Dual baseline & Raw & $.984\!\pm\!.001$ & $.9991\!\pm\!.0002$ & $.041\!\pm\!.003$ & $.983\!\pm\!.002$ & $.9991\!\pm\!.0001$ & $.040\!\pm\!.003$ \\
Dual baseline & $E$ & $.747\!\pm\!.013$ & $.8382\!\pm\!.0087$ & $.494\!\pm\!.010$ & $.744\!\pm\!.008$ & $.8366\!\pm\!.0084$ & $.499\!\pm\!.012$ \\
Dual baseline & $(E,R)=(E_0,R)$ & $.984\!\pm\!.001$ & $.9991\!\pm\!.0001$ & $.040\!\pm\!.002$ & $.981\!\pm\!.001$ & $.9985\!\pm\!.0002$ & $.051\!\pm\!.001$ \\
Raw only & Raw & $.986\!\pm\!.001$ & $.9994\!\pm\!.0001$ & $.032\!\pm\!.003$ & $.986\!\pm\!.001$ & $.9994\!\pm\!.0001$ & $.033\!\pm\!.002$ \\
Raw only & $(E_0,R)$ & $.985\!\pm\!.003$ & $.9993\!\pm\!.0002$ & $.036\!\pm\!.005$ & $.983\!\pm\!.002$ & $.9989\!\pm\!.0002$ & $.043\!\pm\!.004$ \\
SCCR & Raw & $.966\!\pm\!.007$ & $.9953\!\pm\!.0023$ & $.087\!\pm\!.018$ & $.967\!\pm\!.005$ & $.9960\!\pm\!.0014$ & $.082\!\pm\!.013$ \\
SCCR & Residual $R$ & $.903\!\pm\!.030$ & $.9484\!\pm\!.0172$ & $.317\!\pm\!.067$ & $.964\!\pm\!.005$ & $.9907\!\pm\!.0025$ & $.113\!\pm\!.017$ \\
SCCR & $E$ & $.762\!\pm\!.026$ & $.8374\!\pm\!.0388$ & $.492\!\pm\!.037$ & $.759\!\pm\!.030$ & $.8369\!\pm\!.0385$ & $.495\!\pm\!.038$ \\
SCCR & $(E,R)$ & $.961\!\pm\!.007$ & $.9901\!\pm\!.0018$ & $.117\!\pm\!.013$ & $.964\!\pm\!.007$ & $.9922\!\pm\!.0026$ & $.114\!\pm\!.019$ \\
SCCR & $(E_0,R)$ & $.959\!\pm\!.005$ & $.9882\!\pm\!.0021$ & $.129\!\pm\!.014$ & $.961\!\pm\!.007$ & $.9913\!\pm\!.0017$ & $.125\!\pm\!.014$ \\
OLS & Raw & $.984\!\pm\!.003$ & $.9991\!\pm\!.0002$ & $.039\!\pm\!.004$ & $.985\!\pm\!.003$ & $.9990\!\pm\!.0002$ & $.040\!\pm\!.004$ \\
OLS & Residual $R$ & $.978\!\pm\!.011$ & $.9954\!\pm\!.0077$ & $.072\!\pm\!.059$ & $.984\!\pm\!.004$ & $.9988\!\pm\!.0007$ & $.043\!\pm\!.009$ \\
OLS & $E$ & $.568\!\pm\!.108$ & $.5837\!\pm\!.1384$ & $.656\!\pm\!.081$ & $.562\!\pm\!.111$ & $.5737\!\pm\!.1429$ & $.657\!\pm\!.081$ \\
OLS & $(E,R)$ & $.982\!\pm\!.002$ & $.9986\!\pm\!.0003$ & $.048\!\pm\!.005$ & $.980\!\pm\!.002$ & $.9961\!\pm\!.0011$ & $.068\!\pm\!.008$ \\
OLS & $(E_0,R)$ & $.982\!\pm\!.002$ & $.9984\!\pm\!.0006$ & $.049\!\pm\!.007$ & $.981\!\pm\!.003$ & $.9973\!\pm\!.0006$ & $.062\!\pm\!.007$ \\
R-LEACE & Raw & $.587\!\pm\!.029$ & $.6222\!\pm\!.0451$ & $.667\!\pm\!.021$ & $.575\!\pm\!.027$ & $.6109\!\pm\!.0370$ & $.670\!\pm\!.015$ \\
R-LEACE & Residual $R$ & $.510\!\pm\!.007$ & $.5011\!\pm\!.0026$ & $.693\!\pm\!.000$ & $.510\!\pm\!.005$ & $.5004\!\pm\!.0040$ & $.693\!\pm\!.000$ \\
R-LEACE & $E$ & $.774\!\pm\!.002$ & $.8559\!\pm\!.0019$ & $.473\!\pm\!.003$ & $.772\!\pm\!.001$ & $.8543\!\pm\!.0013$ & $.477\!\pm\!.002$ \\
R-LEACE & $(E,R)$ & $.774\!\pm\!.002$ & $.8559\!\pm\!.0018$ & $.473\!\pm\!.003$ & $.772\!\pm\!.001$ & $.8547\!\pm\!.0016$ & $.476\!\pm\!.002$ \\
R-LEACE & $(E_0,R)$ & $.747\!\pm\!.013$ & $.8382\!\pm\!.0086$ & $.494\!\pm\!.010$ & $.742\!\pm\!.011$ & $.8351\!\pm\!.0103$ & $.499\!\pm\!.012$ \\
VCCA-private & Full & $.771\!\pm\!.069$ & $.8425\!\pm\!.0764$ & $.481\!\pm\!.082$ & $.767\!\pm\!.071$ & $.8443\!\pm\!.0755$ & $.480\!\pm\!.081$ \\
VCCA-private & Shared & $.678\!\pm\!.123$ & $.7341\!\pm\!.1512$ & $.578\!\pm\!.112$ & $.680\!\pm\!.121$ & $.7319\!\pm\!.1552$ & $.587\!\pm\!.101$ \\
VCCA-private & Private $R$ & $.764\!\pm\!.073$ & $.8347\!\pm\!.0845$ & $.490\!\pm\!.087$ & $.755\!\pm\!.085$ & $.8309\!\pm\!.1029$ & $.492\!\pm\!.116$ \\
VCCA-private & $(E_0,R)$ & $.785\!\pm\!.018$ & $.8649\!\pm\!.0136$ & $.461\!\pm\!.019$ & $.778\!\pm\!.023$ & $.8589\!\pm\!.0159$ & $.471\!\pm\!.022$ \\
DSN-MMD & Full & $.973\!\pm\!.005$ & $.9973\!\pm\!.0008$ & $.066\!\pm\!.010$ & $.954\!\pm\!.009$ & $.9935\!\pm\!.0028$ & $.106\!\pm\!.023$ \\
DSN-MMD & Shared & $.978\!\pm\!.005$ & $.9980\!\pm\!.0006$ & $.059\!\pm\!.009$ & $.981\!\pm\!.003$ & $.9985\!\pm\!.0004$ & $.050\!\pm\!.006$ \\
DSN-MMD & Private $R$ & $.529\!\pm\!.026$ & $.5404\!\pm\!.0433$ & $.687\!\pm\!.011$ & $.532\!\pm\!.018$ & $.5523\!\pm\!.0341$ & $.684\!\pm\!.013$ \\
DSN-MMD & $(E_0,R)$ & $.744\!\pm\!.011$ & $.8367\!\pm\!.0084$ & $.495\!\pm\!.010$ & $.738\!\pm\!.012$ & $.8394\!\pm\!.0108$ & $.482\!\pm\!.009$ \\
\end{longtable}
\endgroup

\subsection{Common-embedding results}
\label{sec:synimg_common_results}
All operators below use the frozen raw and expert embeddings of the same Dual baseline within each seed. Five-seed estimates use exactly seeds 907--911 and the same 4,096 assessment pairs; they are a retrospective subset. Seeds 907 and 908 reused a previously inspected assessment cohort. Means and sample standard deviations therefore summarize initialization variation conditional on shared data, not uncertainty over independently sampled datasets. The later JIVE budget follow-up is reported separately on its two converged seeds.
AJIVE reports the paired individual output for the gap-selected and $(4,4)$ initial ranks. For JIVE controls, P denotes paired inference, R raw-only inference, I the denoised individual reconstruction, and Rem the centered raw-minus-joint remainder. A paired extension already uses both views before its residual-only readout. D-CCA also requires both views. Combined readouts concatenate the common Dual baseline expert code with the listed output. Factor scores below are retained-output $R^2$; negative values are retained and only rounded for display.

Identical controls are printed once: common-embedding SCCR and matched top-$q$ have identical results; for JIVE with $q=0$, paired and raw-only individual outputs coincide at each $k$, and every remainder equals Raw. These equalities apply to both probe families and all reported prediction metrics. Zero and Exact LEACE remain separate because their factor scores differ.

\begin{table}[H]
\centering
\caption{Factor recovery on common dual-baseline embeddings: mean $\pm$ sample SD over seeds 907--911. AJIVE uses paired individual outputs; Gap is the primary singular-value-ratio initial-rank rule. JIVE $q=0$ tests no joint removal. D-CCA is the empirical covariance plug-in without S-POET.}
\label{tab:synimg_common_operators_linear_factors}
\label{tab:synimg_common_ajive_linear_factors}
\label{tab:synimg_common_jive_q0_linear_factors}
\label{tab:synimg_common_operators_quadratic_factors}
\label{tab:synimg_common_ajive_quadratic_factors}
\label{tab:synimg_common_jive_q0_quadratic_factors}
\setlength{\tabcolsep}{2pt}
\begin{tabular*}{\textwidth}{@{\extracolsep{\fill}}lccccc@{}}
\toprule
Output & $R_S^2$ & $R_A^2$ & $R_{F_1}^2$ & $R_{F_2}^2$ & $R_{F_1F_2}^2$ \\
\midrule
\midrule
\multicolumn{6}{l}{\textit{Linear probes}} \\*
Raw & $.053\!\pm\!.023$ & $.826\!\pm\!.021$ & $.693\!\pm\!.028$ & $.634\!\pm\!.032$ & $.726\!\pm\!.022$ \\
Zero & $-.001\!\pm\!.000$ & $-.000\!\pm\!.000$ & $-.001\!\pm\!.000$ & $-.000\!\pm\!.000$ & $-.000\!\pm\!.000$ \\
SCCR & $.055\!\pm\!.020$ & $.010\!\pm\!.004$ & $.691\!\pm\!.027$ & $.633\!\pm\!.033$ & $.724\!\pm\!.022$ \\
Exact LEACE & $-.001\!\pm\!.000$ & $-.000\!\pm\!.000$ & $.001\!\pm\!.003$ & $.000\!\pm\!.000$ & $.001\!\pm\!.002$ \\
D-CCA & $.043\!\pm\!.019$ & $.451\!\pm\!.039$ & $.690\!\pm\!.029$ & $.632\!\pm\!.030$ & $.723\!\pm\!.023$ \\
AJIVE Gap individual & $-.002\!\pm\!.000$ & $.445\!\pm\!.014$ & $-.000\!\pm\!.001$ & $.002\!\pm\!.004$ & $.161\!\pm\!.006$ \\
AJIVE $(4,4)$ individual & $-.002\!\pm\!.000$ & $.013\!\pm\!.014$ & $.010\!\pm\!.013$ & $.010\!\pm\!.011$ & $.369\!\pm\!.023$ \\
JIVE $k=4,q=0$ individual & $-.002\!\pm\!.000$ & $.579\!\pm\!.025$ & $.009\!\pm\!.011$ & $.009\!\pm\!.009$ & $.381\!\pm\!.041$ \\
JIVE $k=8,q=0$ individual & $-.001\!\pm\!.000$ & $.697\!\pm\!.025$ & $.333\!\pm\!.151$ & $.245\!\pm\!.086$ & $.488\!\pm\!.046$ \\
JIVE $k=16,q=0$ individual & $.004\!\pm\!.004$ & $.741\!\pm\!.035$ & $.500\!\pm\!.044$ & $.431\!\pm\!.068$ & $.598\!\pm\!.029$ \\
\midrule
\multicolumn{6}{l}{\textit{Quadratic probes}} \\*
Raw & $.111\!\pm\!.057$ & $.919\!\pm\!.012$ & $.887\!\pm\!.022$ & $.812\!\pm\!.040$ & $.898\!\pm\!.013$ \\
Zero & $-.001\!\pm\!.000$ & $-.000\!\pm\!.000$ & $-.001\!\pm\!.000$ & $-.000\!\pm\!.000$ & $-.000\!\pm\!.000$ \\
SCCR & $.056\!\pm\!.028$ & $.662\!\pm\!.067$ & $.836\!\pm\!.034$ & $.757\!\pm\!.040$ & $.882\!\pm\!.015$ \\
Exact LEACE & $-.001\!\pm\!.000$ & $-.000\!\pm\!.000$ & $.001\!\pm\!.003$ & $.000\!\pm\!.000$ & $.001\!\pm\!.003$ \\
D-CCA & $.073\!\pm\!.030$ & $.798\!\pm\!.021$ & $.860\!\pm\!.027$ & $.780\!\pm\!.043$ & $.887\!\pm\!.013$ \\
AJIVE Gap individual & $-.002\!\pm\!.000$ & $.462\!\pm\!.024$ & $.018\!\pm\!.030$ & $.017\!\pm\!.008$ & $.217\!\pm\!.043$ \\
AJIVE $(4,4)$ individual & $-.002\!\pm\!.001$ & $.174\!\pm\!.115$ & $.088\!\pm\!.065$ & $.048\!\pm\!.034$ & $.469\!\pm\!.038$ \\
JIVE $k=4,q=0$ individual & $-.001\!\pm\!.002$ & $.687\!\pm\!.037$ & $.149\!\pm\!.076$ & $.089\!\pm\!.041$ & $.593\!\pm\!.030$ \\
JIVE $k=8,q=0$ individual & $.005\!\pm\!.002$ & $.784\!\pm\!.013$ & $.562\!\pm\!.064$ & $.498\!\pm\!.043$ & $.687\!\pm\!.033$ \\
JIVE $k=16,q=0$ individual & $.033\!\pm\!.027$ & $.858\!\pm\!.010$ & $.744\!\pm\!.031$ & $.659\!\pm\!.055$ & $.810\!\pm\!.019$ \\
\bottomrule
\end{tabular*}
\end{table}

\begin{table}[H]
\centering
\caption{Task AUROC on common embeddings: mean $\pm$ sample SD over seeds 907--911. Res uses the listed output alone; Comb adds the common expert code. Paired AJIVE outputs and D-CCA already use both views before readout. These are frozen readouts, not native heads.}
\label{tab:synimg_common_operators_auroc}
\label{tab:synimg_common_ajive_auroc}
\label{tab:synimg_common_jive_q0_auroc}
\setlength{\tabcolsep}{2pt}
\begin{tabular*}{\textwidth}{@{\extracolsep{\fill}}lcccc@{}}
\toprule
Output & Linear Res & Linear Comb & Quadratic Res & Quadratic Comb \\
\midrule
Raw & $.9991\!\pm\!.0002$ & $.9991\!\pm\!.0001$ & $.9991\!\pm\!.0001$ & $.9985\!\pm\!.0002$ \\
Zero & $.5000\!\pm\!.0000$ & $.8382\!\pm\!.0087$ & $.5000\!\pm\!.0000$ & $.8366\!\pm\!.0084$ \\
SCCR & $.9132\!\pm\!.0142$ & $.9916\!\pm\!.0018$ & $.9814\!\pm\!.0049$ & $.9952\!\pm\!.0011$ \\
Exact LEACE & $.5000\!\pm\!.0000$ & $.8382\!\pm\!.0087$ & $.5000\!\pm\!.0000$ & $.8366\!\pm\!.0084$ \\
D-CCA & $.9948\!\pm\!.0012$ & $.9988\!\pm\!.0002$ & $.9979\!\pm\!.0004$ & $.9981\!\pm\!.0002$ \\
AJIVE Gap individual & $.9991\!\pm\!.0001$ & $.9990\!\pm\!.0001$ & $.9992\!\pm\!.0001$ & $.9986\!\pm\!.0002$ \\
AJIVE $(4,4)$ individual & $.8597\!\pm\!.0211$ & $.9958\!\pm\!.0030$ & $.9325\!\pm\!.0254$ & $.9971\!\pm\!.0011$ \\
JIVE $k=4,q=0$ individual & $.9991\!\pm\!.0001$ & $.9991\!\pm\!.0001$ & $.9992\!\pm\!.0001$ & $.9986\!\pm\!.0001$ \\
JIVE $k=8,q=0$ individual & $.9991\!\pm\!.0001$ & $.9991\!\pm\!.0001$ & $.9992\!\pm\!.0001$ & $.9986\!\pm\!.0001$ \\
JIVE $k=16,q=0$ individual & $.9991\!\pm\!.0002$ & $.9991\!\pm\!.0001$ & $.9992\!\pm\!.0001$ & $.9985\!\pm\!.0001$ \\
\bottomrule
\end{tabular*}
\end{table}

\begin{table}[H]
\centering
\caption{Accuracy and log loss on common embeddings: mean $\pm$ sample SD over seeds 907--911. Res/Comb denote output-only/expert-plus-output. Accuracy uses threshold $.5$; validation log loss selects each frozen readout.}
\label{tab:synimg_common_operators_linear_prediction}
\label{tab:synimg_common_operators_quadratic_prediction}
\setlength{\tabcolsep}{2pt}
\begin{tabular*}{\textwidth}{@{\extracolsep{\fill}}lcccc@{}}
\toprule
Output & Accuracy Res & Accuracy Comb & Log loss Res & Log loss Comb \\
\midrule
\midrule
\multicolumn{5}{l}{\textit{Linear readouts}} \\*
Raw & $.984\!\pm\!.001$ & $.984\!\pm\!.001$ & $.041\!\pm\!.003$ & $.040\!\pm\!.002$ \\
Zero & $.513\!\pm\!.000$ & $.747\!\pm\!.013$ & $.693\!\pm\!.000$ & $.494\!\pm\!.010$ \\
SCCR & $.824\!\pm\!.023$ & $.968\!\pm\!.004$ & $.380\!\pm\!.022$ & $.105\!\pm\!.012$ \\
Exact LEACE & $.513\!\pm\!.000$ & $.747\!\pm\!.013$ & $.693\!\pm\!.000$ & $.494\!\pm\!.010$ \\
D-CCA & $.982\!\pm\!.003$ & $.983\!\pm\!.001$ & $.072\!\pm\!.009$ & $.045\!\pm\!.002$ \\
\midrule
\multicolumn{5}{l}{\textit{Quadratic readouts}} \\*
Raw & $.983\!\pm\!.002$ & $.981\!\pm\!.001$ & $.040\!\pm\!.003$ & $.051\!\pm\!.001$ \\
Zero & $.513\!\pm\!.000$ & $.744\!\pm\!.008$ & $.693\!\pm\!.000$ & $.499\!\pm\!.012$ \\
SCCR & $.955\!\pm\!.015$ & $.972\!\pm\!.002$ & $.184\!\pm\!.034$ & $.086\!\pm\!.007$ \\
Exact LEACE & $.513\!\pm\!.000$ & $.744\!\pm\!.008$ & $.693\!\pm\!.000$ & $.499\!\pm\!.012$ \\
D-CCA & $.981\!\pm\!.002$ & $.981\!\pm\!.000$ & $.054\!\pm\!.005$ & $.057\!\pm\!.001$ \\
\bottomrule
\end{tabular*}
\end{table}

\begin{table}[H]
\centering
\caption{Common-embedding rank diagnostics. SCCR and matched top-$q$ remove the same $q$; D-CCA uses all supported pairs (relative covariance cutoff $10^{-12}$). JIVE $q=0$ converges in two iterations at each $k=4,8,16$, with individual ranks $(k,k)$; all inference/output variants are eligible. AJIVE uses 1,000 Wedin/random-direction resamples and identifiability checks; final individual ranks use the original singular thresholds without a cap.}
\label{tab:synimg_common_rank_summary}
\label{tab:synimg_common_ajive_ranks_1}
\setlength{\tabcolsep}{2pt}

\par\noindent\textit{(a) Operator ranks}
\begin{tabular*}{\textwidth}{@{\extracolsep{\fill}}lrrrrr@{}}
\toprule
Seed & SCCR $q$ & LEACE removed & D-CCA supports & D-CCA pairs & JIVE $q=0$ iterations \\
\midrule
907 & 2 & 63 & (63,64) & 63 & 2 / 2 / 2 \\
908 & 2 & 64 & (64,64) & 64 & 2 / 2 / 2 \\
909 & 2 & 64 & (64,64) & 64 & 2 / 2 / 2 \\
910 & 3 & 64 & (64,64) & 64 & 2 / 2 / 2 \\
911 & 3 & 64 & (64,64) & 64 & 2 / 2 / 2 \\
\bottomrule
\end{tabular*}

\par\noindent\textit{(b) AJIVE initial and final ranks}
\begin{tabular*}{\textwidth}{@{\extracolsep{\fill}}llrrr@{}}
\toprule
Seed & Setting & Initial (raw, expert) & Joint & Individual (raw, expert) \\
\midrule
907 & gap & (2,2) & 0 & (2,2) \\
907 & (4,4) & (4,4) & 1 & (4,4) \\
908 & gap & (2,2) & 0 & (2,2) \\
908 & (4,4) & (4,4) & 1 & (4,4) \\
909 & gap & (2,2) & 0 & (2,2) \\
909 & (4,4) & (4,4) & 1 & (3,4) \\
910 & gap & (2,2) & 0 & (2,2) \\
910 & (4,4) & (4,4) & 1 & (3,4) \\
911 & gap & (2,2) & 0 & (2,2) \\
911 & (4,4) & (4,4) & 1 & (3,4) \\
\bottomrule
\end{tabular*}

\end{table}

\begin{table}[H]
\caption{Original positive-joint-rank JIVE outcomes: all entries are NC and unscored. Numbers give final relative fitted-matrix change at the 5,000-iteration cap; eligibility requires a value below $10^{-9}$. Each $(k,q)$ has four variants (P/I, P/Rem, R/I, R/Rem): 24 ineligible variants per seed, 120 across five seeds, arising from 30 distinct rank fits. NC is never assigned chance or zero recovery. Individual rank budgets are $(k-q,k-q)$, with block Frobenius scaling and no rank expansion.}
\label{tab:synimg_common_jive_positive_q_nc}
\centering
\setlength{\tabcolsep}{2pt}
\begin{tabular*}{\textwidth}{@{\extracolsep{\fill}}lrrrrr@{}}
\toprule
Rank setting & 907 & 908 & 909 & 910 & 911 \\
\midrule
$k=4,q=1$ & $1.999\times10^{-6}$ & $2.366\times10^{-6}$ & $1.325\times10^{-6}$ & $1.932\times10^{-6}$ & $1.071\times10^{-6}$ \\
$k=4,q=2$ & $4.410\times10^{-5}$ & $1.485\times10^{-6}$ & $5.008\times10^{-8}$ & $1.788\times10^{-5}$ & $4.893\times10^{-8}$ \\
$k=8,q=1$ & $3.054\times10^{-7}$ & $3.164\times10^{-7}$ & $1.142\times10^{-6}$ & $7.696\times10^{-7}$ & $3.311\times10^{-7}$ \\
$k=8,q=2$ & $9.055\times10^{-6}$ & $3.289\times10^{-5}$ & $4.887\times10^{-7}$ & $1.032\times10^{-6}$ & $8.283\times10^{-7}$ \\
$k=16,q=1$ & $4.030\times10^{-7}$ & $1.725\times10^{-6}$ & $6.336\times10^{-7}$ & $9.460\times10^{-7}$ & $2.266\times10^{-6}$ \\
$k=16,q=2$ & $1.584\times10^{-6}$ & $6.083\times10^{-6}$ & $9.851\times10^{-8}$ & $2.062\times10^{-7}$ & $1.007\times10^{-7}$ \\
\bottomrule
\end{tabular*}
\end{table}
\subsection{Exploratory JIVE iteration-budget follow-up}
\label{sec:synimg_common_jive_followup}
This separate follow-up reran only $k=4,q=2$, individual ranks $(2,2)$, on the same calibration embeddings. It preserved initialization, block scaling, tolerance $10^{-9}$ and numerical cutoff $10^{-12}$, without whitening. A 20,000-step failure triggered a restart from the original initialization at a 50,000-step cap; recovery and prediction never selected the budget. All five development outcomes and eligible readouts were locked before this runner loaded assessment arrays. The cohort had already been inspected, so the follow-up remains exploratory. No neural network was trained, and no raw-minus-joint remainder was added in this follow-up.

\begin{table}[H]
\centering
\caption{Exploratory JIVE iteration-budget follow-up. Every 20,000-step attempt was NC; only seeds 908 and 911 passed the unchanged rule at the 50,000 cap. Factor $R^2$ and frozen task AUROC are mean $\pm$ sample SD on those two converged seeds only ($n=2$), not five-seed estimates; seeds 907, 909 and 910 remain unscored. P/I and R/I are two extensions of the same decomposition. Relative fitted-matrix change and sampled component stability do not establish a global optimum. Near-zero weak-factor recovery can coexist with strong combined prediction, so these results do not establish general predictive failure.}
\label{tab:synimg_common_jive_followup_convergence}
\label{tab:synimg_common_jive_followup_factors}
\label{tab:synimg_common_jive_followup_auroc}
\setlength{\tabcolsep}{2pt}

\par\noindent\textit{(a) Convergence by seed}
\begin{tabular*}{\textwidth}{@{\extracolsep{\fill}}lrrrr@{}}
\toprule
Seed & 20k final change & 50k iterations & 50k final change & Eligible \\
\midrule
907 & $1.316\times10^{-7}$ & 50000 & $1.442\times10^{-8}$ & NC; unscored \\
908 & $5.627\times10^{-8}$ & 48350 & $9.999\times10^{-10}$ & Yes \\
909 & $7.541\times10^{-8}$ & 50000 & $3.507\times10^{-6}$ & NC; unscored \\
910 & $6.373\times10^{-8}$ & 50000 & $1.725\times10^{-9}$ & NC; unscored \\
911 & $1.373\times10^{-7}$ & 43623 & $9.988\times10^{-10}$ & Yes \\
\bottomrule
\end{tabular*}

\par\noindent\textit{(b) Individual-output factor recovery on converged seeds}
\begin{tabular*}{\textwidth}{@{\extracolsep{\fill}}lrrrrr@{}}
\toprule
Output & $R_S^2$ & $R_A^2$ & $R_{F_1}^2$ & $R_{F_2}^2$ & $R_{F_1F_2}^2$ \\
\midrule
P/I, linear & $-.0015\!\pm\!.0002$ & $.0308\!\pm\!.0066$ & $.0000\!\pm\!.0014$ & $.0057\!\pm\!.0074$ & $.2678\!\pm\!.0176$ \\
P/I, quadratic & $-.0021\!\pm\!.0008$ & $.0599\!\pm\!.0239$ & $.0172\!\pm\!.0114$ & $.0119\!\pm\!.0115$ & $.3035\!\pm\!.0505$ \\
R/I, linear & $-.0015\!\pm\!.0000$ & $.0876\!\pm\!.0091$ & $.0004\!\pm\!.0015$ & $.0055\!\pm\!.0074$ & $.3220\!\pm\!.0271$ \\
R/I, quadratic & $-.0024\!\pm\!.0006$ & $.1273\!\pm\!.0251$ & $.0167\!\pm\!.0123$ & $.0143\!\pm\!.0108$ & $.3787\!\pm\!.0622$ \\
\bottomrule
\end{tabular*}

\par\noindent\textit{(c) Task AUROC on converged seeds; Res/Comb: individual only/expert plus individual}
\begin{tabular*}{\textwidth}{@{\extracolsep{\fill}}lrrrr@{}}
\toprule
Output & Linear Res & Linear Comb & Quadratic Res & Quadratic Comb \\
\midrule
P/I & $.9700\!\pm\!.0083$ & $.9990\!\pm\!.0001$ & $.9823\!\pm\!.0116$ & $.9983\!\pm\!.0001$ \\
R/I & $.9827\!\pm\!.0030$ & $.9954\!\pm\!.0009$ & $.9958\!\pm\!.0010$ & $.9976\!\pm\!.0000$ \\
\bottomrule
\end{tabular*}

\end{table}
\clearpage

\section{Experimental Details: Skin-lesion Classification}
\label{app:skincon}
\label{app:skincon_protocol}
This study tests three questions: whether the residual improves
prediction beyond the observed concepts, whether Crust explains part
of the gain from a fixed residual, and whether restoring Crust changes
its linear accessibility after retraining. The latter two comparisons
are separate experiments. We report both seeds, the in-distribution
Fitzpatrick17k test set, and the external DDI cohort. Models trained
end to end and operators fitted to common embeddings are kept separate.
The following protocol and tables use the completed studies; no new
model or probe was fitted for this appendix.

\subsection{Data, exclusions, and diagnosis label}
\label{app:skincon_data}
We use the 48 dermatologist annotations provided by SkinCon~\citep{skincon}
on Fitzpatrick17k~\citep{fitz} and DDI~\citep{ddi} images. Image identifiers join the annotations,
images, and diagnosis metadata. Every retained image has all 48 binary
annotations; we do not impute missing concept values. The source
annotation file contains 3,690 Fitzpatrick17k records, of which 460
are flagged by SkinCon for exclusion. We also exclude three records
marked wrongly labelled in the source metadata and both records of
one exact-image duplicate pair whose Erosion and skin-tone annotations
conflict. These exclusions leave 3,225 images. The duplicate pair is
separate from the three wrongly labelled records. The DDI annotation
file contains 656 records; its 20 SkinCon-flagged records are excluded,
leaving 636 images. All-zero concept vectors in flagged records are
not treated as observed concept absence.

The three additional Fitzpatrick17k exclusions are
\texttt{02fd9e1b56e0c3adae18233f430f2fea.jpg},
\texttt{c5c075da4848a22e7b904957deb7c609.jpg}, and
\texttt{c93610357aa4f74a99c6e7f1b009f9cb.jpg}.
The two excluded copies are
\texttt{22f1d783dd6821defafcc915a8146c41.jpg} and
\texttt{5f567374f688120109c2bab349ff21cf.jpg}.
We preserve the original images and annotations separately from the
analysis eligibility fields, without resolving the disagreement by
choosing one copy.

The binary source label is malignant versus benign or non-neoplastic.
DDI uses its supplied malignant indicator, except that
\texttt{000264.png}, \texttt{000502.png}, \texttt{000538.png}, and
\texttt{000556.png}, labelled actinic keratosis, are assigned to the
positive class under our study convention. These are study-specific
label overrides, not corrections to the source pathology labels.
They increase the positive count from 167 to 171. Other differences
in diagnosis labels between the datasets were not exhaustively
adjudicated. We treat skin-tone value $-1$ as unknown and retain the
84 eligible Fitzpatrick17k images with unknown primary skin tone in
the overall analysis. This value is not used for missing concepts.
DDI reports skin-tone groups 1--2, 3--4, and 5--6; we do not infer an
individual Fitzpatrick value from those groups.

\paragraph{Image groups and splits.}
We group exact oriented-RGB matches, identical source URLs, and
aligned perceptual duplicate candidates before splitting. Perceptual
candidates have pHash distance at most 6 or dHash distance at most 4,
absolute log-aspect-ratio difference at most $.08$, and normalized
grayscale correlation at least $.97$. Eleven candidate pairs were
identified and kept within their image groups. The scan found no
cross-dataset candidates and quarantined no additional source images.
It does not exhaustively detect crops, rotations, altered borders, or
multiple photographs of the same patient.

A fixed split seed of 909 allocates source image groups to training,
validation, and test sets in proportions $.6/.2/.2$, balancing image
and diagnosis counts with larger groups assigned first. All model
seeds share these assignments. Table~\ref{tab:skincon_cohort_counts}
reports the resulting counts. DDI is the external evaluation cohort
and does not enter this split assignment. Patient identifiers are
unavailable, so these are image-level splits; image matching does not
establish independence between patients.

We further split the 1,935 source training images by image group into
965 calibration images and 970 audit-readout images, using a fixed
50\% group split with seed 909. The calibration subset fits the
residualization operators; the readout subset fits and selects the
post-hoc predictors. Both subsets train the image encoder. This
separates operator and readout fitting, but does not make either
subset independent of representation training. Validation selects
network checkpoints, while readout selection uses folds within the
970 readout images.

\begin{table}[t]
\centering
\caption{SkinCon cohort counts. Positive denotes the study's binary diagnosis label; concept columns count positive annotations. Calibration and readout rows partition the training row, rather than adding further images.}
\label{tab:skincon_cohort_counts}
\setlength{\tabcolsep}{3pt}
\begin{tabular*}{\textwidth}{@{\extracolsep{\fill}}lrrrrrr@{}}
\toprule
Cohort / partition & Images & Diagnosis $+$ & Diagnosis $-$ & Crust & Erosion & Scale \\
\midrule
Fitzpatrick17k training & 1,935 & 286 & 1,649 & 294 & 116 & 393 \\
\quad Operator calibration & 965 & 129 & 836 & 128 & 67 & 199 \\
\quad Audit-readout fitting & 970 & 157 & 813 & 166 & 49 & 194 \\
Fitzpatrick17k validation & 645 & 96 & 549 & 101 & 43 & 143 \\
Fitzpatrick17k test & 645 & 95 & 550 & 102 & 40 & 148 \\
DDI external evaluation & 636 & 171 & 465 & 53 & 14 & 103 \\
\bottomrule
\end{tabular*}
\end{table}

\subsection{Expert views and image preprocessing}
\label{app:skincon_views}
Let $\tilde{x}_{45}$ denote the 45 annotations remaining after omitting
Crust, Erosion, and Scale. Crust is the primary omitted concept,
denoted by $C$. The omitted network input is
$(\tilde{x}_{45},0,0,0)$, the Crust-restored input is
$(\tilde{x}_{45},C,0,0)$, and its shuffled control is
$(\tilde{x}_{45},\widetilde C,0,0)$. The expert network therefore has
48 input slots in every condition; Erosion and Scale remain zero in
the Crust-only comparison. These zeros implement deliberate omission,
not a claim that the omitted concept is absent from the image.
Earlier models that restored all three concepts are a separate
experiment and are not the Crust-only intervention.

For shuffled restoration, we randomly reassign the true binary Crust
annotations among ID-sorted images within each partition. Each model
seed has one fixed, prevalence-preserving permutation for calibration,
readout fitting, validation, source test, and DDI. The training
assignment is the union of the calibration and readout assignments,
so an image receives the same expert input during training and later
readout fitting. There is no diagnosis stratification, and an image
may retain its original binary value by chance. The concept labels
used to evaluate recovery always remain the true annotations.

Images undergo EXIF orientation correction, RGB conversion,
aspect-preserving bicubic resizing to fit within $224\times224$, and
centered padding with the ImageNet channel means. We apply the fixed
ImageNet mean and standard-deviation normalization. Training adds
independent horizontal and vertical flips with probability $.5$,
rotation $\theta\in[-10^\circ,10^\circ]$, and a zoom-out factor
$u\in[.9,1]$. To keep the full image visible after rotation, the total
geometric scale is $u/(|\cos\theta|+|\sin\theta|)$; there is no random
crop. Brightness and contrast factors lie in $[.95,1.05]$, with no hue
or saturation changes. Transformations are fixed by image ID, model
seed, and epoch. Calibration, validation, and evaluation use only the
deterministic preprocessing.

\subsection{Matched diagnosis networks and training}
\label{app:skincon_native_training}
The SkinCon experiments run on a MacBook Pro with an Apple M5 Max
chip and 128\,GB of RAM.
The dual baseline, SCCR, OLS, and regularized LEACE-style models use
the same architecture. An ImageNet-initialized EfficientNetV2-S~\citep{pmlr-v139-tan21a,5206848} image
encoder produces a pooled 1,280-dimensional vector. We normalize this
vector to unit $\ell_2$ norm and project it through
$1280\to32\to32$, with SiLU and dropout $.2$ between the two linear
layers. The expert input passes through a $48\to64$ linear layer and
SiLU, unit $\ell_2$ normalization, and a $64\to32\to32$ projection of
the same form. Thus $z_r,z_e\in\mathbb R^{32}$ are the projected codes
that enter the operator, rather than the pooled image features.

The prediction head separately applies layer normalization,
a $32\to32$ linear layer, SiLU, and dropout $.2$ to the residual and
centered expert codes. Writing the resulting vectors as $h_r,h_e$,
the single diagnosis logit is
\[
\ell=w_r^\top h_r+w_e^\top h_e+w_{\rm int}^\top(h_r\odot h_e)+b.
\]
We train with unweighted binary cross-entropy on the diagnosis label.
There is no auxiliary concept-prediction loss or factor supervision.
The matched networks each have 20,228,145 trainable parameters. All
encoder, projection, and head parameters update from the first step.

AdamW uses learning rate $10^{-4}$, weight decay $10^{-3}$, and
gradient-norm clipping at 5, without a learning-rate schedule.
Each training image appears once per epoch. The shuffled sampler
makes 15 equally sized batches of 129, avoiding a small covariance
batch at the end; it does not balance the diagnosis classes.
Evaluation batches contain at most 64 images. Training has a maximum
of 100 epochs and stops after five epochs without lower validation
log loss. Strict ties keep the earlier checkpoint. The two model
seeds are 909 and 910, with matched initialization, image order, and
augmentation across the matched conditions within each seed.

During training, SCCR estimates its OAS-regularized CCA operator from
the current minibatch and keeps the operator in the gradient graph.
We use $\tau=.40$, fixed after earlier development. Before each
validation pass, the current encoders are applied to the 965
unaugmented calibration images. We recompute means, covariances,
branchwise OAS ridges, and the frozen operator $T^\star$ from those
embeddings. The selected checkpoint stores this fitted operator with
its network weights. Validation and final evaluation use that fixed
map without estimating moments from the evaluation images. These
calibration images belong to the training split but are distinct
from the audit-readout subset. OLS and regularized LEACE-style follow
the same batch-fitting and epoch-level calibration schedule. The
dual baseline centers the raw code without removing directions.

\paragraph{Numerical implementation.}
With $n$ rows, the covariance denominator is $B=\max(n-1,1)$, as
in Section~\ref{sec:method}. The OAS calculation uses this same
effective count in Equation~\eqref{eq:oas_rho}. In code, a denominator
at most $10^{-12}$ gives the maximum shrinkage; otherwise $10^{-12}$
is added to the denominator. Shrinkage is clipped to
$[0,1-10^{-6}]$, and the ridge is
$\lambda_j=\rho_j m_j/\max(1-\rho_j,10^{-6})$.
The native network centers float32 embeddings and computes the
covariances, OAS ridges, Cholesky factors, and spectral decomposition
in CPU float64. It casts the final map back to network precision.
The device transfers preserve gradients. The selected-subspace
projector uses divided differences of squared singular values only
across removed/retained groups; threshold membership itself remains discrete.
Empty or full selections return exact identity or zero maps.

Cholesky first tries zero jitter, then diagonal additions
$10^{-6},10^{-5},10^{-4},10^{-3},10^{-2}$ if needed. The selected
calibration maps in the reported native runs required zero jitter;
this does not establish that every training batch did. Common-embedding
SCCR instead uses symmetric covariance roots in float64, with relative
support cutoff $10^{-12}$ and the same OAS ridge calculation. It
has no training gradient. Its full-support map is equivalent to the
Cholesky convention up to numerical precision. Native inference uses
the saved float32 means and map.

The Crust-only intervention trains eight new networks: dual baseline
or SCCR, true or shuffled Crust, and two seeds. Each starts from its
matched fresh initialization, rather than continuing an omitted-view
checkpoint. Four completed omitted-view models are reused. The four
other omitted-view baseline methods contribute eight further fits.
Table~\ref{tab:skincon_training_completion} reports the actual stopping
and selected epochs. VCCA-private and DSN-MMD use the separate losses
and architectures described below; VCCA-private selects validation
negative ELBO, while DSN-MMD selects shared-code diagnosis log loss.
All models in the table reached their declared early-stopping rule.
The 100-epoch setting is a maximum, not the number completed by every
model.

\begin{table}[!htbp]
\centering
\setlength{\tabcolsep}{2pt}
\caption{Completed SkinCon training. Epoch cells give completed / selected checkpoint, both one-based; all runs reached patience-five early stopping. True and shuffled conditions restore Crust only.}
\label{tab:skincon_training_completion}
\label{tab:skincon_tau_training}
\begin{tabular*}{\textwidth}{@{\extracolsep{\fill}}llcccc@{}}
\toprule
Model & Expert condition & Seed 909 & Seed 910 & $q_{909}$ & $q_{910}$ \\
\midrule
Dual baseline & Omitted & 23 / 18 & 19 / 14 & --- & --- \\
SCCR & Omitted & 29 / 24 & 21 / 16 & --- & --- \\
OLS & Omitted & 22 / 17 & 23 / 18 & --- & --- \\
Regularized LEACE-style & Omitted & 65 / 60 & 54 / 49 & --- & --- \\
VCCA-private & Omitted & 70 / 65 & 44 / 39 & --- & --- \\
DSN-MMD & Omitted & 16 / 11 & 21 / 16 & --- & --- \\
Dual baseline & True Crust & 16 / 11 & 13 / 8 & --- & --- \\
Dual baseline & Shuffled Crust & 20 / 15 & 16 / 11 & --- & --- \\
SCCR & True Crust & 31 / 26 & 23 / 18 & --- & --- \\
SCCR & Shuffled Crust & 41 / 36 & 21 / 16 & --- & --- \\
\midrule
\multicolumn{6}{@{}l}{\textit{SCCR sensitivity runs ($\tau=.20$)}} \\
SCCR ($\tau=.20$) & Omitted & 33 / 28 & 42 / 37 & 9 & 10 \\
SCCR ($\tau=.20$) & True & 36 / 31 & 41 / 36 & 8 & 9 \\
SCCR ($\tau=.20$) & Shuffled & 35 / 30 & 32 / 27 & 10 & 11 \\
\bottomrule
\end{tabular*}
\par\smallskip
Diagnosis models select validation log loss; VCCA-private selects validation negative ELBO. Removed rank $q$ is shown for the additional SCCR $.20$ fits at their selected checkpoints; dashes mean not tabulated here.
\end{table}

\subsection{Baseline objectives and adaptations}
\label{app:skincon_baselines}
We distinguish models that learn their own representations from
operators fitted to the same frozen representations. The native OLS
and regularized LEACE-style networks use the same image encoder,
expert encoder, projected dimensions, diagnosis head, and training
recipe as SCCR. VCCA-private and DSN-MMD use their own shared/private
architectures and objectives. The common-embedding comparison uses the
selected dual baseline with the three concepts omitted, separately for
each seed. Its projected codes $z_r,z_e\in\mathbb R^{32}$ are fixed;
only the 965 calibration images fit an operator. We write $R$ for the
retained raw representation and use the same observed annotations
$\tilde{x}_{45}$ as the diagnosis comparison input for all methods.
This annotation input differs from each network's learned expert code.
The common-embedding operators reported here are SCCR, matched top-$q$,
exact empirical LEACE, D-CCA, and the primary paired individual
outputs of JIVE and AJIVE. OLS is reported only as an end-to-end model.

\paragraph{OLS and regularized LEACE-style training.}
For a centered training batch, OLS fits
\[
 B_{\rm OLS}=Z_{e,c}^{\dagger}Z_{r,c},\qquad
 R=Z_{r,c}-Z_{e,c}B_{\rm OLS},
\]
using a relative pseudoinverse cutoff of $10^{-6}$. The fit remains in
the training graph. Before validation, the current network is
recalibrated on the 965 calibration images; its selected checkpoint
stores the means and coefficient matrix for later inference. Both
views are needed to compute an OLS residual. For this fixed map,
$Z_{r,c}=R+Z_{e,c}B_{\rm OLS}$, so the pair $(Z_e,R)$ retains the raw
code. Lower performance with a particular audit predictor therefore
does not establish loss of all predictive information.

The trained LEACE-style map follows the same batch-fit and global
calibration schedule. Let $R_{rr}=S_{rr}+\lambda_r I=L_rL_r^\top$,
with the raw branch's OAS ridge, and let $U_H$ span the left singular
vectors of $L_r^{-1}S_{re}$ whose singular values exceed $10^{-6}$
times the largest. Its row-vector map is
\[
 T_{\rm RLEACE}=L_r^{-\top}(I-U_HU_H^\top)L_r^\top,
 \qquad R=Z_{r,c}T_{\rm RLEACE}.
\]
This is a regularized adaptation of LEACE~\citep{10.5555/3666122.3669006},
not exact empirical covariance erasure. Both selected native fits
remove all 32 directions and produce a literal zero residual.
This observed full-rank outcome is distinct from a failed fit.

\paragraph{SCCR and exact empirical LEACE on common embeddings.}
Common-embedding SCCR uses branchwise OAS ridge and $\tau=.4$, with
the same row-vector operator as in Section~\ref{sec:method}. It removes
one direction in seed 909 and two in seed 910. Matched top-$q$ reuses
that exact fitted spectrum and those values of $q$; the saved matrices
are identical. This comparison checks operator equivalence and does
not supply an advantage for thresholding.

For exact empirical LEACE, define $Q_r=S_{rr}^{1/2}$ and
$W_r=S_{rr}^{+1/2}$ on the empirical covariance support, and let $P$
project onto the column space of $W_rS_{re}$. The retained row code is
\[
 R=Z_{r,c}(I-W_r P Q_r).
\]
The covariance support and cross-covariance use a relative numerical
cutoff of $10^{-12}$, without OAS ridge. Both fits have full raw
support and erasure rank 32. We store the resulting map as exact zero,
so numerical roundoff cannot become a predictive feature after
standardization. Frozen SCCR and empirical LEACE require only the
raw embedding to compute $R$; their combined audit predictors still
receive the observed expert annotations.

\paragraph{JIVE and AJIVE.}
JIVE~\citep{10.1214/12-AOAS597} fits joint and individual components by
minimizing the sum of squared reconstruction errors of the two centered
blocks, each divided by its calibration Frobenius norm:
\[
 \min_{\widetilde J_r,\widetilde J_e,\widetilde I_r,\widetilde I_e}
 \sum_{v\in\{r,e\}}\|Z_{v,c}/a_v-\widetilde J_v-\widetilde I_v\|_F^2.
\]
Here the components in this objective are in scaled coordinates. The joint
components share a sample-score space; individual components are
orthogonal to that space. The supplied joint and individual ranks are
upper bounds. We use alternating truncated SVD, at most 5,000
iterations, and a relative fitted-matrix change tolerance of $10^{-9}$,
without expanding the individual ranks.

The declared initial-rank rule selects the largest adjacent
singular-value ratio over ranks $1,\ldots,16$, with the smaller rank
breaking ties. It uses no diagnosis or concept outcomes. It returns
$(1,1)$ in both seeds. JIVE then uses joint rank one and individual
ranks $(0,0)$. The resulting individual outputs are zero by this rank
choice; this is not evidence that JIVE always discards useful signal,
and the native JIVE permutation rank selector was not used.

AJIVE~\citep{10.1016/j.jmva.2018.03.008} uses the same initial ranks,
then applies its Wedin and random-direction bounds to the alignment
of the retained sample-score spaces. If $\mathcal Q_r,\mathcal Q_e$
are their orthonormal bases, the squared singular values of
$[\mathcal Q_r\ \mathcal Q_e]$ satisfy
$d_i^2=1+\sigma_i(\mathcal Q_r^\top\mathcal Q_e)$ for the leading paired directions.
We use 1,000 samples for each
bound, the 5th and 95th percentiles respectively, and the native
identifiability check and individual reconstruction thresholds.
There is no extra cap on the final individual ranks. Both primary
fits select joint rank zero and individual ranks $(1,1)$. Thus their
individual outputs retain the fitted rank-one signal rather than
subtracting a nonzero fitted joint component.

Both methods originally decompose the calibration sample matrix.
We therefore state their extension to a new image explicitly. Let
$I_r$ be the fitted raw individual component in raw feature units,
and let $a_r,a_e$ be the fitted block scales (Frobenius norms for
JIVE, one for AJIVE). Using calibration rows only, fit
\[
 \begin{aligned}
 B_I&=[Z_{r,c}/a_r,\,Z_{e,c}/a_e]^{\dagger}(I_r/a_r),\\
 R_{\rm new}&=a_r[Z_{{r,\rm new},c}/a_r,\,
                    Z_{{e,\rm new},c}/a_e]B_I.
 \end{aligned}
\]
The pseudoinverse cutoff is $10^{-12}$. This paired individual
extension requires both embeddings and can therefore inherit expert
information. Assessment rows are never added to a new decomposition.
The primary gap-rank fits converge. Separate fixed-rank checks
included six JIVE fits that did not meet the convergence criterion;
these remain unavailable, rather than receiving chance scores, and
are distinct from the primary rank-conditioned zero outputs.

\paragraph{Decomposition-based CCA.}
D-CCA~\citep{Shu02012020} uses unregularized empirical signal
covariances. For each unit-variance canonical score pair $(u_i,v_i)$
with correlation $s_i$, define
\[
 c_i=\frac{s_i}{1+s_i+\sqrt{1-s_i^2}}(u_i+v_i),
 \qquad d_{r,i}=u_i-c_i.
\]
The distinctive scores $d_{r,i}$ are reconstructed in raw feature
units. In matrix form, let
$W_rS_{re}W_e=U\operatorname{diag}(s)V^\top$ with $W_e=S_{ee}^{+1/2}$, using unregularized
covariance inverse square roots. Writing
$U_r=Z_{r,c}W_rU$, $U_e=Z_{e,c}W_eV$, and
$C_{\rm shared}=(U_r+U_e)\operatorname{diag}(s_i/(1+s_i+\sqrt{1-s_i^2}))$,
the retained code is $R=(U_r-C_{\rm shared})U^\top Q_r$.
Both seeds have 32 supported directions in each view, and all
32 canonical pairs are used, with numerical support cutoff
$10^{-12}$. This is an empirical signal-covariance adaptation without
S-POET denoising or a significance-based common-rank selector. It
requires both views at inference. D-CCA denotes decomposition-based
CCA here; Deep CCA and sJIVE were not fitted in this study.

\paragraph{VCCA-private.}
VCCA-private~\citep{wang2017deepvariationalcanonicalcorrelation} learns
32-dimensional shared, raw-private, and expert-private Gaussian codes,
written $s,p_r,p_e$. Its factorization and variational distribution are
\[
 \begin{aligned}
 p(x,\tilde{x}_{45},s,p_r,p_e)
   &=p(s)p(p_r)p(p_e)p(x\mid s,p_r)
                     p(\tilde{x}_{45}\mid s,p_e),\\
 q(s,p_r,p_e\mid x,\tilde{x}_{45})
   &=q(s\mid x)q(p_r\mid x)q(p_e\mid\tilde{x}_{45}).
 \end{aligned}
\]
All priors are standard normal and posteriors are diagonal Gaussian.
The objective is
\[
 \begin{aligned}
 \mathcal L_{\rm VCCA}={}&
 \mathbb E_q[-\log p(x\mid s,p_r)
             -\log p(\tilde{x}_{45}\mid s,p_e)]\\
 &+\operatorname{KL}(q_s\Vert p_s)
  +\operatorname{KL}(q_{p_r}\Vert p_{p_r})
  +\operatorname{KL}(q_{p_e}\Vert p_{p_e}).
 \end{aligned}
\]
The image likelihood is Gaussian with variance one on the actual
$224\times224$ RGB values, obtained by reversing ImageNet
normalization. The annotation likelihood is Bernoulli on the 45
observed concepts. Each likelihood sums observation coordinates
within an image and then averages the batch; each KL sums 32 latent
coordinates. All coefficients are one. The omitted concepts and
diagnosis labels do not enter this training loss.

Two independent ImageNet-initialized EfficientNetV2-S encoders infer
the raw shared and raw-private posteriors. Each normalized
1,280-dimensional image feature vector passes through a
$1280\to64\to64$ Gaussian head with SiLU, yielding 32 means and 32
log variances. The expert input uses a $45\to64$ SiLU encoder,
normalization, and a $64\to64\to64$ Gaussian head. Log variances are
clipped to $[-12,8]$. Decoders receive concatenated shared and private
samples. The image decoder maps 64 coordinates to $128\times7\times7$
features, then uses five bilinear upsampling stages with $3\times3$
convolutions and SiLU, with channels $64,32,16,16,16$, followed by a
three-channel convolution. The expert decoder is $64\to64\to45$ with
SiLU and Bernoulli logits. Neither decoder has encoder skip connections.

Training uses one reparameterized sample per image. Checkpoint
selection minimizes validation negative ELBO estimated with four
fixed posterior draws per image, role, and draw index; these draws
do not change with epoch or batch order. Audits use posterior means:
$R=\mathbb E_q[p_r\mid x]$, shared code $\mathbb E_q[s\mid x]$, and their
64-dimensional concatenation as the full-code control. The exported
expert code is $\mathbb E_q[p_e\mid\tilde{x}_{45}]$, an expert-private
code rather than a shared code inferred from annotations. VCCA has
no native diagnosis head.

\paragraph{Domain Separation Networks.}
DSN-MMD~\citep{NIPS2016_45fbc6d3} is adapted to heterogeneous paired
images and annotation vectors. Two independent image encoders supply
raw shared and raw-private features. Their pooled features are
normalized to unit $\ell_2$ norm before the latent-code layers. The raw shared stem is
$1280\to64$ with SiLU and the expert shared stem is $45\to64$ with
SiLU. Both feed the same $64\to64\to32$ shared bridge. Independent
private projectors are $1280\to64\to32$ for the raw image and
$45\to64\to32$ for annotations. All hidden activations are SiLU.
Write the shared and private codes as $s_r,s_e,p_r,p_e$. Each decoder
receives the sum of shared and private codes from its own view;
the image decoder has the VCCA upsampling architecture with a
32-dimensional input, and the expert decoder is $32\to64\to45$.
The diagnosis classifier reads only $s_r$, using LayerNorm,
$32\to64$, SiLU, dropout .2, and a scalar output.

The training objective is
\[
 \begin{aligned}
 \mathcal L_{\rm DSN}={}&\mathcal L_{\rm BCE}(y,g(s_r))
 +\mathcal L_{\rm image}(x,d_r(s_r+p_r))\\
 &+\mathcal L_{\rm BCE}(\tilde{x}_{45},d_e(s_e+p_e))
 +.1\mathcal L_{\rm diff}+.1\mathcal L_{\rm MMD}.
 \end{aligned}
\]
For an image with $P=3\cdot224^2$ entries and reconstruction error
$e$, the image term is $P^{-1}\sum_{j=1}^P(e_j-\bar e)^2$, averaged
over images. It removes the mean reconstruction error; it does not
imply invariance to arbitrary multiplicative scaling. Annotation BCE
averages the 45 coordinates and images. For each view, center the
shared and private latent columns and normalize each row to unit
Euclidean norm, giving $H_{s,v},H_{p,v}$. Then
\[
 \mathcal L_{\rm diff}
 =\sum_{v\in\{r,e\}}\left\|H_{s,v}^\top H_{p,v}/n\right\|_F^2.
\]
The MMD term is the biased squared statistic, including diagonal
terms,
\[
 \mathcal L_{\rm MMD}=\frac{1}{n^2}\sum_{i,j=1}^n
 [k(s_{r,i},s_{r,j})+k(s_{e,i},s_{e,j})
                         -2k(s_{r,i},s_{e,j})],
\]
with equal-weight kernel
\[
 k(a,b)=\frac17\sum_{\sigma\in\{.125,.25,.5,1,2,4,8\}}
             \exp[-\|a-b\|^2/(2\sigma^2)].
\]
It aligns marginal shared-code distributions, rather than forcing
paired codes to be identical or proving independence. Selection uses
the native raw-shared validation diagnosis log loss. The audit output
is $R=p_r$; $s_r$ and $[s_r,p_r]$ are separate learning controls. The
exported expert code is $s_e$.

\paragraph{Learning controls and limits of the comparison.}
The two shared/private models each have two image encoders and a
64-dimensional full raw code, compared with one image encoder and a
32-dimensional raw code in the matched residual networks. We report
shared, private, and full-code accessibility separately. Poor private
prediction cannot identify a decomposition failure when the full
representation also has poor accessibility under the fitted probes.
Likewise, diagnosis supervision differs: DSN trains a shared-code
classifier, while VCCA optimizes reconstruction and KL terms only.
These adaptations do not establish that the original methods cannot
work on this task.

Table~\ref{tab:skincon_reconstruction_controls} reports validation
reconstruction diagnostics at the selected checkpoints, together
with a fixed training-mean-image and annotation-prevalence reference.
VCCA improves RGB reconstruction but does not improve annotation BCE
relative to the reference. DSN learns diagnosis, but its reconstruction
criteria do not improve on their references. VCCA's table uses
posterior-mean reconstructions, so these values differ from its
stochastic negative-ELBO checkpoint criterion. These are learning
diagnostics, not bootstrap comparisons or common loss functions
across the two models.

\begin{table}[t]
\caption{SkinCon reconstruction controls on the 645 validation images.
Lower is better. The reference uses the training mean RGB image and
per-concept training prevalence. VCCA uses posterior means. Centered
image error is the DSN reconstruction criterion; ordinary RGB MSE is
also shown.}
\label{tab:skincon_reconstruction_controls}
\centering
\setlength{\tabcolsep}{3pt}
\begin{tabular*}{\textwidth}{@{\extracolsep{\fill}}llccc@{}}
\toprule
Model & Seed & RGB MSE & Centered image error & Annotation BCE\\
\midrule
VCCA-private & 909 & .025619 & .024970 & .149454\\
VCCA-private & 910 & .026722 & .025869 & .139893\\
DSN-MMD & 909 & .235905 & .026299 & .616325\\
DSN-MMD & 910 & .274671 & .026180 & .507646\\
Training-mean reference & --- & .032361 & .025816 & .132480\\
\bottomrule
\end{tabular*}
\end{table}

\subsection{Audit predictors and their selection}
\label{app:skincon_readouts}
After checkpoint selection and operator calibration, we freeze the
encoders and all decomposition maps. A post-hoc predictor, also called
a readout or probe, is then fitted to a specified code or annotation
input. These predictors are separate from the trained neural diagnosis
head. Crust recovery uses the correct binary annotation as its target;
diagnosis prediction uses the study's binary diagnosis label. We report
both linear logistic regression and nonlinear gradient-boosted trees.
AUROC uses their continuous scores, without a classification threshold.
AUROCs below $.5$ are not reversed after inspecting evaluation labels.

All predictor fitting and selection use the 970 audit-readout images.
We use three-fold stratified group cross-validation, with detected
image groups kept together, stratification by the target, and split
seed 909 for both neural-network seeds. Thus methods use the same
folds for a given target. Crust and diagnosis need not use the same
fold assignments. Both the calibration and readout images previously
trained the CNN: this procedure validates the predictors conditional
on the learned representation, not the complete learning pipeline.
The 645-image validation set selects network checkpoints and does not
select these final readout settings. Neither Fitzpatrick17k test nor
DDI fits a predictor, probability calibrator, or removal map.

\paragraph{Candidate predictors.}
Within each fitting fold, we remove columns with standard deviation
at most $10^{-10}$ and standardize the remaining columns using that
fold's means and standard deviations. Linear candidates are
$\ell_2$-regularized logistic regression with inverse regularization
$C_{\rm LR}\in\{10^{-5},10^{-4},10^{-3},.01,.1,1,10\}$,
L-BFGS, at most 2,000 iterations, and tolerance $10^{-6}$.
Nonlinear candidates use histogram gradient boosting with at most
7 or 15 leaves, depth 3, 100 boosting iterations, learning rate $.1$,
minimum leaf size 20, and $\ell_2$ penalty 10. There is no internal
early stopping; its random seed is 909. Each family also includes a
constant predictor equal to its fitting-fold prevalence.
Failed or nonconverged learned candidates are excluded from selection
and are not silently replaced by successful learned models.

Concept probes maximize the unweighted mean of the three held-out
AUROCs. Their primary scores are uncalibrated. Diagnosis predictors
minimize sample-weighted held-out log loss, considering both calibrated
and uncalibrated versions. Exact ties favor the constant prior and
then the uncalibrated version. Once selected, the base predictor is
refitted to all 970 images. Selection of a constant prior is marked
P in the tables; it means that this predictor family did not select
a better learned model, rather than proving absence of information.

\paragraph{Diagnosis probability calibration.}
For a candidate logit score $s$, probability calibration fits
\[
 p_{\rm cal}(s)=\sigma\!\left(
 \alpha\frac{s-\mu_s}{\sigma_s}+\beta\right),
 \qquad .001\leq\alpha\leq100,\quad -30\leq\beta\leq30.
\]
Here $\mu_s,\sigma_s$ come from out-of-fold fitting scores, and
$\sigma$ is the logistic function. Inside each outer fitting fold,
two inner stratified group folds produce scores for fitting this
calibrator. The outer held-out labels fit neither the base predictor
nor its calibrator. For the final predictor, the selected candidate's
three-fold out-of-fold scores on the 970 images fit the calibrator,
and its base model is refitted to all 970. If the out-of-fold score
standard deviation is at most $10^{-12}$, we use unit scale and fix
the raw slope to one, fitting only the prevalence intercept. The positive slope preserves
the ordering of finite logits, although saturated probabilities can
create ties. This fitted correction does not guarantee calibration
on new data or DDI.

\paragraph{Expert-only fallback and learning controls.}
For each family, we select a predictor on the observed annotations
$\tilde{x}_{45}$ and a predictor on $(\tilde{x}_{45},R)$. The primary
combined predictor is used only when its cross-validated log loss is
strictly lower than that of the expert-only predictor; otherwise the
expert-only predictor is used for both terms. This fallback, marked F,
gives zero added gain by construction. It does not establish that $R$
is uninformative. We also report the selected combined predictor
without this fallback, so an apparent zero cannot hide a positive
or negative assessment result. The choice is made on source-training
folds and is never changed after observing test or DDI performance.

Raw, residual, and learned expert codes are probed separately. For
VCCA-private and DSN-MMD we additionally probe the full, shared, and
private raw codes. These checks distinguish poor accessibility in the
whole learned representation from allocation of an accessible signal
to a particular component. Crust is the primary concept in the tables;
the saved analysis also contains secondary probes for Erosion, Scale,
Plaque, Papule, and Erythema. Those secondary probes do not select the
reported model, threshold, or Crust contrast.

\subsection{Three audit questions}
\label{app:skincon_estimands}
\paragraph{Does the residual improve diagnosis prediction?}
Let $\widehat p_W$ be the frozen source-fitted diagnosis predictor
whose input is $W$. On an evaluation cohort of $n$ images, define
\[
 \mathcal L(W)=-\frac1n\sum_{i=1}^n
 \left[y_i\log\widehat p_W(W_i)
 +(1-y_i)\log\{1-\widehat p_W(W_i)\}\right],
 \qquad
 G=\mathcal L(\tilde{x}_{45})-
   \mathcal L(\tilde{x}_{45},R).
\]
Positive $G$ means that the residual improves fitted diagnosis
prediction relative to the 45 observed annotations. These annotations
are common across methods; they are not each model's learned expert
code. We report $G$ separately from residual-only AUROC and native
neural-head performance. Its value depends on the fitted predictors
and their selection; it is not a direct estimate of conditional mutual
information or proof that the annotation set contains no equivalent
nonlinear information. Paired differences $G_{\rm SCCR}-G_m$ compare
SCCR with baseline $m$ under the same evaluation images and probe family.

\paragraph{Does Crust explain part of a fixed residual's gain?}
We take the SCCR network trained with Crust, Erosion, and Scale
omitted and hold its residual $R$ fixed. We add either each image's
correct Crust annotation $C$ or a fixed random reassignment
$\widetilde C$ only to the diagnosis readouts. For
$H\in\{\emptyset,C,\widetilde C\}$, define
\[
 \begin{aligned}
 G(H)&=\mathcal L(\tilde{x}_{45},H)
       -\mathcal L(\tilde{x}_{45},H,R),\\
 A_H&=G(\emptyset)-G(H),\\
 A_C-A_{\widetilde C}&=G(\widetilde C)-G(C).
 \end{aligned}
\]
We refit and select these readouts using the same source-only
procedure; we do not change the encoder, expert input to the network,
or SCCR operator. A positive last contrast means that correct Crust
reduces the measured gain more than randomly reassigned Crust.
This tests predictive overlap with an omitted annotation, not causal
mediation or removal of Crust from $R$.

We also separate the two contributions to true-Crust attenuation:
\[
 \begin{aligned}
 A_C={}&\underbrace{\mathcal L(\tilde{x}_{45})
             -\mathcal L(\tilde{x}_{45},C)}_{
             \text{improvement in expert-only prediction}}\\
       &-\underbrace{\left[\mathcal L(\tilde{x}_{45},R)
             -\mathcal L(\tilde{x}_{45},C,R)\right]}_{
             \text{improvement in combined prediction}}.
 \end{aligned}
\]
Thus attenuation can arise because the annotations predict better
after Crust is supplied, even when the combined predictor changes
little. The residual itself remains exactly the same.

\paragraph{Does restoring Crust change its accessibility after retraining?}
This is a separate intervention. Write $v=O,T,S$ for the omitted,
true-Crust, and shuffled-Crust network conditions. Each has its own
learned codes and calibrated map. For either probe family, define
\[
 D_v=\mathrm{AUROC}(C\mid Z_r^{(v)})-
     \mathrm{AUROC}(C\mid R^{(v)}).
\]
Here AUROC means the performance of a separately source-fitted probe
for the true Crust annotation. The primary comparison is $D_T-D_S$:
does correct Crust increase the raw-to-residual AUROC drop more than
randomly reassigned Crust? We also report $D_T-D_O$, the direct
comparison with leaving Crust omitted, for both probe families.
All three raw, expert, and residual AUROCs remain visible. An expert
code that exposes true Crust is a necessary learning control for this
test; lower residual accessibility alone can also reflect poorer raw
encoding after retraining.

Retraining changes the raw encoder, expert encoder, operator, and
prediction head. These contrasts therefore do not isolate the
projection while holding the representation fixed. The dual baseline
undergoes the same restoration as an encoder-learning control. Its
raw and residual codes differ only by centering, so its zero
raw-to-residual gap is structural. Other baselines were not retrained
for this Crust-only intervention; it does not test superiority of
SCCR's restoration response over theirs.

For diagnosis after retraining, we report gains relative to the same
$\tilde{x}_{45}$ in every condition and relative to the annotations
actually supplied in that condition: $\tilde{x}_{45}$, $(\tilde{x}_{45},C)$,
or $(\tilde{x}_{45},\widetilde C)$. A reduced gain against a stronger
annotation input need not imply more removal from the residual.
Erosion and Scale remain omitted in all these Crust-only comparisons.

\subsection{Uncertainty, external evaluation, and reporting limits}
\label{app:skincon_uncertainty}
We report seeds 909 and 910 separately on the same evaluation images.
Intervals are pointwise 95\% percentile intervals from 1,000 ordinary
image-bootstrap samples, using the same sampled indices for every
term in a contrast. Sampling is with replacement and is neither
class-stratified nor clustered by patient or detected image group.
Single-class draws cannot estimate AUROC and are excluded from that
metric's interval. Bounds use the 2.5th and 97.5th percentiles. The
bootstrap holds encoders, operators, predictor selection, and fitted
parameters fixed; it omits training, model-selection, and unknown
patient-dependence uncertainty, and does not adjust for multiple
comparisons. Two seeds on one cohort are not independent clinical
replications.

The final readout pipelines are locked before their assessment replay,
and DDI is never used to refit them. However, source-test and DDI
results had already been inspected during earlier development.
These later locks do not make either cohort an untouched confirmatory
sample. The comparison, Crust-only intervention, and threshold
follow-up are exploratory. We retain negative gains and intervals
that include zero. A failed decomposition is unscored, rather than
assigned chance performance; this differs from a successfully fitted
zero code or a selected constant predictor.

DDI tests transfer of the complete source-trained pipeline under
different images, diagnosis composition, and annotation relationships.
The observed loss of benefit does not isolate a single cause of
distribution shift. Crust is an existing clinical annotation, and
predictive overlap with it is not evidence of a new clinical mechanism.
These experiments do not establish nonlinear independence, causal
relevance, complete concept removal, or reliable external clinical
benefit. They test added fitted prediction and the response of measured
linear accessibility to changing the expert view.

\subsection{Interpretation of the completed comparisons}
\label{app:skincon_findings}
On Fitzpatrick17k test, the nonlinear residual diagnosis gains of
end-to-end SCCR are $.073$ and $.051$, with both intervals above zero.
Its paired advantages over end-to-end OLS are $.105$ and $.050$,
again with both intervals above zero
(Table~\ref{tab:table-paired-gains-native-source-test}). This finding concerns added
prediction from the residual. SCCR's native diagnosis AUROCs,
$.845$ and $.805$, are below those of the dual baseline and DSN-MMD.
On common dual-baseline embeddings, SCCR's nonlinear gains are
approximately $.000$ and $.018$ with intervals including zero;
the common-embedding comparison does not establish a general gain
advantage over the other decompositions.

With the SCCR omission residual fixed, true Crust attenuates nonlinear
gain more than shuffled Crust by $.020$ and $.009$. Only seed 909's
interval excludes zero (Table~\ref{tab:skincon_fixed_crust}). After Crust-only retraining, the linear
true-minus-shuffled contrasts are $.016$ and $.045$, with both
intervals above zero (Table~\ref{tab:skincon_refit}). The corresponding nonlinear contrasts and
both linear true-minus-omitted contrasts include zero. True-restored
residuals still predict Crust nonlinearly, with AUROCs $.581$ and $.578$.
The supported finding is a modest linear response to correct rather
than shuffled restoration, not complete removal or a consistent
improvement over omission.

VCCA-private's low full-code diagnosis and Crust accessibility limits
conclusions about its private decomposition. DSN learns diagnosis
mainly through its shared code, and its direct linear private gain
in source seed 910 is positive ($.021$ with interval $[.012,.032]$).
We retain this result even though the primary source-selected fallback
does not use that combined predictor. The primary JIVE zero output
follows from the supplied individual rank zero; its result is not a
general failure of JIVE. On DDI, SCCR's primary nonlinear gains are
negative in both seeds ($-.187$ and $-.229$), and correct restoration
does not establish a consistent external improvement. All results
below retain these limits alongside the in-distribution findings.

\subsection{Complete primary Crust and diagnosis results}
\label{app:skincon_complete_results}
Tables in this section show both seeds separately. Unless indicated otherwise, entries are estimates with pointwise 95\% paired image-bootstrap intervals from 1,000 draws, conditional on the fitted pipelines. The same assessment images are used for both seeds. These intervals exclude training and selection uncertainty, do not account for unknown patient dependence, and are not adjusted for multiple comparisons. Both assessment cohorts had been inspected during development; the results are exploratory.

The superscript P denotes a constant-prior predictor selected on source-training cross-validation. F denotes an expert-only fallback, which gives zero added diagnosis gain by construction. Neither selection establishes absence of information in a code. Positive $G=\mathcal L(\tilde{x}_{45})-\mathcal L(\tilde{x}_{45},R)$ means that the residual improves fitted diagnosis prediction relative to the observed 45 annotations. Native predictions, fitted audit predictions, and common-embedding operators are reported separately. Entries are shown as estimate [lower, upper]; F17k denotes the Fitzpatrick17k test set. Repeated values are printed once only where the captions identify the equivalent condition or table. All original precision is retained, including six decimals for the fixed-residual decomposition.

\paragraph{Full, shared, and private codes.} These controls distinguish poor private-code recovery from poor learning by the complete model. Full codes concatenate the raw shared and private codes (64 dimensions); shared and private codes each have 32 dimensions. Both probe families use the same source-training selection procedure.

\begin{table}[H]
\centering
\small
\setlength{\tabcolsep}{2pt}
\caption{Native diagnosis prediction from separately trained models on Fitzpatrick17k test (F17k) and DDI.}
\label{tab:table-native-prediction-source-test}
\label{tab:table-native-prediction-ddi}
\begin{tabular*}{\textwidth}{@{\extracolsep{\fill}}llcccc@{}}
\toprule
Model & Seed & \shortstack{F17k\\AUROC} & \shortstack{F17k\\log loss} & \shortstack{DDI\\AUROC} & \shortstack{DDI\\log loss} \\
\midrule
Dual baseline & 909 & .873 [.824, .915] & .263 [.217, .316] & .559 [.510, .607] & .900 [.821, .975] \\
Dual baseline & 910 & .875 [.823, .918] & .265 [.214, .325] & .592 [.541, .641] & .791 [.720, .863] \\
SCCR ($\tau=.40$) & 909 & .845 [.796, .892] & .318 [.281, .360] & .541 [.489, .588] & .648 [.592, .707] \\
SCCR ($\tau=.40$) & 910 & .805 [.756, .853] & .349 [.315, .387] & .533 [.481, .581] & .646 [.596, .703] \\
OLS & 909 & .858 [.811, .902] & .314 [.282, .350] & .547 [.500, .594] & .689 [.646, .733] \\
OLS & 910 & .823 [.768, .869] & .321 [.275, .379] & .562 [.508, .611] & .714 [.645, .792] \\
Reg. LEACE-style & 909 & .730 [.673, .788] & .359 [.310, .413] & .460 [.407, .513] & .665 [.601, .732] \\
Reg. LEACE-style & 910 & .722 [.665, .782] & .367 [.318, .423] & .473 [.420, .524] & .665 [.603, .732] \\
DSN-MMD & 909 & .897 [.860, .929] & .312 [.272, .354] & .604 [.555, .650] & .700 [.639, .757] \\
DSN-MMD & 910 & .833 [.782, .883] & .314 [.261, .373] & .582 [.534, .630] & .790 [.709, .874] \\
\bottomrule
\end{tabular*}
\par\smallskip
VCCA-private is unsupervised and has no native diagnosis head. Native prediction uses the full trained architecture, not a residual-only probe.
\end{table}

\begin{table}[H]
\centering
\small
\setlength{\tabcolsep}{2pt}
\caption{Residual audits of separately trained models.}
\label{tab:table-native-audit-source-test}
\label{tab:table-native-audit-ddi}
\begin{tabular*}{\textwidth}{@{\extracolsep{\fill}}llcccc@{}}
\toprule
Method & Seed & \shortstack{Crust AUROC$(R)$\\linear} & \shortstack{Crust AUROC$(R)$\\nonlinear} & $G$ linear & $G$ nonlinear \\
\midrule
\multicolumn{6}{@{}l}{\textit{Fitzpatrick17k test}} \\
Dual baseline & 909 & .588 [.529, .650] & .615 [.553, .677] & .031 [-.029, .088] & -.014 [-.098, .061] \\
Dual baseline & 910 & .628 [.569, .689] & .606 [.545, .665] & .040 [-.030, .102] & .001 [-.082, .068] \\
SCCR ($\tau=.40$) & 909 & .635 [.578, .692] & .626 [.571, .681] & .074 [.027, .121] & .073 [.027, .114] \\
SCCR ($\tau=.40$) & 910 & .568 [.505, .631] & .579 [.522, .641] & .042 [-.004, .087] & .051 [.003, .093] \\
OLS & 909 & .604 [.538, .670] & .639 [.579, .703] & -.121 [-.239, -.020] & -.032 [-.112, .041] \\
OLS & 910 & .593 [.528, .661] & .627 [.567, .691] & .014 [-.051, .073] & .001 [-.059, .056] \\
Reg. LEACE-style & 909 & .500$^{\mathrm{P}}$ [.500, .500] & .500$^{\mathrm{P}}$ [.500, .500] & .000$^{\mathrm{F}}$ [.000, .000] & .000$^{\mathrm{F}}$ [.000, .000] \\
Reg. LEACE-style & 910 & .500$^{\mathrm{P}}$ [.500, .500] & .500$^{\mathrm{P}}$ [.500, .500] & .000$^{\mathrm{F}}$ [.000, .000] & .000$^{\mathrm{F}}$ [.000, .000] \\
VCCA-private & 909 & .483 [.422, .536] & .467 [.407, .526] & .000$^{\mathrm{F}}$ [.000, .000] & .000$^{\mathrm{F}}$ [.000, .000] \\
VCCA-private & 910 & .571 [.509, .640] & .479 [.418, .540] & .000$^{\mathrm{F}}$ [.000, .000] & .000$^{\mathrm{F}}$ [.000, .000] \\
DSN-MMD & 909 & .592 [.530, .649] & .531 [.474, .590] & .000$^{\mathrm{F}}$ [.000, .000] & .000$^{\mathrm{F}}$ [.000, .000] \\
DSN-MMD & 910 & .591 [.528, .652] & .570 [.513, .631] & .000$^{\mathrm{F}}$ [.000, .000] & .000$^{\mathrm{F}}$ [.000, .000] \\
\midrule
\multicolumn{6}{@{}l}{\textit{DDI}} \\
Dual baseline & 909 & .594 [.506, .676] & .599 [.516, .682] & -.501 [-.623, -.387] & -.937 [-1.090, -.789] \\
Dual baseline & 910 & .659 [.579, .740] & .605 [.532, .679] & -.465 [-.576, -.353] & -.739 [-.878, -.607] \\
SCCR ($\tau=.40$) & 909 & .748 [.679, .811] & .648 [.566, .726] & -.182 [-.260, -.103] & -.187 [-.266, -.109] \\
SCCR ($\tau=.40$) & 910 & .601 [.520, .676] & .612 [.542, .690] & -.205 [-.283, -.127] & -.229 [-.309, -.147] \\
OLS & 909 & .553 [.479, .623] & .557 [.479, .632] & -.728 [-.874, -.592] & -.642 [-.785, -.515] \\
OLS & 910 & .577 [.493, .661] & .651 [.563, .737] & -.281 [-.370, -.191] & -.279 [-.368, -.194] \\
Reg. LEACE-style & 909 & .500$^{\mathrm{P}}$ [.500, .500] & .500$^{\mathrm{P}}$ [.500, .500] & .000$^{\mathrm{F}}$ [.000, .000] & .000$^{\mathrm{F}}$ [.000, .000] \\
Reg. LEACE-style & 910 & .500$^{\mathrm{P}}$ [.500, .500] & .500$^{\mathrm{P}}$ [.500, .500] & .000$^{\mathrm{F}}$ [.000, .000] & .000$^{\mathrm{F}}$ [.000, .000] \\
VCCA-private & 909 & .482 [.405, .557] & .525 [.454, .609] & .000$^{\mathrm{F}}$ [.000, .000] & .000$^{\mathrm{F}}$ [.000, .000] \\
VCCA-private & 910 & .550 [.467, .623] & .485 [.402, .562] & .000$^{\mathrm{F}}$ [.000, .000] & .000$^{\mathrm{F}}$ [.000, .000] \\
DSN-MMD & 909 & .571 [.492, .649] & .545 [.465, .629] & .000$^{\mathrm{F}}$ [.000, .000] & .000$^{\mathrm{F}}$ [.000, .000] \\
DSN-MMD & 910 & .630 [.549, .702] & .587 [.509, .666] & .000$^{\mathrm{F}}$ [.000, .000] & .000$^{\mathrm{F}}$ [.000, .000] \\
\bottomrule
\end{tabular*}
\par\smallskip
$R$ is the residual, centered raw code for the dual baseline, or private code for VCCA/DSN. Each method learns its own embeddings, so this comparison does not isolate the operator. DSN-MMD is the paired-view adaptation. The selected regularized LEACE-style fits have literal zero residuals.
\end{table}

\begin{table}[H]
\centering
\small
\setlength{\tabcolsep}{2pt}
\caption{Raw-code learning controls and nonlinear audit prediction without expert-only fallback.}
\label{tab:table-native-code-controls-source-test}
\label{tab:table-native-code-controls-ddi}
\begin{tabular*}{\textwidth}{@{\extracolsep{\fill}}llcccc@{}}
\toprule
Method & Seed & \shortstack{Crust AUROC\\raw/full} & \shortstack{Diagnosis AUROC\\raw/full} & \shortstack{Diagnosis AUROC\\residual/private} & $G$ without fallback \\
\midrule
\multicolumn{6}{@{}l}{\textit{Fitzpatrick17k test}} \\
Dual baseline & 909 & .615 [.553, .677] & .865 [.815, .907] & .865 [.815, .907] & -.014 [-.098, .061] \\
Dual baseline & 910 & .606 [.545, .665] & .856 [.806, .903] & .856 [.806, .903] & .001 [-.082, .068] \\
SCCR ($\tau=.40$) & 909 & .620 [.565, .672] & .862 [.813, .902] & .868 [.824, .905] & .073 [.027, .114] \\
SCCR ($\tau=.40$) & 910 & .664 [.603, .722] & .853 [.803, .896] & .831 [.782, .876] & .051 [.003, .093] \\
OLS & 909 & .586 [.524, .643] & .822 [.766, .875] & .847 [.797, .894] & -.032 [-.112, .041] \\
OLS & 910 & .597 [.541, .654] & .842 [.791, .888] & .813 [.760, .863] & .001 [-.059, .056] \\
Reg. LEACE-style & 909 & .540 [.480, .593] & .652 [.595, .708] & .500$^{\mathrm{P}}$ [.500, .500] & .000 [.000, .000] \\
Reg. LEACE-style & 910 & .603 [.545, .658] & .736 [.684, .786] & .500$^{\mathrm{P}}$ [.500, .500] & .000 [.000, .000] \\
VCCA-private & 909 & .458 [.396, .516] & .595 [.526, .656] & .579 [.512, .638] & -.011 [-.027, .003] \\
VCCA-private & 910 & .538 [.477, .603] & .548 [.488, .606] & .500$^{\mathrm{P}}$ [.500, .500] & -.002 [-.009, .005] \\
DSN-MMD & 909 & .561 [.501, .619] & .860 [.814, .904] & .500$^{\mathrm{P}}$ [.500, .500] & -.024 [-.040, -.010] \\
DSN-MMD & 910 & .614 [.557, .671] & .813 [.757, .863] & .643 [.583, .699] & .002 [-.017, .018] \\
\midrule
\multicolumn{6}{@{}l}{\textit{DDI}} \\
Dual baseline & 909 & .599 [.516, .682] & .563 [.513, .610] & .563 [.513, .610] & -.937 [-1.090, -.789] \\
Dual baseline & 910 & .605 [.532, .679] & .572 [.518, .618] & .572 [.518, .618] & -.739 [-.878, -.607] \\
SCCR ($\tau=.40$) & 909 & .646 [.561, .725] & .580 [.529, .631] & .579 [.528, .629] & -.187 [-.266, -.109] \\
SCCR ($\tau=.40$) & 910 & .631 [.544, .710] & .593 [.542, .641] & .567 [.517, .616] & -.229 [-.309, -.147] \\
OLS & 909 & .542 [.464, .626] & .567 [.517, .615] & .576 [.526, .622] & -.642 [-.785, -.515] \\
OLS & 910 & .539 [.450, .633] & .593 [.542, .641] & .590 [.537, .642] & -.279 [-.368, -.194] \\
Reg. LEACE-style & 909 & .485 [.400, .563] & .578 [.527, .630] & .500$^{\mathrm{P}}$ [.500, .500] & .000 [.000, .000] \\
Reg. LEACE-style & 910 & .614 [.530, .695] & .573 [.520, .623] & .500$^{\mathrm{P}}$ [.500, .500] & .000 [.000, .000] \\
VCCA-private & 909 & .516 [.434, .604] & .511 [.457, .566] & .513 [.463, .566] & -.003 [-.020, .012] \\
VCCA-private & 910 & .480 [.401, .554] & .510 [.460, .558] & .500$^{\mathrm{P}}$ [.500, .500] & -.010 [-.020, .000] \\
DSN-MMD & 909 & .623 [.547, .699] & .587 [.536, .635] & .500$^{\mathrm{P}}$ [.500, .500] & .005 [-.011, .024] \\
DSN-MMD & 910 & .618 [.544, .693] & .547 [.496, .601] & .556 [.504, .605] & -.018 [-.041, .003] \\
\bottomrule
\end{tabular*}
\par\smallskip
All columns use nonlinear source-selected readouts. Poor raw/full-code recovery does not establish a private-allocation failure. The last column forces the selected combined predictor, preserving positive or negative performance hidden by fallback. Full VCCA/DSN codes have 64 dimensions, shared/private codes 32; other raw/residual codes have 32. Dual-baseline raw and residual diagnosis results coincide.
\end{table}

\begin{table}[H]
\centering
\setlength{\tabcolsep}{2pt}
\caption{Linear diagnosis gains with the combined predictor forced at assessment.}
\label{tab:table-forced-linear-gains}
\begin{tabular*}{\textwidth}{@{\extracolsep{\fill}}llcc@{}}
\toprule
Method & Seed & Fitzpatrick17k test $G$ & DDI $G$ \\
\midrule
Dual baseline & 909 & .031 [-.029, .088] & -.501 [-.623, -.387] \\
Dual baseline & 910 & .040 [-.030, .102] & -.465 [-.576, -.353] \\
SCCR ($\tau=.40$) & 909 & .074 [.027, .121] & -.182 [-.260, -.103] \\
SCCR ($\tau=.40$) & 910 & .042 [-.004, .087] & -.205 [-.283, -.127] \\
OLS & 909 & -.121 [-.239, -.020] & -.728 [-.874, -.592] \\
OLS & 910 & .014 [-.051, .073] & -.281 [-.370, -.191] \\
Reg. LEACE-style & 909 & .000 [.000, .000] & .000 [.000, .000] \\
Reg. LEACE-style & 910 & .000 [.000, .000] & .000 [.000, .000] \\
VCCA-private & 909 & .003 [-.006, .012] & .024 [.013, .036] \\
VCCA-private & 910 & .001 [-.001, .003] & .008 [.005, .011] \\
DSN-MMD & 909 & .004 [-.003, .011] & .006 [-.004, .017] \\
DSN-MMD & 910 & .021 [.012, .032] & -.002 [-.017, .012] \\
\bottomrule
\end{tabular*}
\par\smallskip
These source-selected linear readouts are assessed without the expert-only fallback. Positive baseline gains remain visible, including DSN-MMD in source seed 910. Residual/private inputs are as in the primary audit tables.
\end{table}

\begin{table}[H]
\centering
\small
\setlength{\tabcolsep}{2pt}
\caption{Residual audits on common frozen dual-baseline embeddings.}
\label{tab:table-common-audit-source-test}
\label{tab:table-common-audit-ddi}
\begin{tabular*}{\textwidth}{@{\extracolsep{\fill}}llcccc@{}}
\toprule
Method & Seed & \shortstack{Crust AUROC$(R)$\\linear} & \shortstack{Crust AUROC$(R)$\\nonlinear} & $G$ linear & $G$ nonlinear \\
\midrule
\multicolumn{6}{@{}l}{\textit{Fitzpatrick17k test}} \\
Dual baseline & 909 & .588 [.529, .650] & .615 [.553, .677] & .031 [-.029, .088] & -.014 [-.098, .061] \\
Dual baseline & 910 & .628 [.569, .689] & .606 [.545, .665] & .040 [-.030, .102] & .001 [-.082, .068] \\
SCCR ($\tau=.40$) & 909 & .554 [.493, .614] & .642 [.584, .698] & -.090 [-.185, -.003] & .000 [-.075, .068] \\
SCCR ($\tau=.40$) & 910 & .624 [.565, .684] & .637 [.574, .691] & .003 [-.034, .038] & .018 [-.040, .068] \\
Empirical LEACE & 909 & .500$^{\mathrm{P}}$ [.500, .500] & .500$^{\mathrm{P}}$ [.500, .500] & .000$^{\mathrm{F}}$ [.000, .000] & .000$^{\mathrm{F}}$ [.000, .000] \\
Empirical LEACE & 910 & .500$^{\mathrm{P}}$ [.500, .500] & .500$^{\mathrm{P}}$ [.500, .500] & .000$^{\mathrm{F}}$ [.000, .000] & .000$^{\mathrm{F}}$ [.000, .000] \\
JIVE & 909 & .500$^{\mathrm{P}}$ [.500, .500] & .500$^{\mathrm{P}}$ [.500, .500] & .000$^{\mathrm{F}}$ [.000, .000] & .000$^{\mathrm{F}}$ [.000, .000] \\
JIVE & 910 & .500$^{\mathrm{P}}$ [.500, .500] & .500$^{\mathrm{P}}$ [.500, .500] & .000$^{\mathrm{F}}$ [.000, .000] & .000$^{\mathrm{F}}$ [.000, .000] \\
AJIVE & 909 & .588 [.523, .651] & .619 [.560, .676] & .018 [-.045, .077] & -.032 [-.118, .048] \\
AJIVE & 910 & .551 [.486, .619] & .592 [.521, .659] & .033 [-.033, .090] & .015 [-.067, .079] \\
D-CCA & 909 & .561 [.497, .623] & .610 [.550, .668] & .010 [-.058, .074] & .005 [-.071, .079] \\
D-CCA & 910 & .603 [.540, .665] & .676 [.616, .728] & .030 [-.044, .092] & -.009 [-.096, .060] \\
\midrule
\multicolumn{6}{@{}l}{\textit{DDI}} \\
Dual baseline & 909 & .594 [.506, .676] & .599 [.516, .682] & -.501 [-.623, -.387] & -.937 [-1.090, -.789] \\
Dual baseline & 910 & .659 [.579, .740] & .605 [.532, .679] & -.465 [-.576, -.353] & -.739 [-.878, -.607] \\
SCCR ($\tau=.40$) & 909 & .514 [.422, .598] & .621 [.538, .707] & -.770 [-.937, -.609] & -.548 [-.664, -.438] \\
SCCR ($\tau=.40$) & 910 & .632 [.552, .713] & .554 [.462, .641] & -.067 [-.125, -.006] & -.240 [-.322, -.163] \\
Empirical LEACE & 909 & .500$^{\mathrm{P}}$ [.500, .500] & .500$^{\mathrm{P}}$ [.500, .500] & .000$^{\mathrm{F}}$ [.000, .000] & .000$^{\mathrm{F}}$ [.000, .000] \\
Empirical LEACE & 910 & .500$^{\mathrm{P}}$ [.500, .500] & .500$^{\mathrm{P}}$ [.500, .500] & .000$^{\mathrm{F}}$ [.000, .000] & .000$^{\mathrm{F}}$ [.000, .000] \\
JIVE & 909 & .500$^{\mathrm{P}}$ [.500, .500] & .500$^{\mathrm{P}}$ [.500, .500] & .000$^{\mathrm{F}}$ [.000, .000] & .000$^{\mathrm{F}}$ [.000, .000] \\
JIVE & 910 & .500$^{\mathrm{P}}$ [.500, .500] & .500$^{\mathrm{P}}$ [.500, .500] & .000$^{\mathrm{F}}$ [.000, .000] & .000$^{\mathrm{F}}$ [.000, .000] \\
AJIVE & 909 & .647 [.563, .729] & .621 [.537, .701] & -.532 [-.659, -.412] & -.746 [-.885, -.613] \\
AJIVE & 910 & .619 [.537, .703] & .569 [.482, .658] & -.463 [-.570, -.354] & -.582 [-.703, -.462] \\
D-CCA & 909 & .567 [.479, .655] & .641 [.555, .729] & -.556 [-.681, -.434] & -.741 [-.885, -.608] \\
D-CCA & 910 & .629 [.551, .712] & .579 [.494, .664] & -.465 [-.576, -.354] & -.772 [-.908, -.629] \\
\bottomrule
\end{tabular*}
\par\smallskip
All operators use the same projected codes and calibration images within each seed. SCCR and matched top-$q$ have identical fitted maps. JIVE and AJIVE use the primary paired individual output; JIVE's zero output follows from its supplied individual rank zero, not a failed fit. Empirical LEACE has a literal zero map.
\end{table}

\begin{table}[H]
\centering
\setlength{\tabcolsep}{2pt}
\caption{Primary common-embedding JIVE and AJIVE rank settings.}
\label{tab:table-operator-settings}
\begin{tabular}{@{}llccc@{}}
\toprule
Method & Seed & Initial / total ranks & Joint rank & Individual ranks \\
\midrule
JIVE & 909 & (1, 1) & 1 & (0, 0) \\
AJIVE & 909 & (1, 1) & 0 & (1, 1) \\
JIVE & 910 & (1, 1) & 1 & (0, 0) \\
AJIVE & 910 & (1, 1) & 0 & (1, 1) \\
\bottomrule
\end{tabular}
\par\smallskip
\begin{minipage}{\textwidth}Ranks are determined or declared on source calibration data, not assessment performance. JIVE's fixed rank budget here allocates no individual dimensions; this explains its constant individual code. This JIVE configuration does not use the original method's permutation rank selector. AJIVE uses native Wedin/random-direction joint selection conditional on the gap-based initial ranks. Both primary outputs are paired-view individual components with a training-only extension. Separate fixed-rank JIVE checks included six nonconverged fits; these remain unscored and are distinct from the primary rank-conditioned zero outputs.\end{minipage}
\end{table}

\begin{table}[H]
\centering
\small
\setlength{\tabcolsep}{2pt}
\caption{Paired SCCR-minus-baseline diagnosis gain differences on common embeddings.}
\label{tab:table-paired-gains-analytic-source-test}
\label{tab:table-paired-gains-analytic-ddi}
\begin{tabular*}{\textwidth}{@{\extracolsep{\fill}}llcccc@{}}
\toprule
Baseline & Seed & \shortstack{F17k\\linear} & \shortstack{F17k\\nonlinear} & \shortstack{DDI\\linear} & \shortstack{DDI\\nonlinear} \\
\midrule
Dual baseline & 909 & -.121 [-.190, -.052] & .014 [-.028, .058] & -.269 [-.420, -.129] & .389 [.291, .483] \\
Dual baseline & 910 & -.037 [-.092, .021] & .018 [-.031, .071] & .398 [.297, .503] & .499 [.396, .596] \\
Empirical LEACE & 909 & -.090 [-.185, -.003] & .000 [-.075, .068] & -.770 [-.937, -.609] & -.548 [-.664, -.438] \\
Empirical LEACE & 910 & .003 [-.034, .038] & .018 [-.040, .068] & -.067 [-.125, -.006] & -.240 [-.322, -.163] \\
JIVE & 909 & -.090 [-.185, -.003] & .000 [-.075, .068] & -.770 [-.937, -.609] & -.548 [-.664, -.438] \\
JIVE & 910 & .003 [-.034, .038] & .018 [-.040, .068] & -.067 [-.125, -.006] & -.240 [-.322, -.163] \\
AJIVE & 909 & -.109 [-.178, -.038] & .032 [-.016, .083] & -.238 [-.396, -.088] & .197 [.101, .285] \\
AJIVE & 910 & -.030 [-.082, .023] & .003 [-.044, .054] & .396 [.295, .497] & .341 [.250, .435] \\
D-CCA & 909 & -.101 [-.167, -.033] & -.005 [-.047, .032] & -.214 [-.359, -.078] & .193 [.113, .273] \\
D-CCA & 910 & -.026 [-.082, .032] & .027 [-.018, .081] & .398 [.295, .504] & .532 [.432, .628] \\
\bottomrule
\end{tabular*}
\par\smallskip
Entries are $G_{\rm SCCR}-G_{\rm baseline}$ using the primary source-selected readouts and fallback. Positive values favor SCCR for fitted diagnosis gain, not nonlinear independence or clinical usefulness. These saved paired intervals are not differences between marginal confidence bounds.
\end{table}

\begin{table}[H]
\centering
\small
\setlength{\tabcolsep}{2pt}
\caption{Paired SCCR-minus-baseline diagnosis gain differences for separately trained models.}
\label{tab:table-paired-gains-native-source-test}
\label{tab:table-paired-gains-native-ddi}
\begin{tabular*}{\textwidth}{@{\extracolsep{\fill}}llcccc@{}}
\toprule
Baseline & Seed & \shortstack{F17k\\linear} & \shortstack{F17k\\nonlinear} & \shortstack{DDI\\linear} & \shortstack{DDI\\nonlinear} \\
\midrule
Dual baseline & 909 & .043 [-.009, .098] & .087 [.022, .157] & .319 [.225, .422] & .751 [.627, .872] \\
Dual baseline & 910 & .002 [-.052, .062] & .050 [-.010, .122] & .260 [.169, .350] & .509 [.407, .622] \\
OLS & 909 & .194 [.106, .291] & .105 [.039, .174] & .547 [.426, .673] & .456 [.355, .572] \\
OLS & 910 & .028 [-.015, .080] & .050 [.005, .103] & .076 [.003, .153] & .050 [-.017, .127] \\
Reg. LEACE-style & 909 & .074 [.027, .121] & .073 [.027, .114] & -.182 [-.260, -.103] & -.187 [-.266, -.109] \\
Reg. LEACE-style & 910 & .042 [-.004, .087] & .051 [.003, .093] & -.205 [-.283, -.127] & -.229 [-.309, -.147] \\
VCCA-private & 909 & .074 [.027, .121] & .073 [.027, .114] & -.182 [-.260, -.103] & -.187 [-.266, -.109] \\
VCCA-private & 910 & .042 [-.004, .087] & .051 [.003, .093] & -.205 [-.283, -.127] & -.229 [-.309, -.147] \\
DSN-MMD & 909 & .074 [.027, .121] & .073 [.027, .114] & -.182 [-.260, -.103] & -.187 [-.266, -.109] \\
DSN-MMD & 910 & .042 [-.004, .087] & .051 [.003, .093] & -.205 [-.283, -.127] & -.229 [-.309, -.147] \\
\bottomrule
\end{tabular*}
\par\smallskip
Entries are $G_{\rm SCCR}-G_{\rm baseline}$ using the primary source-selected readouts and fallback. Positive values favor SCCR for fitted diagnosis gain, not nonlinear independence or clinical usefulness. These saved paired intervals are not differences between marginal confidence bounds.
\end{table}

\begin{table}[H]
\centering
\small
\setlength{\tabcolsep}{2pt}
\caption{Fixed SCCR omission residual: predictive overlap with Crust and the two contributions to its attenuation.}
\label{tab:skincon_fixed_crust}
\label{tab:table-fixed-residual-decomposition}
\begin{tabular*}{\textwidth}{@{\extracolsep{\fill}}llccc@{}}
\toprule
Cohort & Seed / probe & $G(\emptyset)$ & $A_C$ & $A_C-A_{\widetilde C}$ \\
\midrule
\multicolumn{5}{@{}l}{\textit{Predictive overlap: initial gain and attenuation}} \\
F17k & 909 / linear & .074 [.027, .121] & .003 [-.006, .012] & .005 [-.004, .014] \\
F17k & 909 / nonlinear & .073 [.027, .114] & .015 [.003, .027] & .020 [.008, .034] \\
F17k & 910 / linear & .042 [-.004, .087] & .000 [-.006, .008] & .001 [-.006, .008] \\
F17k & 910 / nonlinear & .051 [.003, .093] & .008 [-.002, .018] & .009 [-.001, .020] \\
DDI & 909 / linear & -.182 [-.260, -.103] & -.014 [-.023, -.005] & -.015 [-.024, -.006] \\
DDI & 909 / nonlinear & -.187 [-.266, -.109] & -.003 [-.015, .010] & .001 [-.013, .015] \\
DDI & 910 / linear & -.205 [-.283, -.127] & -.020 [-.028, -.012] & -.028 [-.037, -.018] \\
DDI & 910 / nonlinear & -.229 [-.309, -.147] & -.018 [-.029, -.006] & -.012 [-.023, .000] \\
\midrule
\multicolumn{5}{@{}l}{\textit{Decomposition of true-Crust attenuation (six decimal places)}} \\
Cohort & Seed / probe & Expert improvement & Combined improvement & Attenuation $A_C$ \\
F17k & 909 / linear & .009506 [-.002083, .021911] & .006323 [.001495, .011734] & .003183 [-.005962, .012102] \\
F17k & 909 / nonlinear & .014389 [.002885, .027178] & -.000128 [-.000249, -.000009] & .014516 [.003056, .027328] \\
F17k & 910 / linear & .009506 [-.002083, .021911] & .009029 [.000210, .018324] & .000476 [-.005941, .007834] \\
F17k & 910 / nonlinear & .014389 [.002885, .027178] & .006694 [.001788, .012030] & .007694 [-.001942, .018415] \\
DDI & 909 / linear & -.007896 [-.018858, .003870] & .006012 [.000220, .012379] & -.013908 [-.022556, -.004872] \\
DDI & 909 / nonlinear & -.004249 [-.016360, .008763] & -.001039 [-.001232, -.000845] & -.003210 [-.015225, .009766] \\
DDI & 910 / linear & -.007896 [-.018858, .003870] & .012132 [.002430, .022706] & -.020027 [-.028180, -.011974] \\
DDI & 910 / nonlinear & -.004249 [-.016360, .008763] & .013635 [.007713, .019969] & -.017885 [-.029415, -.005987] \\
\bottomrule
\end{tabular*}
\par\smallskip
No encoder or SCCR map changes. $A_C=G(\emptyset)-G(C)$ compares true Crust with omission; $A_C-A_{\widetilde C}$ compares true with shuffled Crust. Attenuation equals expert-only improvement minus combined-predictor improvement in log loss. Six decimals retain small nonzero changes. This measures fitted predictive overlap, not mediation or concept removal; negative DDI gains do not constitute useful residual signal explained by Crust.
\end{table}

\begin{table}[H]
\centering
\small
\setlength{\tabcolsep}{2pt}
\caption{Crust-only retraining: restoration contrasts at both training thresholds.}
\label{tab:skincon_refit}
\label{tab:skincon_tau_native_refit_source_test}
\label{tab:skincon_tau_native_refit_ddi}
\begin{tabular*}{\textwidth}{@{\extracolsep{\fill}}llcccc@{}}
\toprule
$\tau$ & Seed / probe & \shortstack{F17k\\$D_T-D_S$} & \shortstack{F17k\\$D_T-D_O$} & \shortstack{DDI\\$D_T-D_S$} & \shortstack{DDI\\$D_T-D_O$} \\
\midrule
.20 & 909 / linear & .064 [.002, .133] & .055 [-.032, .138] & .070 [-.039, .170] & .028 [-.088, .144] \\
.20 & 909 / nonlinear & -.012 [-.109, .093] & -.074 [-.169, .026] & -.025 [-.157, .108] & .044 [-.098, .194] \\
.20 & 910 / linear & .085 [.012, .161] & .089 [.013, .164] & .018 [-.069, .109] & -.024 [-.108, .075] \\
.20 & 910 / nonlinear & .001 [-.091, .093] & .018 [-.077, .114] & -.077 [-.188, .048] & .000 [-.100, .112] \\
.40 & 909 / linear & .016 [.002, .029] & .008 [-.023, .039] & .003 [-.013, .019] & .025 [-.005, .053] \\
.40 & 909 / nonlinear & .045 [-.016, .111] & .009 [-.062, .083] & -.013 [-.101, .087] & .017 [-.070, .111] \\
.40 & 910 / linear & .045 [.003, .088] & -.030 [-.082, .022] & .025 [-.034, .083] & .006 [-.060, .073] \\
.40 & 910 / nonlinear & .018 [-.059, .090] & -.076 [-.165, .011] & .093 [.004, .183] & .018 [-.102, .154] \\
\bottomrule
\end{tabular*}
\par\smallskip
$D_v$ is the raw-minus-residual Crust AUROC gap; T, S and O denote true, shuffled and omitted Crust. True versus shuffled is the primary contrast; both comparisons are retained. Erosion and Scale remain omitted. Retraining changes embeddings as well as the operator. The $.40$ primary results appear only here; $.20$ is an exploratory sensitivity check. Dual baseline is an encoder-learning control, not an independent suppression contrast.
\end{table}

\begin{table}[H]
\centering
\setlength{\tabcolsep}{2pt}
\caption{SCCR ($\tau=.40$): Crust accessibility after retraining.}
\label{tab:table-accessibility-sccr-tau040-source-test}
\label{tab:table-accessibility-sccr-tau040-ddi}
\begin{tabular*}{\textwidth}{@{\extracolsep{\fill}}llccc@{}}
\toprule
Seed / condition & Probe & Raw AUROC & Residual AUROC & Expert-code AUROC \\
\midrule
\multicolumn{5}{@{}l}{\textit{Fitzpatrick17k test}} \\
909 / Omitted & linear & .640 [.583, .695] & .635 [.578, .692] & .662 [.597, .725] \\
909 / Omitted & nonlinear & .620 [.565, .672] & .626 [.571, .681] & .674 [.610, .736] \\
909 / True Crust & linear & .605 [.546, .668] & .592 [.530, .657] & 1.000 [1.000, 1.000] \\
909 / True Crust & nonlinear & .583 [.524, .643] & .581 [.520, .646] & .999 [.999, 1.000] \\
909 / Shuffled Crust & linear & .654 [.593, .712] & .658 [.597, .717] & .671 [.606, .732] \\
909 / Shuffled Crust & nonlinear & .629 [.566, .689] & .672 [.612, .730] & .701 [.639, .755] \\
910 / Omitted & linear & .667 [.607, .727] & .568 [.505, .631] & .695 [.636, .753] \\
910 / Omitted & nonlinear & .664 [.603, .722] & .579 [.522, .641] & .703 [.647, .764] \\
910 / True Crust & linear & .612 [.550, .673] & .544 [.482, .610] & 1.000 [1.000, 1.000] \\
910 / True Crust & nonlinear & .587 [.528, .643] & .578 [.514, .649] & .999 [.999, 1.000] \\
910 / Shuffled Crust & linear & .643 [.580, .704] & .620 [.558, .683] & .691 [.629, .750] \\
910 / Shuffled Crust & nonlinear & .594 [.534, .655] & .604 [.548, .659] & .704 [.642, .766] \\
\midrule
\multicolumn{5}{@{}l}{\textit{DDI}} \\
909 / Omitted & linear & .727 [.651, .795] & .748 [.679, .811] & .742 [.669, .813] \\
909 / Omitted & nonlinear & .646 [.561, .725] & .648 [.566, .726] & .731 [.659, .798] \\
909 / True Crust & linear & .597 [.518, .669] & .593 [.515, .667] & 1.000 [1.000, 1.000] \\
909 / True Crust & nonlinear & .552 [.462, .631] & .537 [.448, .611] & .999 [.998, 1.000] \\
909 / Shuffled Crust & linear & .622 [.533, .702] & .620 [.533, .698] & .722 [.646, .795] \\
909 / Shuffled Crust & nonlinear & .672 [.593, .748] & .643 [.570, .727] & .738 [.665, .812] \\
910 / Omitted & linear & .652 [.578, .728] & .601 [.520, .676] & .755 [.685, .823] \\
910 / Omitted & nonlinear & .631 [.544, .710] & .612 [.542, .690] & .772 [.707, .838] \\
910 / True Crust & linear & .626 [.551, .695] & .569 [.486, .650] & 1.000 [1.000, 1.000] \\
910 / True Crust & nonlinear & .628 [.550, .707] & .592 [.516, .673] & 1.000 [.999, 1.000] \\
910 / Shuffled Crust & linear & .677 [.597, .751] & .645 [.563, .723] & .771 [.704, .836] \\
910 / Shuffled Crust & nonlinear & .589 [.509, .672] & .645 [.573, .721] & .734 [.671, .798] \\
\bottomrule
\end{tabular*}
\par\smallskip
Separate source-trained readouts probe each code. True Crust must be accessible in the expert code for removal to be a meaningful test. Raw changes reflect encoder learning; neither probe family establishes independence.
\end{table}

\begin{table}[H]
\centering
\setlength{\tabcolsep}{2pt}
\caption{Dual baseline: Crust accessibility after retraining.}
\label{tab:table-accessibility-identity-source-test}
\label{tab:table-accessibility-identity-ddi}
\begin{tabular*}{\textwidth}{@{\extracolsep{\fill}}llcc@{}}
\toprule
Seed / condition & Probe & Raw / residual AUROC & Expert-code AUROC \\
\midrule
\multicolumn{4}{@{}l}{\textit{Fitzpatrick17k test}} \\
909 / Omitted & linear & .588 [.529, .650] & .674 [.606, .736] \\
909 / Omitted & nonlinear & .615 [.553, .677] & .685 [.624, .743] \\
909 / True Crust & linear & .570 [.503, .634] & 1.000 [.999, 1.000] \\
909 / True Crust & nonlinear & .591 [.534, .651] & .999 [.997, 1.000] \\
909 / Shuffled Crust & linear & .644 [.587, .710] & .676 [.611, .737] \\
909 / Shuffled Crust & nonlinear & .627 [.567, .691] & .686 [.622, .746] \\
910 / Omitted & linear & .628 [.569, .689] & .697 [.637, .756] \\
910 / Omitted & nonlinear & .606 [.545, .665] & .699 [.641, .762] \\
910 / True Crust & linear & .641 [.588, .698] & 1.000 [1.000, 1.000] \\
910 / True Crust & nonlinear & .614 [.557, .673] & 1.000 [.999, 1.000] \\
910 / Shuffled Crust & linear & .623 [.557, .682] & .715 [.659, .774] \\
910 / Shuffled Crust & nonlinear & .577 [.510, .640] & .717 [.660, .774] \\
\midrule
\multicolumn{4}{@{}l}{\textit{DDI}} \\
909 / Omitted & linear & .594 [.506, .676] & .757 [.681, .828] \\
909 / Omitted & nonlinear & .599 [.516, .682] & .726 [.651, .798] \\
909 / True Crust & linear & .614 [.541, .693] & 1.000 [1.000, 1.000] \\
909 / True Crust & nonlinear & .634 [.565, .709] & 1.000 [.999, 1.000] \\
909 / Shuffled Crust & linear & .601 [.515, .684] & .742 [.670, .811] \\
909 / Shuffled Crust & nonlinear & .683 [.612, .753] & .741 [.669, .807] \\
910 / Omitted & linear & .659 [.579, .740] & .755 [.684, .823] \\
910 / Omitted & nonlinear & .605 [.532, .679] & .761 [.694, .824] \\
910 / True Crust & linear & .629 [.549, .704] & 1.000 [1.000, 1.000] \\
910 / True Crust & nonlinear & .620 [.535, .702] & 1.000 [.999, 1.000] \\
910 / Shuffled Crust & linear & .616 [.530, .697] & .766 [.694, .833] \\
910 / Shuffled Crust & nonlinear & .625 [.552, .705] & .759 [.691, .823] \\
\bottomrule
\end{tabular*}
\par\smallskip
Separate source-trained readouts probe each code. True Crust must be accessible in the expert code for removal to be a meaningful test. Raw changes reflect encoder learning; neither probe family establishes independence. Raw and residual codes differ only by centering, so their identical measurements are printed once.
\end{table}

\begin{table}[H]
\centering
\setlength{\tabcolsep}{2pt}
\caption{Diagnosis after SCCR Crust-only retraining ($\tau=.40$).}
\label{tab:table-retraining-diagnosis-source-test}
\label{tab:table-retraining-diagnosis-ddi}
\begin{tabular*}{\textwidth}{@{\extracolsep{\fill}}llccc@{}}
\toprule
Seed / condition & Probe & Native AUROC & $G$: common $\tilde{x}_{45}$ & $G$: matched expert \\
\midrule
\multicolumn{5}{@{}l}{\textit{Fitzpatrick17k test}} \\
909 / Omitted & linear & .845 [.796, .892] & .074 [.027, .121] & .074 [.027, .121] \\
909 / Omitted & nonlinear &  & .073 [.027, .114] & .073 [.027, .114] \\
909 / True Crust & linear & .842 [.795, .884] & .044 [-.006, .094] & .041 [-.009, .092] \\
909 / True Crust & nonlinear &  & .029 [-.023, .078] & .021 [-.028, .068] \\
909 / Shuffled Crust & linear & .872 [.825, .911] & .010 [-.058, .073] & .002 [-.065, .065] \\
909 / Shuffled Crust & nonlinear &  & .024 [-.039, .083] & .030 [-.035, .089] \\
910 / Omitted & linear & .805 [.756, .853] & .042 [-.004, .087] & .042 [-.004, .087] \\
910 / Omitted & nonlinear &  & .051 [.003, .093] & .051 [.003, .093] \\
910 / True Crust & linear & .820 [.771, .865] & -.019 [-.081, .032] & -.019 [-.078, .033] \\
910 / True Crust & nonlinear &  & .023 [-.026, .066] & .016 [-.034, .059] \\
910 / Shuffled Crust & linear & .838 [.791, .878] & .055 [.010, .095] & .055 [.010, .095] \\
910 / Shuffled Crust & nonlinear &  & .048 [-.002, .089] & .048 [-.001, .089] \\
\midrule
\multicolumn{5}{@{}l}{\textit{DDI}} \\
909 / Omitted & linear & .541 [.489, .588] & -.182 [-.260, -.103] & -.182 [-.260, -.103] \\
909 / Omitted & nonlinear &  & -.187 [-.266, -.109] & -.187 [-.266, -.109] \\
909 / True Crust & linear & .608 [.560, .654] & -.164 [-.246, -.085] & -.169 [-.249, -.091] \\
909 / True Crust & nonlinear &  & -.238 [-.325, -.154] & -.223 [-.303, -.142] \\
909 / Shuffled Crust & linear & .599 [.550, .648] & -.345 [-.452, -.242] & -.355 [-.464, -.251] \\
909 / Shuffled Crust & nonlinear &  & -.361 [-.470, -.254] & -.362 [-.470, -.253] \\
910 / Omitted & linear & .533 [.481, .581] & -.205 [-.283, -.127] & -.205 [-.283, -.127] \\
910 / Omitted & nonlinear &  & -.229 [-.309, -.147] & -.229 [-.309, -.147] \\
910 / True Crust & linear & .534 [.481, .582] & -.233 [-.313, -.152] & -.234 [-.313, -.156] \\
910 / True Crust & nonlinear &  & -.187 [-.258, -.120] & -.173 [-.241, -.105] \\
910 / Shuffled Crust & linear & .542 [.489, .592] & -.161 [-.235, -.092] & -.165 [-.238, -.095] \\
910 / Shuffled Crust & nonlinear &  & -.203 [-.285, -.127] & -.204 [-.287, -.128] \\
\bottomrule
\end{tabular*}
\par\smallskip
Native AUROC is printed once per condition because there is one trained head. The common input is the same 45 annotations; matched inputs also include true or shuffled Crust when supplied to the network. Gains use the selected fallback. A smaller gain with matched inputs can reflect better expert-only prediction, not greater representation suppression.
\end{table}

\begin{table}[H]
\centering
\small
\setlength{\tabcolsep}{2pt}
\caption{Fitzpatrick17k test: recovery from full, shared, and private codes.}
\label{tab:skincon_codes_source_test_linear}
\label{tab:skincon_codes_source_test_nonlinear}
\begin{tabular*}{\textwidth}{@{\extracolsep{\fill}}llcccc@{}}
\toprule
Method / seed & Code & \shortstack{Crust AUROC\\linear} & \shortstack{Diagnosis AUROC\\linear} & \shortstack{Crust AUROC\\nonlinear} & \shortstack{Diagnosis AUROC\\nonlinear} \\
\midrule
VCCA-private / 909 & Full & .506 [.446, .563] & .590 [.521, .653] & .458 [.396, .516] & .595 [.526, .656] \\
VCCA-private / 909 & Shared & .564 [.505, .631] & .500$^{\mathrm{P}}$ [.500, .500] & .434 [.382, .496] & .495 [.432, .556] \\
VCCA-private / 909 & Private & .483 [.422, .536] & .581 [.513, .645] & .467 [.407, .526] & .579 [.512, .638] \\
VCCA-private / 910 & Full & .526 [.464, .590] & .649 [.587, .714] & .538 [.477, .603] & .548 [.488, .606] \\
VCCA-private / 910 & Shared & .522 [.462, .584] & .648 [.588, .714] & .538 [.473, .600] & .558 [.496, .618] \\
VCCA-private / 910 & Private & .571 [.509, .640] & .500$^{\mathrm{P}}$ [.500, .500] & .479 [.418, .540] & .500$^{\mathrm{P}}$ [.500, .500] \\
DSN-MMD / 909 & Full & .616 [.552, .676] & .875 [.832, .914] & .561 [.501, .619] & .860 [.814, .904] \\
DSN-MMD / 909 & Shared & .600 [.541, .660] & .897 [.860, .929] & .578 [.515, .632] & .870 [.824, .908] \\
DSN-MMD / 909 & Private & .592 [.530, .649] & .532 [.474, .587] & .531 [.474, .590] & .500$^{\mathrm{P}}$ [.500, .500] \\
DSN-MMD / 910 & Full & .669 [.610, .725] & .806 [.742, .862] & .614 [.557, .671] & .813 [.757, .863] \\
DSN-MMD / 910 & Shared & .653 [.596, .710] & .824 [.768, .874] & .617 [.563, .675] & .810 [.751, .863] \\
DSN-MMD / 910 & Private & .591 [.528, .652] & .688 [.628, .749] & .570 [.513, .631] & .643 [.583, .699] \\
\bottomrule
\end{tabular*}
\par\smallskip
Full codes concatenate shared and private codes (64 dimensions); each component has 32. These controls distinguish poor private-code recovery from poor accessibility in the whole learned representation.
\end{table}

\begin{table}[H]
\centering
\small
\setlength{\tabcolsep}{2pt}
\caption{DDI: recovery from full, shared, and private codes.}
\label{tab:skincon_codes_ddi_linear}
\label{tab:skincon_codes_ddi_nonlinear}
\begin{tabular*}{\textwidth}{@{\extracolsep{\fill}}llcccc@{}}
\toprule
Method / seed & Code & \shortstack{Crust AUROC\\linear} & \shortstack{Diagnosis AUROC\\linear} & \shortstack{Crust AUROC\\nonlinear} & \shortstack{Diagnosis AUROC\\nonlinear} \\
\midrule
VCCA-private / 909 & Full & .509 [.422, .594] & .560 [.510, .610] & .516 [.434, .604] & .511 [.457, .566] \\
VCCA-private / 909 & Shared & .540 [.458, .615] & .500$^{\mathrm{P}}$ [.500, .500] & .456 [.370, .539] & .509 [.457, .561] \\
VCCA-private / 909 & Private & .482 [.405, .557] & .557 [.506, .606] & .525 [.454, .609] & .513 [.463, .566] \\
VCCA-private / 910 & Full & .591 [.508, .671] & .555 [.504, .605] & .480 [.401, .554] & .510 [.460, .558] \\
VCCA-private / 910 & Shared & .568 [.492, .647] & .559 [.508, .612] & .456 [.371, .536] & .515 [.467, .564] \\
VCCA-private / 910 & Private & .550 [.467, .623] & .500$^{\mathrm{P}}$ [.500, .500] & .485 [.402, .562] & .500$^{\mathrm{P}}$ [.500, .500] \\
DSN-MMD / 909 & Full & .669 [.591, .744] & .601 [.553, .647] & .623 [.547, .699] & .587 [.536, .635] \\
DSN-MMD / 909 & Shared & .658 [.583, .726] & .603 [.553, .649] & .627 [.558, .698] & .591 [.540, .638] \\
DSN-MMD / 909 & Private & .571 [.492, .649] & .495 [.445, .546] & .545 [.465, .629] & .500$^{\mathrm{P}}$ [.500, .500] \\
DSN-MMD / 910 & Full & .693 [.625, .766] & .568 [.521, .616] & .618 [.544, .693] & .547 [.496, .601] \\
DSN-MMD / 910 & Shared & .657 [.580, .741] & .581 [.534, .628] & .595 [.517, .669] & .547 [.498, .601] \\
DSN-MMD / 910 & Private & .630 [.549, .702] & .572 [.518, .620] & .587 [.509, .666] & .556 [.504, .605] \\
\bottomrule
\end{tabular*}
\par\smallskip
Full codes concatenate shared and private codes (64 dimensions); each component has 32. These controls distinguish poor private-code recovery from poor accessibility in the whole learned representation.
\end{table}

\subsection{Exploratory threshold sensitivity}
\label{app:skincon_tau_sensitivity}
This follow-up reports all three evaluated frozen-encoder thresholds, $\tau\in\{.20,.40,.50\}$, and both end-to-end SCCR thresholds, $\tau\in\{.20,.40\}$. Both seeds are retained. Frozen-encoder analyses refit the removal operator and audit predictors on the saved dual-baseline or SCCR $.40$ embeddings; they do not train a new neural prediction head. The end-to-end comparison adds six networks at $\tau=.20$ (two seeds and three Crust conditions) and reuses the $.40$ models. No threshold is selected from these results.

The tables retain saved pointwise 95\% paired image-bootstrap intervals and the same P/F flags as above. These cohorts had already been inspected. Frozen omission results measure sensitivity to the removal threshold on fixed encoders. Native restoration contrasts compare separately trained true-Crust, shuffled-Crust, and omission models, so their embeddings also differ. Saved paired restoration intervals are available for the native models; no corresponding paired restoration intervals are inferred from the marginal frozen-encoder intervals.
The frozen cutoffs were chosen from existing training-calibration
spectra to examine stronger and weaker removal around $.40$.
Each encoder pair has one calibration spectrum reused across these
cutoffs. The full saved follow-up covers 12 existing encoder pairs
(two model families, three Crust conditions, and two seeds), yielding
36 frozen maps, together with the native reference models and six
new $.20$ fits. The tables below report all cutoffs for the omission
encoders and both native training thresholds for the primary Crust
restoration contrasts. Secondary-concept probes and other frozen
restoration measurements remain in the saved analysis. No new
$.50$ network was trained.

Fresh $.20$ models use the same initialization, architecture,
augmentation, splits, optimizer, and early-stopping rule as their
$.40$ counterparts. Table~\ref{tab:skincon_tau_training} records all
six completed runs. The frozen-encoder results do not use the old
native head after changing its removal map: they use newly fitted
audit predictors. This separates a change to the map on fixed
embeddings from training a network under a different threshold.

At $.20$, the linear true-minus-shuffled restoration contrast remains
positive in both source seeds, but nonlinear contrasts still include
zero. The true-minus-omitted linear contrast excludes zero only in
seed 910. Omitted-model nonlinear diagnosis gains are $-.002$ and
$.033$ on Fitzpatrick17k, compared with $.073$ and $.051$ at $.40$.
DDI nonlinear gains remain negative at both training thresholds.
Lowering the threshold therefore does not give a uniform improvement
in suppression and retained prediction. This follow-up measures
sensitivity; it does not provide an independent threshold-selection
rule.

\begin{table}[H]
\centering
\small
\setlength{\tabcolsep}{2pt}
\caption{Threshold sensitivity on frozen omission encoders: additional settings around the primary $.40$ fit.}
\label{tab:skincon_tau_frozen_source_test}
\label{tab:skincon_tau_frozen_ddi}
\begin{tabular*}{\textwidth}{@{\extracolsep{\fill}}llcccc@{}}
\toprule
Encoder & Seed / $\tau$ / $q$ & \shortstack{Crust AUROC\\linear} & \shortstack{Crust AUROC\\nonlinear} & $G$ linear & $G$ nonlinear \\
\midrule
\multicolumn{6}{@{}l}{\textit{Fitzpatrick17k test}} \\
Dual baseline & 909 / .20 / 5 & .541 [.483, .603] & .566 [.510, .624] & -.050 [-.111, .006] & .020 [-.034, .075] \\
Dual baseline & 910 / .20 / 8 & .539 [.478, .596] & .548 [.488, .611] & .002 [-.022, .024] & .010 [-.030, .051] \\
Dual baseline & 910 / .50 / 0 & .628 [.569, .689] & .606 [.545, .665] & .040 [-.030, .102] & .001 [-.082, .068] \\
SCCR ($.40$) & 909 / .20 / 10 & .512 [.457, .569] & .579 [.516, .637] & -.028 [-.071, .017] & -.016 [-.057, .025] \\
SCCR ($.40$) & 909 / .50 / 1 & .651 [.596, .706] & .649 [.593, .701] & .074 [.028, .122] & .075 [.030, .117] \\
SCCR ($.40$) & 910 / .20 / 10 & .525 [.466, .582] & .575 [.515, .638] & .029 [.001, .055] & .021 [-.007, .048] \\
SCCR ($.40$) & 910 / .50 / 1 & .641 [.580, .699] & .603 [.541, .665] & .050 [.004, .097] & .059 [.014, .101] \\
\midrule
\multicolumn{6}{@{}l}{\textit{DDI}} \\
Dual baseline & 909 / .20 / 5 & .485 [.407, .567] & .574 [.490, .660] & -.354 [-.458, -.268] & -.300 [-.393, -.220] \\
Dual baseline & 910 / .20 / 8 & .496 [.414, .575] & .548 [.476, .628] & -.032 [-.068, .006] & -.110 [-.167, -.055] \\
Dual baseline & 910 / .50 / 0 & .659 [.579, .740] & .605 [.532, .679] & -.465 [-.576, -.353] & -.739 [-.878, -.607] \\
SCCR ($.40$) & 909 / .20 / 10 & .663 [.582, .736] & .597 [.513, .685] & -.174 [-.241, -.104] & -.126 [-.179, -.069] \\
SCCR ($.40$) & 909 / .50 / 1 & .744 [.674, .810] & .634 [.551, .719] & -.184 [-.263, -.106] & -.170 [-.248, -.096] \\
SCCR ($.40$) & 910 / .20 / 10 & .596 [.517, .674] & .491 [.410, .574] & -.027 [-.073, .018] & -.048 [-.093, -.007] \\
SCCR ($.40$) & 910 / .50 / 1 & .619 [.546, .692] & .630 [.540, .719] & -.194 [-.276, -.113] & -.155 [-.231, -.081] \\
\bottomrule
\end{tabular*}
\par\smallskip
$q$ is removed rank. The $.40$ rows are reported in Tables~\ref{tab:table-common-audit-source-test} and~\ref{tab:table-native-audit-source-test}: dual-baseline ranks are 1 and 2, and SCCR-encoder ranks are 2 and 3, for seeds 909 and 910. Dual seed 909 at $.50$ has $q=1$ and exactly the same results as $.40$ in both cohorts. The encoder and 45 annotations stay fixed; new maps and audit predictors are fitted, without reusing the old native head.
\end{table}

\begin{table}[H]
\centering
\setlength{\tabcolsep}{2pt}
\caption{Additional omission models trained end to end at $\tau=.20$.}
\label{tab:skincon_tau_native_diagnosis}
\begin{tabular*}{\textwidth}{@{\extracolsep{\fill}}llcccc@{}}
\toprule
Cohort & Seed & $\tau$ & Native AUROC & $G$ linear & $G$ nonlinear \\
\midrule
Fitzpatrick17k & 909 & .20 & .708 [.649, .767] & -.012 [-.050, .021] & -.002 [-.040, .031] \\
Fitzpatrick17k & 910 & .20 & .785 [.729, .836] & .016 [-.022, .053] & .033 [.003, .064] \\
DDI & 909 & .20 & .509 [.457, .560] & -.128 [-.181, -.073] & -.103 [-.152, -.051] \\
DDI & 910 & .20 & .553 [.501, .598] & -.032 [-.080, .013] & -.051 [-.100, -.004] \\
\bottomrule
\end{tabular*}
\par\smallskip
Native AUROC uses the trained head; $G$ uses separately fitted audit predictors. The $.40$ comparison is in Tables~\ref{tab:table-native-prediction-source-test} and~\ref{tab:table-native-audit-source-test}, avoiding duplicate measurements. Better native prediction need not imply greater residual gain; negative DDI gains remain visible.
\end{table}
\clearpage

\section{Experimental Details: Treatment Dropout}
\label{app:psyctraj-details}

\subsection{Dataset and prediction task}
\label{app:psyctraj-data}
\begin{table*}[t]
\centering
\scriptsize
\setlength{\tabcolsep}{2.7pt}
\begin{tabular}{lccccccc}
\toprule
& & \multicolumn{2}{c}{Site B}
& \multicolumn{2}{c}{Site C}
& \multicolumn{2}{c}{Site D} \\
Method & $n$
& AUROC & AUPRC
& AUROC & AUPRC
& AUROC & AUPRC \\
\midrule
Expert-only & 5
& $.565\pm.001$ & $.415\pm.001$
& $.574\pm.002$ & $.476\pm.001$
& $.555\pm.002$ & $.359\pm.002$ \\
Raw-only & 5
& $.619\pm.001$ & $.475\pm.001$
& $.611\pm.003$ & $.509\pm.002$
& $.585\pm.006$ & $.389\pm.002$ \\
Identity & 5
& $.615\pm.003$ & $.471\pm.004$
& $.608\pm.005$ & $.508\pm.004$
& $.583\pm.008$ & $.386\pm.008$ \\
OLS & 5
& $.613\pm.004$ & $.470\pm.004$
& $.605\pm.006$ & $.505\pm.004$
& $.573\pm.012$ & $.380\pm.007$ \\
LEACE & 4
& $.564\pm.001$ & $.414\pm.001$
& $.574\pm.002$ & $.468\pm.001$
& $.554\pm.004$ & $.358\pm.006$ \\
\midrule
SCCR, $\tau=.05$ & 5
& $.593\pm.004$ & $.446\pm.006$
& $.592\pm.007$ & $.492\pm.005$
& $.576\pm.005$ & $.381\pm.006$ \\
SCCR, $\tau=.40$ & 5
& $.613\pm.002$ & $.469\pm.002$
& $.605\pm.002$ & $.505\pm.002$
& $.578\pm.008$ & $.384\pm.004$ \\
SCCR, $\tau=.70$ & 5
& $.616\pm.002$ & $.472\pm.003$
& $.606\pm.004$ & $.506\pm.004$
& $.580\pm.007$ & $.386\pm.004$ \\
\bottomrule
\end{tabular}
\caption{Dropout prediction on held-out sites.
Values are mean $\pm$ SD across $n$ completed seeds;
AUPRC is computed as average precision.
Identity uses the dual architecture without removal.}
\label{tab:psych_evidence_audit_main}
\end{table*}
We study binary antidepressant treatment dropout in 155,552 patients from Sites A--D. Each patient is represented by age, demographic variables, and counts of CPT4, ICD10, and RxNorm codes. Race, ethnicity, and site indicators are excluded. Both branches use the same fixed panel of 256 codes selected from Site A training data. The raw branch receives all codes in this panel; the expert branch receives its five psychiatric codes and the same permitted age and demographic variables. Writing $m$ for the psychiatric-code mask, the paired inputs are
\[
x=[x_{\mathrm{age}},x_{\mathrm{demo}},x_{\mathrm{code}}],
\qquad
x_{\psi}=[x_{\mathrm{age}},x_{\mathrm{demo}},m\odot x_{\mathrm{code}}].
\]
For the omission audit, we additionally remove psychotherapy code CPT~90806 from the expert input while retaining it in the raw input.

\subsection{Train, validation, and test splits}
\label{app:psyctraj-splits}

We split Site A patients into 71,241 training and 7,916 validation examples, stratified by the label with seed 2025. Sites B, C, and D provide 55,642, 12,696, and 8,057 test patients, respectively. These splits are identical across methods and training seeds. Within Site A training, a fixed random partition with seed 421337 assigns 35,620 patients to operator calibration and 35,621 to audit-probe fitting. Both partitions are used to train the neural encoders. Audit probes are fitted on their partition only after freezing the network; Site A validation selects the checkpoint and probe settings.

\subsection{Count preprocessing and TF-IDF}
\label{app:psyctraj-tfidf}

We compute smoothed inverse document frequencies from Site A training counts only:
\[
\mathrm{df}_j=\sum_{i\in\mathcal D_{\mathrm{train}}}\mathbb{1}[x_{ij}>0],
\qquad
\mathrm{idf}_j=\log\frac{|\mathcal D_{\mathrm{train}}|+1}{\mathrm{df}_j+1}+1.
\]
The fixed panel contains the 256 allowed codes with the largest training sums $\sum_i x_{ij}\mathrm{idf}_j$; no additional psychiatric codes are forced into the panel. We apply $x_{ij}\leftarrow x_{ij}\mathrm{idf}_j$ to both views using the same saved weights. The expert mask is applied before this transformation. We use no additional $\ell_2$ normalization of the count block, and leave the permitted age and demographic values unchanged.

\subsection{Dual-head tabular-transformer architecture}
\label{app:psyctraj-architecture}

The two encoders have the same architecture and independent parameters. Each uses the fixed training-selected code panel, with absent codes masked. For a code $j$ with weighted count $v_j$, its token is $t_j=e_j+\phi(v_j)$, where $e_j$ is a learned embedding and $\phi$ is a two-layer MLP with SiLU activation. A learned CLS token and an age/demographic token are added. Three pre-layer-normalized transformer blocks use token width 128, four attention heads, feed-forward width 512, and dropout .1. Token dropout is disabled. The final CLS token is projected to a 256-dimensional encoder output; no hybrid MLP branch is used.

Each encoder output is independently $\ell_2$-normalized and passed through a $256\to64\to64$ projector with SiLU and dropout .2 between the linear layers:
\[
z_r=\Pi_r(\mathrm{normalize}(E_r(x))),
\qquad
z_{\psi}=\Pi_{\psi}(\mathrm{normalize}(E_{\psi}(x_{\psi}))).
\]
Raw-only and expert-only baselines use a single encoder with a linear output head. The dual baseline, OLS, LEACE, and SCCR use the same two-branch architecture and are trained end to end.

\subsection{Shared canonical component removal}
\label{app:psyctraj-sccr}

SCCR uses the branchwise OAS regularization defined in the main text. For centered 64-dimensional projected embeddings, let $S_{rr},S_{\psi\psi},S_{r\psi}$ be the empirical covariances. With $m_j=\operatorname{tr}(S_{jj})/64$, the OAS coefficient $\rho_j$ gives $\lambda_j=\rho_jm_j/(1-\rho_j)$, and $R_j=S_{jj}+\lambda_jI$. Using Cholesky factors $R_j=L_jL_j^\top$, we compute
\[
M=L_r^{-1}S_{r\psi}L_{\psi}^{-\top}
 =U\operatorname{diag}(s)V^\top,
\qquad H=\{i:s_i\geq\tau\}.
\]
The residual is
\[
z_\perp=(z_r-\mu_r)T_\tau,
\qquad
T_\tau=L_r^{-\top}(I-U_HU_H^\top)L_r^\top.
\]
During training, means and operators are computed on each minibatch, with gradients through the selected subspace. Covariance calculations use double precision. After each epoch, the means and operator are calibrated on the fixed calibration partition before validation. The selected checkpoint retains this calibrated operator, which is frozen for test-time inference.

The dual baseline uses $T=I$. OLS instead fits $B=Z_{\psi,c}^{\dagger}Z_{r,c}$ and returns $Z_{r,c}-Z_{\psi,c}B$, requiring paired expert embeddings at inference. LEACE uses unregularized empirical covariances and a relative rank tolerance of $10^{-6}$; its fitted map uses only the raw embedding at inference. These operators are also recomputed during training and frozen after calibration.

\subsection{Prediction head}
\label{app:psyctraj-head}

The head combines raw-residual and expert main effects with their elementwise interaction:
\[
\begin{aligned}
h_r&=g_r(\mathrm{LN}(z_\perp)),
& h_{\psi}&=g_{\psi}(\mathrm{LN}(z_{\psi}-\mu_{\psi})),\\
\ell&=w_r^\top h_r+w_{\psi}^\top h_{\psi}
 +w_{\mathrm{int}}^\top(h_r\odot h_{\psi})+b.
\end{aligned}
\]
Each $g$ is a $64\to64$ linear layer followed by SiLU and dropout .2. The sigmoid of $\ell$ gives the predicted dropout probability.

\subsection{Training protocol}
\label{app:psyctraj-training}

The models in Table~\ref{tab:psych_evidence_audit_main} minimize binary cross-entropy with logits, with positive-class weight $w_+=N_{\mathrm{neg}}/N_{\mathrm{pos}}$ computed on Site A training patients. We use AdamW with learning rate $2\times10^{-4}$, weight decay $10^{-5}$, and batch size 512. An epoch-wise cosine schedule decays toward zero over a maximum of 200 epochs, without warmup. We stop after 20 epochs without improvement and select the checkpoint with the lowest weighted validation loss. No auxiliary code-prediction loss is used.

The completed runs use training seeds 1996--2000 and the same data partitions. Table~\ref{tab:psych_evidence_audit_main} reports each method's completed count; unfinished runs are excluded. Runs use a hospital compute cluster with Python 3.10.18 and PyTorch 2.6.0.

\subsection{Evaluation}
\label{app:psyctraj-evaluation}

Native prediction uses the selected classifier and frozen operator, without recalibration on the test sites. We report AUROC and AUPRC (average precision). For accuracy, the native decision threshold maximizes Youden's $J(t)=\mathrm{TPR}(t)-\mathrm{FPR}(t)$ on Site A validation; AUROC and AUPRC do not depend on this threshold.

For the omission audit, we fit separate predictors of CPT~90806 presence and three psychiatric codes retained in the expert input. The included codes are chosen from the training panel among codes with prevalence between 5\% and 95\%, preferring prevalence closest to 50\%. Linear probes use standardized logistic regression with $C\in\{.1,1,10\}$ and at most 2,000 iterations. Nonlinear probes use gradient-boosted trees with at most 7 or 15 leaves and 150 iterations. Probes are fitted on the probe partition and selected by Site A validation log loss. Selection includes a constant predictor; a combined expert--residual predictor can also retain the fitted expert-only predictor. All selected predictors are frozen before evaluation on Sites B--D. Code recovery uses AUROC; included-code scores are averaged within each seed before computing the mean and sample SD across seeds.

To test whether the omitted code explains the residual's predictive contribution, we keep the residual fixed and add either the true code values or values randomly reassigned among patients to the observed expert inputs. The reassignment stays within each calibration, probe, validation, and test partition. We fit dropout predictors with and without the fixed residual and compare their log-loss gains before and after restoration. This differs from retraining the network with restored or shuffled inputs, which changes both the embeddings and the operator. Paired loss contrasts use 1,000 patient-bootstrap samples within each site and fitted seed. These intervals do not measure variability across training seeds, and a constant predictor or zero gain does not establish absence of information.

\subsection{SCCR threshold sweep}
\label{app:psyctraj-tau-sweep}

We train separate models for $\tau\in\{.05,.10,.20,.30,.40,.50,.60,.70,.80,.90\}$, keeping the architecture, optimizer, preprocessing, and patient splits fixed. The main audit uses $\tau=.40$; Table~\ref{tab:psych_evidence_audit_main} also reports $.05$ and $.70$. This sweep describes sensitivity to the threshold and does not establish an optimal value for a new hospital.

\subsection{Fixed code panel}
\label{app:psyctraj-code-panel}

Table~\ref{tab:psych-code-panel} lists the 256 codes in panel rank order. Prefixes identify the code systems: 123 CPT4, 86 ICD10, and 47 RxNorm codes. Read each rank--code pair from top to bottom. CPT~90806 is the code omitted from the expert view in the controlled audit.

\begingroup
\small
\setlength{\tabcolsep}{8pt}
\begin{longtable}{@{}rlrlrlrl@{}}
\caption{Fixed code panel for treatment-dropout prediction.}
\label{tab:psych-code-panel}\\
\toprule
Rank & Code & Rank & Code & Rank & Code & Rank & Code \\
\midrule
\endfirsthead
\multicolumn{8}{c}{\tablename~\thetable\ continued}\\
\toprule
Rank & Code & Rank & Code & Rank & Code & Rank & Code \\
\midrule
\endhead
\midrule
\multicolumn{8}{r}{Continued on next page}\\
\endfoot
\bottomrule
\endlastfoot
1 & \texttt{cpt4:99213} & 65 & \texttt{icd10:f33.9} & 129 & \texttt{cpt4:76856} & 193 & \texttt{cpt4:82247} \\
2 & \texttt{cpt4:99214} & 66 & \texttt{cpt4:99284} & 130 & \texttt{cpt4:88142} & 194 & \texttt{rxnorm:36567} \\
3 & \texttt{icd10:i10} & 67 & \texttt{rxnorm:1191} & 131 & \texttt{icd10:r06.02} & 195 & \texttt{rxnorm:4603} \\
4 & \texttt{cpt4:99211} & 68 & \texttt{cpt4:97110} & 132 & \texttt{rxnorm:8640} & 196 & \texttt{rxnorm:214182} \\
5 & \texttt{cpt4:85027} & 69 & \texttt{cpt4:71010} & 133 & \texttt{cpt4:85007} & 197 & \texttt{icd10:z04} \\
6 & \texttt{cpt4:36415} & 70 & \texttt{cpt4:82310} & 134 & \texttt{cpt4:83550} & 198 & \texttt{icd10:n95.1} \\
7 & \texttt{cpt4:85025} & 71 & \texttt{cpt4:82435} & 135 & \texttt{cpt4:71260} & 199 & \texttt{rxnorm:2598} \\
8 & \texttt{cpt4:84443} & 72 & \texttt{icd10:g44.1} & 136 & \texttt{cpt4:81025} & 200 & \texttt{rxnorm:4511} \\
9 & \texttt{icd10:z00.00} & 73 & \texttt{cpt4:87491} & 137 & \texttt{rxnorm:18631} & 201 & \texttt{icd10:f39} \\
10 & \texttt{cpt4:80048} & 74 & \texttt{cpt4:84460} & 138 & \texttt{cpt4:70450} & 202 & \texttt{cpt4:90853} \\
11 & \texttt{cpt4:82947} & 75 & \texttt{cpt4:82374} & 139 & \texttt{cpt4:90658} & 203 & \texttt{icd10:f31.9} \\
12 & \texttt{cpt4:71020} & 76 & \texttt{cpt4:87591} & 140 & \texttt{cpt4:84702} & 204 & \texttt{cpt4:90834} \\
13 & \texttt{icd10:e78.4} & 77 & \texttt{icd10:m54.5} & 141 & \texttt{cpt4:82728} & 205 & \texttt{cpt4:90804} \\
14 & \texttt{cpt4:80053} & 78 & \texttt{cpt4:87070} & 142 & \texttt{rxnorm:9143} & 206 & \texttt{icd10:f29} \\
15 & \texttt{cpt4:g0001} & 79 & \texttt{rxnorm:89905} & 143 & \texttt{cpt4:76092} & 207 & \texttt{icd10:f31.30} \\
16 & \texttt{icd10:f32.9} & 80 & \texttt{rxnorm:6387} & 144 & \texttt{cpt4:97001} & 208 & \texttt{cpt4:90807} \\
17 & \texttt{cpt4:82565} & 81 & \texttt{rxnorm:29046} & 145 & \texttt{icd10:r00.2} & 209 & \texttt{rxnorm:28439} \\
18 & \texttt{cpt4:84520} & 82 & \texttt{icd10:n39.0} & 146 & \texttt{cpt4:87340} & 210 & \texttt{icd10:f33.1} \\
19 & \texttt{icd10:z23} & 83 & \texttt{icd10:e03.9} & 147 & \texttt{icd10:c50.919} & 211 & \texttt{icd10:f33.2} \\
20 & \texttt{cpt4:88305} & 84 & \texttt{rxnorm:41126} & 148 & \texttt{cpt4:80100} & 212 & \texttt{icd10:f31.10} \\
21 & \texttt{cpt4:85610} & 85 & \texttt{cpt4:83690} & 149 & \texttt{cpt4:77052} & 213 & \texttt{cpt4:90791} \\
22 & \texttt{cpt4:80061} & 86 & \texttt{cpt4:81002} & 150 & \texttt{cpt4:83615} & 214 & \texttt{cpt4:90805} \\
23 & \texttt{cpt4:80076} & 87 & \texttt{rxnorm:214183} & 151 & \texttt{cpt4:82306} & 215 & \texttt{icd10:f25.9} \\
24 & \texttt{cpt4:99215} & 88 & \texttt{cpt4:99000} & 152 & \texttt{cpt4:99282} & 216 & \texttt{icd10:f31.60} \\
25 & \texttt{cpt4:83735} & 89 & \texttt{icd10:i25.10} & 153 & \texttt{rxnorm:36387} & 217 & \texttt{cpt4:80178} \\
26 & \texttt{icd10:r07.9} & 90 & \texttt{icd10:m54.2} & 154 & \texttt{rxnorm:5856} & 218 & \texttt{cpt4:90847} \\
27 & \texttt{icd10:z12.31} & 91 & \texttt{cpt4:84075} & 155 & \texttt{cpt4:93000} & 219 & \texttt{cpt4:90792} \\
28 & \texttt{cpt4:84100} & 92 & \texttt{cpt4:90471} & 156 & \texttt{cpt4:99243} & 220 & \texttt{rxnorm:6448} \\
29 & \texttt{cpt4:93005} & 93 & \texttt{cpt4:82150} & 157 & \texttt{cpt4:90801} & 221 & \texttt{cpt4:90832} \\
30 & \texttt{rxnorm:161} & 94 & \texttt{icd10:z01.810} & 158 & \texttt{icd10:n63} & 222 & \texttt{icd10:f33.0} \\
31 & \texttt{cpt4:87086} & 95 & \texttt{cpt4:82465} & 159 & \texttt{cpt4:99291} & 223 & \texttt{rxnorm:2002} \\
32 & \texttt{rxnorm:5640} & 96 & \texttt{icd10:d64.9} & 160 & \texttt{cpt4:84478} & 224 & \texttt{icd10:f20.9} \\
33 & \texttt{rxnorm:6470} & 97 & \texttt{cpt4:99204} & 161 & \texttt{rxnorm:723} & 225 & \texttt{icd10:f31.81} \\
34 & \texttt{cpt4:99212} & 98 & \texttt{cpt4:82040} & 162 & \texttt{icd10:k52.89} & 226 & \texttt{icd10:f32.1} \\
35 & \texttt{icd10:f41.9} & 99 & \texttt{cpt4:82607} & 163 & \texttt{cpt4:86592} & 227 & \texttt{icd10:f20.0} \\
36 & \texttt{icd10:e11.9} & 100 & \texttt{cpt4:80051} & 164 & \texttt{rxnorm:221147} & 228 & \texttt{cpt4:90870} \\
37 & \texttt{rxnorm:7804} & 101 & \texttt{cpt4:85652} & 165 & \texttt{icd10:z01.411} & 229 & \texttt{rxnorm:32624} \\
38 & \texttt{cpt4:81003} & 102 & \texttt{cpt4:99244} & 166 & \texttt{rxnorm:1202} & 230 & \texttt{icd10:f33.3} \\
39 & \texttt{rxnorm:7646} & 103 & \texttt{cpt4:g0202} & 167 & \texttt{cpt4:72193} & 231 & \texttt{icd10:z65.8} \\
40 & \texttt{cpt4:93010} & 104 & \texttt{icd10:z01.818} & 168 & \texttt{cpt4:88307} & 232 & \texttt{icd10:f10.99} \\
41 & \texttt{icd10:z12.4} & 105 & \texttt{rxnorm:83367} & 169 & \texttt{cpt4:87186} & 233 & \texttt{icd10:f33.41} \\
42 & \texttt{cpt4:99283} & 106 & \texttt{icd10:m25.569} & 170 & \texttt{rxnorm:10831} & 234 & \texttt{icd10:f23} \\
43 & \texttt{icd10:m54.9} & 107 & \texttt{icd10:r42} & 171 & \texttt{cpt4:82248} & 235 & \texttt{icd10:f32.2} \\
44 & \texttt{cpt4:85730} & 108 & \texttt{rxnorm:8591} & 172 & \texttt{icd10:j02.9} & 236 & \texttt{icd10:f32.0} \\
45 & \texttt{rxnorm:26225} & 109 & \texttt{icd10:r06.09} & 173 & \texttt{icd10:z33.1} & 237 & \texttt{cpt4:90837} \\
46 & \texttt{cpt4:84132} & 110 & \texttt{rxnorm:7052} & 174 & \texttt{cpt4:99395} & 238 & \texttt{icd10:f31.4} \\
47 & \texttt{cpt4:99232} & 111 & \texttt{rxnorm:3423} & 175 & \texttt{rxnorm:3264} & 239 & \texttt{icd10:f32.3} \\
48 & \texttt{icd10:m79.609} & 112 & \texttt{cpt4:82550} & 176 & \texttt{icd10:r10.84} & 240 & \texttt{cpt4:90836} \\
49 & \texttt{icd10:r10.9} & 113 & \texttt{cpt4:97140} & 177 & \texttt{rxnorm:5487} & 241 & \texttt{icd10:f32.89} \\
50 & \texttt{cpt4:99285} & 114 & \texttt{icd10:e78.0} & 178 & \texttt{rxnorm:3498} & 242 & \texttt{icd10:f32.4} \\
51 & \texttt{cpt4:84450} & 115 & \texttt{rxnorm:6918} & 179 & \texttt{rxnorm:6585} & 243 & \texttt{icd10:f31.32} \\
52 & \texttt{icd10:r53.1} & 116 & \texttt{cpt4:88175} & 180 & \texttt{icd10:n76.2} & 244 & \texttt{cpt4:90833} \\
53 & \texttt{icd10:k21.9} & 117 & \texttt{cpt4:83718} & 181 & \texttt{icd10:z71.89} & 245 & \texttt{cpt4:96151} \\
54 & \texttt{cpt4:83036} & 118 & \texttt{icd10:j06.9} & 182 & \texttt{rxnorm:2551} & 246 & \texttt{icd10:f33.42} \\
55 & \texttt{icd10:r05} & 119 & \texttt{rxnorm:10582} & 183 & \texttt{rxnorm:6960} & 247 & \texttt{icd10:f31.31} \\
56 & \texttt{cpt4:81001} & 120 & \texttt{rxnorm:25480} & 184 & \texttt{cpt4:99205} & 248 & \texttt{cpt4:90816} \\
57 & \texttt{rxnorm:82003} & 121 & \texttt{icd10:f17.200} & 185 & \texttt{icd10:m25.519} & 249 & \texttt{cpt4:g0073} \\
58 & \texttt{rxnorm:9863} & 122 & \texttt{cpt4:76830} & 186 & \texttt{icd10:j98.4} & 250 & \texttt{icd10:f32.5} \\
59 & \texttt{cpt4:90862} & 123 & \texttt{rxnorm:4337} & 187 & \texttt{icd10:r21} & 251 & \texttt{icd10:f31.5} \\
60 & \texttt{rxnorm:435} & 124 & \texttt{cpt4:99233} & 188 & \texttt{cpt4:86803} & 252 & \texttt{icd10:f20.89} \\
61 & \texttt{icd10:e66.9} & 125 & \texttt{cpt4:99231} & 189 & \texttt{cpt4:74160} & 253 & \texttt{icd10:f31.12} \\
62 & \texttt{cpt4:84295} & 126 & \texttt{cpt4:99396} & 190 & \texttt{icd10:n92.5} & 254 & \texttt{icd10:f31.64} \\
63 & \texttt{cpt4:99203} & 127 & \texttt{icd10:j45.998} & 191 & \texttt{cpt4:82570} & 255 & \texttt{cpt4:90868} \\
64 & \texttt{cpt4:90806} & 128 & \texttt{cpt4:87205} & 192 & \texttt{icd10:n94.89} & 256 & \texttt{icd10:f33.40} \\
\end{longtable}
\endgroup

\section{Experimental Details: CAD detection in Coronary angiography experiment}
\label{app:angio-sccr}
\begin{table*}[h]
\centering
\scriptsize
\setlength{\tabcolsep}{2.7pt}
\begin{tabular}{lcccccc}
\toprule
$\boldsymbol{\tau}$ & \textbf{AUROC [95\% CI]} & \textbf{AUPRC [95\% CI]} & \textbf{Mask drop} & \textbf{Raw drop} & \textbf{Mask/raw attr.} & \textbf{Evidence} \\
 & & & \textbf{zero / perm.} & \textbf{zero / perm.} & & \\
\midrule
0.1 & 0.792 [0.772, 0.809] & 0.766 [0.738, 0.792] & \textbf{0.283 / 0.303} & 0.000 / 0.003 & 31.25 & Mask \\
0.5 & 0.790 [0.771, 0.807] & 0.770 [0.745, 0.792] & \textbf{0.261 / 0.278} & 0.000 / 0.000 & 16.17 & Mask \\
0.7 & \textbf{0.835 [0.819, 0.851]} & 0.820 [0.795, 0.842] & 0.000 / 0.000 & \textbf{0.387 / 0.348} & 0.029 & Raw \\
0.9 & 0.831 [0.816, 0.846] & 0.820 [0.797, 0.842] & -0.001 / 0.000 & \textbf{0.342 / 0.318} & 0.015 & Raw \\
\bottomrule
\end{tabular}
\caption{\textbf{Historical CAD angiogram evidence-source audit.} AUROC and AUPRC include bootstrap 95\% confidence intervals. Drops are point-estimate decreases in test AUROC after zeroing or permuting the corresponding input branch. Attribution is the mask/raw Grad$\times$Input mass ratio for the full logit.}
\label{tab:cad_evidence_audit}
\end{table*}

\paragraph{Dataset and sample construction.}
We evaluate on a coronary angiography video prediction task using a metadata parquet file, a raw-frame HDF5 file, and a mask HDF5 file. The raw HDF5 stores each angiography video as a Blosc-compressed uint8 tensor of shape
\[
20 \times 256 \times 256,
\]
corresponding to 20 grayscale frames per video. The metadata table is indexed by video path and contains the binary label for CAD or no CAD.

Each metadata row contributes at most one training example. For a retained video, the raw branch input is the selected grayscale angiography frame. The expert branch input is the corresponding vessel mask for the same frame. Both raw frames and masks are converted to three-channel tensors by repeating the grayscale channel. Raw frames are scaled to $[0,1]$ and normalized with ImageNet statistics:
\[
\mu=(0.485,0.456,0.406), \qquad
\sigma=(0.229,0.224,0.225).
\]
Vessel masks are scaled to $[-1,1]$.

\paragraph{Train, validation, and test splits.}
The split is performed at the video level. The default split fractions are
\[
80\% \text{ train}, \qquad 10\% \text{ validation}, \qquad 10\% \text{ test}.
\]
The test split is produced with random seed 42, and the validation split is produced from the remaining train-validation pool with seed 43. Final metrics record both the number of videos in each split and the number of retained longest-frame samples after mask-availability filtering. The validation and test sets each contain 2,234 retained images; the follow-up audit uses 17,866 eligible training images.

\paragraph{Training-time augmentation.}
During classifier training, data augmentation is applied only to the training split. Geometric transformations are applied jointly to the raw angiography frame and the corresponding mask, preserving spatial alignment between branches. With probability 0.5, the pair undergoes a random affine transform with rotation sampled from $[-20^\circ,20^\circ]$, translations sampled from $[-10,10]$ pixels in each direction, and scale sampled from $[0.9,1.1]$.

Photometric augmentations are applied only to the raw angiography image. With probability 0.5, gamma is sampled from $[0.7,1.5]$. With probability 0.5, brightness and contrast factors are each sampled from $[0.8,1.2]$. With probability 0.3, Gaussian blur with kernel size 3 and $\sigma \in [0.1,1.5]$ is applied. Validation, test, and calibration loaders use no augmentation.

\paragraph{Dual-encoder architecture.}
The prediction model is a two-branch image encoder. The raw branch receives the normalized angiography frame, and the expert branch receives the corresponding vessel mask. Each branch uses an independent EfficientNetV2-S backbone initialized with ImageNet weights. The classifier layer of each EfficientNet is replaced by the identity map, yielding one image embedding per branch. If the requested torchvision backbone is unavailable, the implementation falls back to a small two-layer convolutional encoder; all reported runs used the EfficientNetV2-S backbone.

Let $E_r$ and $E_m$ denote the raw and mask encoders. Their outputs are $\ell_2$-normalized and passed through independent two-layer projector MLPs:
\[
z_r = \Pi_r(\mathrm{normalize}(E_r(x_r))),
\qquad
z_m = \Pi_m(\mathrm{normalize}(E_m(x_m))).
\]
Each projector has the form
\[
\Pi(z)
=
W_2 \,
\mathrm{SiLU}(W_1 z + b_1)
+ b_2.
\]
The reported checkpoints use a projected dimension of 30. The same dimension is used for the SCCR layer and the downstream prediction head.

\paragraph{Shared canonical component removal.}
We remove mask-shared directions from the raw branch using Shared Canonical Component Removal. Given centered projected embeddings
\[
Z_r, Z_m \in \mathbb{R}^{n \times d},
\]
we compute empirical covariances
\[
\Sigma_{rr}, \qquad \Sigma_{mm}, \qquad \Sigma_{rm}.
\]
Each branch covariance is regularized by an OAS-selected equivalent ridge. For a covariance matrix $S$, with
\[
\mu=\frac{\operatorname{tr}(S)}{d},
\]
the OAS shrinkage coefficient $\rho$ is converted to
\[
\lambda = \frac{\rho \mu}{1-\rho}.
\]
The regularized branch covariances are
\[
R_{rr} = \Sigma_{rr} + \lambda_r I,
\qquad
R_{mm} = \Sigma_{mm} + \lambda_m I.
\]
Let
\[
L_r L_r^\top = R_{rr},
\qquad
L_m L_m^\top = R_{mm}.
\]
SCCR forms the whitened cross-covariance
\[
M =
L_r^{-1}
\Sigma_{rm}
L_m^{-\top},
\]
and computes
\[
M = U \operatorname{diag}(s) V^\top.
\]
For a threshold $\tau$, all canonical directions with
\[
s_i \geq \tau
\]
are removed from the raw branch. If $U_\tau$ denotes the corresponding left singular vectors, then
\[
P_\tau = U_\tau U_\tau^\top,
\qquad
K_\tau = I - P_\tau,
\]
and the raw residualization operator is
\[
T_\tau
=
L_r^{-\top}
K_\tau
L_r^\top.
\]
The residualized raw representation is
\[
\widetilde{z}_r = z_r T_\tau.
\]

During training, embeddings are centered using the current minibatch mean, and SCCR is recomputed dynamically. Before validation after each epoch, and again before final test evaluation, the model is calibrated on the non-augmented training split. Calibration estimates the global raw and mask projector means on the training set and freezes the SCCR projection $T_\tau$. Validation and test metrics are computed using this frozen train-set projection.

\paragraph{Prediction head.}
The classifier head uses explicit raw, mask, and interaction terms. Given the SCCR-residualized raw representation $\widetilde{z}_r$ and the centered mask representation $z_m$, the head computes
\[
h_r = g_r(\mathrm{LN}(\widetilde{z}_r)),
\qquad
h_m = g_m(\mathrm{LN}(z_m)),
\]
where $g_r$ and $g_m$ are one-layer SiLU networks. The interaction representation is
\[
h_{\mathrm{int}} = h_r \odot h_m.
\]
The final logit is
\[
\ell
=
w_r^\top h_r
+
w_m^\top h_m
+
w_{\mathrm{int}}^\top h_{\mathrm{int}}
+
b.
\]
The output probability is $\sigma(\ell)$.

\paragraph{Optimization and model selection.}
The model is trained with binary cross-entropy with logits. When both positive and negative examples are present in the training split, the loss uses a positive-class weight
\[
w_+ = \frac{N_{\mathrm{neg}}}{N_{\mathrm{pos}}},
\]
computed at the retained sample level. The reported checkpoints use AdamW with learning rate $3\times 10^{-4}$, weight decay $10^{-2}$, batch size 32, no gradient clipping, and a maximum of 30 epochs, with training seed 42.

A ReduceLROnPlateau scheduler monitors validation loss, reduces the learning rate by a factor of 0.5 after three non-improving validation epochs, and uses a minimum learning rate of $10^{-6}$. Early stopping is triggered after five consecutive non-improving validation epochs. The best checkpoint is selected by lowest validation loss and saved.

\paragraph{Evaluation.}
After training, the best checkpoint is loaded, the SCCR projection is recalibrated on the non-augmented training split, and final validation and test metrics are computed. We report binary cross-entropy loss, accuracy, AUROC, AUPRC, sample count, and positive-label rate. 

\paragraph{SCCR threshold sweep.}
We sweep the SCCR threshold over
\[
\tau \in \{0.1,0.2,0.3,0.4,0.5,0.6,0.7,0.8,0.9\}.
\]
All other hyperparameters are held fixed unless explicitly overridden. Each threshold run should use its own output directory so that checkpoints, split files, and final metrics do not collide across runs.

\subsection{Follow-up prediction and residual audits}
\label{app:angio-followup}

We compare the saved SCCR checkpoints with a matched dual baseline without removal and with OLS and LEACE training adaptations. These are retrospective, single-seed comparisons on the previously inspected test set of 2,234 video rows, including 1,075 positive labels. Table~\ref{tab:cad_evidence_audit} retains the original threshold-sweep results; the following tables report the later audit and its bootstrap estimates. They do not provide external-site or training-seed validation.

\paragraph{Comparisons and readouts.}
Native prediction uses each trained head with its own expert and retained embeddings. The saved SCCR maps are loaded without recalibration. For OLS and LEACE, the operator is fitted on non-augmented training embeddings before the first epoch and refitted before validation, then held fixed during each epoch; gradients pass through the embeddings and fixed map, not through its fit. OLS uses $R=X_c-E_c(E_c^+X_c)$ and requires the paired expert at inference. Separately, we fit common-embedding operators on 8,933 training rows and readouts on a disjoint 8,933 training rows, both previously used to train the CNN. Validation selects readouts. Linear diagnosis probes use standardized logistic regression with $C\in\{.01,.1,1\}$; recovery uses ridge penalties $\{1,100\}$. Nonlinear probes apply an RBF Nystroem expansion with at most 64 components, with bandwidth multipliers $\{.5,2\}$. Constant predictors are eligible; failed or nonconverged candidates are excluded.

The common-embedding comparison uses the dual-baseline encoders and a fixed expert input $E_0$ containing its expert embedding and mask area. We define $G_0=L(E_0)-L(E_0,R)$ using unweighted diagnosis log loss. This input differs from each SCCR model's own expert $E$. The latter comparison uses $G_X=L(E,0)-L(E,X)$ and $G_R=L(E,0)-L(E,R)$, where zero padding holds the input width fixed for the nonlinear bandwidth. Common-embedding readouts retain the original width-dependent bandwidth, so adding constant columns can change their nonlinear predictions without adding information. All intervals use 1,000 paired bootstrap resamples of test video rows, conditional on the fitted models and readouts, without multiplicity adjustment.

\paragraph{Baseline adaptations.}
Common OLS and LEACE use empirical covariances. D-CCA uses a signal-covariance plug-in without S-POET denoising or correlation significance selection. JIVE has joint rank 1 and individual ranks $(15,15)$. AJIVE selects initial ranks $(5,2)$ by the largest adjacent singular-value ratio and joint rank zero using its Wedin/random-direction rule with 200 resamples. Both report individual outputs and use a calibration-fitted paired least-squares extension to new rows. The fixed-rank $(16,16)$ AJIVE sensitivity setting was ineligible because the implemented Wedin resampling requires dimension at least $2r$; whitened JIVE settings did not converge within 20,000 iterations. These outcomes are not assigned prediction scores. VCCA-private and DSN were not run here.

\begin{table}[H]
\centering\scriptsize
\setlength{\tabcolsep}{3pt}
\begin{tabular*}{\textwidth}{@{\extracolsep{\fill}}lccc@{}}
\toprule
\multicolumn{4}{l}{\textbf{(a) Separately trained native heads}}\\
Model & AUROC & AUPRC & Log loss \\
\midrule
Dual baseline & .822 [.806, .839] & .815 [.794, .837] & .526 [.499, .550] \\
Epoch OLS & .774 [.754, .793] & .722 [.691, .752] & .594 [.570, .621] \\
Epoch LEACE & .793 [.775, .810] & .773 [.748, .801] & .560 [.537, .585] \\
SCCR $\tau=.10$ & .792 [.775, .811] & .766 [.739, .793] & .552 [.535, .569] \\
SCCR $\tau=.30$ & .788 [.770, .807] & .767 [.742, .792] & .553 [.535, .569] \\
SCCR $\tau=.90$ & .831 [.815, .847] & .820 [.797, .840] & .516 [.491, .541] \\
\midrule
\multicolumn{4}{l}{\textbf{(b) Common frozen embeddings: linear probes unless marked RBF}}\\
Retained output & Residual AUROC & $(E_0,R)$ AUROC & $G_0$ \\
\midrule
Dual baseline / SCCR $.90$ & .819 [.804, .837] & .819 [.803, .837] & .0856 [.0489, .1259] \\
OLS residual & .815 [.798, .834] & .819 [.802, .837] & .0902 [.0535, .1298] \\
LEACE residual & .500 [.500, .500] & .567 [.544, .591] & -.0003 [-.0007, .0001] \\
D-CCA distinctive & .818 [.800, .835] & .819 [.802, .837] & .0882 [.0507, .1272] \\
JIVE individual & .728 [.708, .749] & .819 [.803, .837] & .0907 [.0545, .1301] \\
AJIVE individual & .819 [.803, .837] & .819 [.802, .837] & .0856 [.0491, .1260] \\
SCCR $.10$ & .708 [.686, .729] & .715 [.694, .736] & .0191 [-.0094, .0470] \\
SCCR $.10$, RBF & .807 [.790, .825] & .780 [.761, .799] & .0902 [.0602, .1219] \\
JIVE individual, RBF & .815 [.799, .833] & .817 [.800, .835] & .1004 [.0647, .1412] \\
\bottomrule
\end{tabular*}
\caption{Follow-up CAD prediction, with 95\% bootstrap intervals. Native heads use their own expert and retained embeddings; common-embedding readouts use $E_0$. SCCR $.90$ removes no common directions and agrees with the dual baseline. Common SCCR $.10$ removes three directions and equals matched top-3 CCA; the identical control is not repeated. Larger $G_0$ means lower diagnosis log loss after adding the residual. These rows do not establish a consistent SCCR advantage.}
\label{tab:angio_followup_prediction}
\end{table}

\paragraph{Diagnosis and descriptor recovery.}
The saved native $.10$ map removes four directions, while the $.30$ and $.90$ maps remove none. At $.10$, the own-expert linear diagnosis AUROC is $.794$ $[.777,.813]$, whereas the pre-removal raw AUROC is only $.548$ $[.522,.571]$. Neither the raw embedding nor the residual shows a clear log-loss gain beyond this expert (Table~\ref{tab:angio_followup_audit}). Thus, the absent residual gain cannot be attributed only to removal. Linear mask-area and image-mean recovery decrease, but nonlinear recovery remains. Area is the mask foreground fraction, and perimeter is the four-connected boundary count divided by image pixel count. Image mean and SD use the unaugmented grayscale image; vessel mean uses its pixels inside the mask. These descriptors do not measure clinically validated disease mechanisms. In the separate $E_0$ audit, we keep the residual fixed and add the appearance panel $C$ (global, vessel, and background intensity summaries) or its row-shuffled values $C^\pi$. Defining $G_C=L(E_0,C)-L(E_0,C,R)$, the true-versus-shuffled contrast $G_{C^\pi}-G_C$ for native SCCR at $.10$ is $-.0004$ $[-.0011,.0004]$ with linear probes and $.0009$ $[-.0015,.0033]$ with nonlinear probes. Branch sensitivity and reduced linear recovery therefore do not establish clinically useful complementary evidence.

\begin{table}[H]
\centering\scriptsize
\setlength{\tabcolsep}{3pt}
\begin{tabular*}{\textwidth}{@{\extracolsep{\fill}}llccc@{}}
\toprule
\multicolumn{5}{l}{\textbf{(a) Diagnosis gain beyond each SCCR model's own expert}}\\
$\tau$ & Probe & $G_X$ & $G_R$ & $G_X-G_R$ \\
\midrule
.10 & Linear & -.0022 [-.0044, .0001] & -.0002 [-.0022, .0017] & -.0020 [-.0040, .0001] \\
.10 & RBF & -.0019 [-.0060, .0025] & -.0009 [-.0055, .0040] & -.0010 [-.0042, .0025] \\
.30 & Linear & .0009 [-.0006, .0024] & .0009 [-.0006, .0024] & .0000 [.0000, .0000] \\
.30 & RBF & -.0005 [-.0057, .0049] & -.0005 [-.0057, .0049] & .0000 [.0000, .0000] \\
\midrule
\multicolumn{5}{l}{\textbf{(b) Descriptor recovery in the saved SCCR $\tau=.10$ model}}\\
Descriptor & $X$, linear & $R$, linear & $X$, RBF & $R$, RBF \\
\midrule
Mask area & .092 [.072, .108] & .003 [-.004, .008] & .107 [.082, .126] & .087 [.060, .106] \\
Mask perimeter & .111 [.091, .129] & .006 [-.003, .013] & .141 [.113, .163] & .124 [.096, .147] \\
Image mean & .922 [.916, .927] & .169 [.139, .197] & .949 [.944, .953] & .936 [.931, .941] \\
Image SD & .752 [.732, .771] & .136 [.106, .166] & .806 [.788, .821] & .783 [.765, .801] \\
Vessel mean & .672 [.647, .695] & .375 [.336, .412] & .703 [.680, .725] & .677 [.651, .699] \\
\bottomrule
\end{tabular*}
\caption{Follow-up CAD audit with 95\% bootstrap intervals. (a) Positive $G_X$ or $G_R$ indicates improvement beyond the same model's own expert; positive $G_X-G_R$ indicates loss of gain after removal. Near-zero differences at $.30$ are rounded, and its frozen map is the identity. (b) Values are test $R^2$ before and after removal. RBF denotes the bounded nonlinear probe, not a quadratic probe. The geometric and appearance targets are computed from the masks and images.}
\label{tab:angio_followup_audit}
\end{table}

\subsection{Vessel mask generation}
\label{app:angio-mask-generation}

The expert branch uses vessel masks generated by a U-Net trained on the ARCADE coronary angiography segmentation data. The segmentation model is trained using grayscale angiography images and COCO-style polygon annotations. All images and masks are resized to
\[
256 \times 256.
\]
For each annotated image, polygon coordinates are scaled from the original image resolution to $256\times256$, and vessel polygons are rasterized into a binary mask. The U-Net input has shape $256\times256\times1$, and the output is a single-channel sigmoid mask.

\paragraph{U-Net architecture.}
The mask generator is a four-level U-Net with skip connections. The encoder uses convolutional blocks with channel widths
\[
64,\ 128,\ 256,\ 512,
\]
followed by a bottleneck with 1024 channels. Each block contains two $3\times3$ convolutions with ReLU activations and batch normalization. Max pooling is used for downsampling. Dropout probabilities are 0.1 in the first two encoder stages, 0.2 in the deeper encoder stages, and 0.3 in the bottleneck.

The decoder upsamples by a factor of two at each stage, concatenates the corresponding encoder skip connection, and applies two $3\times3$ convolutional layers with batch normalization. The decoder channel widths are
\[
512,\ 256,\ 128,\ 64.
\]
The final layer is a $1\times1$ convolution with sigmoid activation.

\paragraph{Segmentation training.}
The U-Net is trained with batch size 16 for up to 2000 epochs. The optimizer is Adam. The loss is the sum of binary cross-entropy and Dice loss:
\[
\mathcal{L}_{\mathrm{seg}}
=
\mathcal{L}_{\mathrm{BCE}}
+
\left(1 -
\frac{
2\langle y,\hat{y}\rangle + \epsilon
}{
\|y\|_1 + \|\hat{y}\|_1 + \epsilon
}
\right),
\]
with $\epsilon=10^{-6}$. The monitored metric is the Dice coefficient.

Training uses random augmentations of image-mask pairs. With probability 0.3 each, the generator applies a small rotation sampled from $[-10^\circ,10^\circ]$, a translation sampled from $\pm 5$ pixels, a brightness/contrast perturbation applied only to the image, and light Gaussian blur applied only to the image. Geometric transformations use bilinear interpolation for images and nearest-neighbor interpolation for masks.

The best segmentation model is selected by validation loss and saved. The segmentation training uses ReduceLROnPlateau with factor 0.5, patience 5, and minimum learning rate $10^{-6}$, and early stopping with patience 15 and best-weight restoration.

The downstream SCCR classifier consumes the generated masks from an HDF5 file organized by frame index. For each frame group \texttt{frame\_i}, the file stores mask arrays and the corresponding video-path indices. The classifier uses these stored masks as the expert-view input for the longest frame of each video.

The angiogram experiments used an NVIDIA T4 GPU with 16GB of RAM, and Pytorch 2.2.1.

\end{document}